\documentclass[journal]{IEEEtran}

\usepackage{tipa}
\usepackage{bbm}
\usepackage{mathrsfs}
\usepackage{amsfonts}
\usepackage{cite,url,subfig,epsfig,graphicx}
\usepackage{amssymb,amsmath}
\usepackage{algorithmic}
\usepackage{algorithm}
\usepackage{setspace}
\usepackage{CJK}
\usepackage{soul}
\usepackage{xcolor}

\usepackage{amsthm}
\usepackage{times,verbatim,amsfonts,amsmath,color}
\newtheorem{definition}{\textbf{Definition}}
\newtheorem{lemma}{\textbf{Lemma}}
\newtheorem{theorem}{\textbf{Theorem}}

\newsavebox{\tablebox}

\newcommand{\argmin}{\operatornamewithlimits{argmin}}

\usepackage{float}
\begin{document}
	
	\title{DP-ADMM: ADMM-based Distributed Learning with Differential Privacy}

	\author{Zonghao Huang, Rui Hu, Yuanxiong Guo, Eric Chan-Tin, and Yanmin Gong% <-this % stops a space
		
		\thanks{The work of R. Hu and Y. Gong is supported by National Science Foundation under grant CNS-1850523. Z. Huang is with Oklahoma State University, Stillwater, OK 74075. E-mail: zonghao.huang@okstate.edu R. Hu, Y. Guo, and Y. Gong are with The University of Texas at San Antonio, San Antonio, TX 78249. Email:\{rui.hu@my., yuanxiong.guo@, yanmin.gong@\}utsa.edu E. Chan-Tin is with Loyola University Chicago, Chicago, IL 60660. E-mail: chantin@cs.luc.edu}
		\thanks{This paper has supplementary downloadable material available at https://ieeexplore.ieee.org provided by the author. This includes a PDF file containing mathematical proofs. 
		}
		
		% <-this % stops a space
	}

	%	\markboth{IEEE TRANSACTIONS ON XXXXX,~Vol.~XX, No.~X, XXXX~XXXX}%
	%	{Zonghao \MakeLowercase{\textit{et al.}}: Bare Demo of IEEEtran.cls for Computer Society Journals}
	
	\maketitle
	
	\begin{abstract}
		
		%distributed learning is important but privacy risk. we focus on differential private distributed learning(overall). we use admm to enable distributed learning and differential privacy guarantee. we find that it is not easy. so we improve it.
		
		Alternating direction method of multipliers (ADMM) is a widely used tool for machine learning in distributed settings, where a machine learning model is trained over distributed data sources through an interactive process of local computation and message passing. Such an iterative process could cause privacy concerns of data owners. The goal of this paper is to provide differential privacy for ADMM-based distributed machine learning. 
		%and we focus on a class of convex machine learning problems which can be formulated as regularized empirical risk minimization. 
		Prior approaches on differentially private ADMM exhibit low utility under high privacy guarantee and assume the objective functions of the learning problems to be smooth and strongly convex. To address these concerns, we propose a novel differentially private ADMM-based distributed learning algorithm called DP-ADMM, which combines an approximate augmented Lagrangian function with time-varying Gaussian noise addition in the iterative process to achieve higher utility for general objective functions under the same differential privacy guarantee. We also apply the moments accountant method to analyze the end-to-end privacy loss. The theoretical analysis shows that DP-ADMM can be applied to a wider class of distributed learning problems, is provably convergent, and offers an explicit utility-privacy tradeoff. 
		%We provide rigorous convergence analysis and utility bound of our approach. 
		To our knowledge, this is the first paper to provide explicit convergence and utility properties for differentially private ADMM-based distributed learning algorithms. The evaluation results demonstrate that our approach can achieve good convergence and model accuracy under high end-to-end differential privacy guarantee.

	\end{abstract}
	
	\begin{IEEEkeywords}
		Machine learning, ADMM, distributed algorithms, privacy, differential privacy, and moments accountant.
	\end{IEEEkeywords}
	
	\IEEEpeerreviewmaketitle
	
	\section{introduction} \label{sec:intr}

	% distributed learning is important
	% ADMM and privacy concern
	% privacy-preserving mechanism is needed, existing works
	% our work
	
	\IEEEPARstart{D}{istributed} machine learning is a widely adopted approach due to the high demand of large-scale and distributed data processing. It allows multiple entities to keep their datasets unexposed, and meanwhile to collaborate in a common learning objective (usually formulated as a regularized empirical risk minimization problem) by iterative local computation and message passing. Therefore, distributed machine learning helps to reduce computational burden and improves both robustness and scalability of data processing. 
	%
	%Thus, distributed machine learning can provide a degree of data privacy, help reduce the computational burden, and improve the scalability of data processing. 
	%
	As pointed out in recent studies \cite{ZhKh18,ZhangKh18}, existing approaches to decentralizing an optimization problem mainly consist of subgradient-based algorithms \cite{NeOl08,NeOz09}, alternating direction method of multipliers (ADMM) based algorithms \cite{BoPa11,LingRi14,ShiLing14,ZhangKw14}, and composite of sub-gradient descent and ADMM \cite{BiHa14}. It has been shown that ADMM-based algorithms can converge at the rate of $O(1/t)$ while subgradient-based algorithms typically converge at the rate of $O(1/\sqrt{t})$, where $t$ is the number of iterations \cite{WeiOz12}. Therefore, ADMM has become a popular method for designing distributed versions of a machine learning algorithm \cite{BoPa11,MoXa13,ZhangKw14}, and our work focuses on ADMM-based distributed learning. 
	
	With ADMM, the learning problem is divided into several sub-problems solved by agents independently and locally, and only intermediate parameters need to be shared.  However, the iterative process of ADMM involves privacy leakage, and the adversary can obtain the sensitive information from the shared model parameters as shown in \cite{ShSt17,FrJh15}. Thus, we aim to limit the privacy leakage during the iterative process of ADMM using differential privacy. Differential privacy is a widely used privacy definition \cite{DwMc06,huang2017differential,WangZLWQR18} and can be guaranteed in ADMM through adding noise to the exchanged messages. However, in existing studies on ADMM-based distributed learning with differential privacy \cite{ZhZh17,ZhKh18,ZhangKh18,guo2018practical,DiEr19}, noise addition would disrupt the learning process and severely degrade the performance of the trained model, especially when large noise is needed to provide high privacy protection. Besides, their privacy-preserving algorithms only apply to the learning problems with both smoothness and strongly convexity assumptions about the objective functions. Such weaknesses and limitations motivate us to explore further in this area.
	
	In this paper,  we mainly focus on using ADMM to enable distributed learning while guaranteeing differential privacy, and propose a novel differentially private ADMM-based distributed learning algorithm called DP-ADMM, which has good convergence properties, low computational cost, and an explicit and improved utility-privacy tradeoff, and can be applied to a wide class of distributed learning problems.
	%
	%In DP-ADMM, we employ the Gaussian mechanism with time-varying noise magnitude to guarantee $(\epsilon,\delta)$-differential privacy in each iteration, and employ the moments accountant method \cite{AbCh16} to bound the end-to-end privacy loss. 
	%
	The key algorithmic feature of DP-ADMM is the combination of an approximate augmented Lagrangian function and time-varying Gaussian noise addition in the iterative process, which enables the algorithm to be noise-resilient and provably convergent. The moments accountant method \cite{AbCh16} is used to analyze the end-to-end privacy guarantee of DP-ADMM. 
	%We also demonstrate that our algorithm can be applied to the distributed learning problems with non-smooth objective functions. 
	%
	We also rigorously analyze the convergence rate and utility bound of our approach. To our knowledge, this is the first paper to provide explicit convergence and utility properties for differentially private ADMM-based distributed learning algorithms.

	The main contributions of this paper are summarized as follows:
	\begin{enumerate}
		\item 
		% 	\textcolor{red}{We design an approximate augmented Lagrangian function for ADMM-based distributed algorithm development and use Gaussian mechanisms with time-varying noise variance to preserve differential privacy during the distributed learning process.}
		
		% 	The approximation of augmented Lagrangian function enables the local primal variable update function to be smooth. Hence compared with prior studies in differentially private ADMM algorithms, our approach has wider applicability (i.e., can be applied to non-smooth problem) and has fast computation (i.e., can obtain closed-form solution to local primal update step). Furthermore, the time-varying variance in Gaussian mechanism ensures that our approach is more noise-resilient, which is significant in learning problems.
		
		We design a novel differentially private ADMM-based distributed learning algorithm called DP-ADMM, which combines an approximate augmented Lagrangian function with time-varying Gaussian noise addition in the iterative process to achieve higher utility for more general objective functions than prior works under the same differential privacy guarantee.
		
		\item Different from previous studies providing only differential privacy guarantee for each iteration, we use the moments accountant method to analyze the total privacy loss and provide a tight end-to-end differential privacy guarantee for DP-ADMM.
		
		\item We provide rigorous convergence and utility analysis of the proposed DP-ADMM. To our knowledge, this is the first paper to provide explicit convergence and utility properties for differentially private ADMM-based distributed learning algorithms.
		
		\item We conduct extensive simulations based on real-world datasets to validate the effectiveness of DP-ADMM in distributed learning settings. %\hl{under moderate total privacy loss.}
	\end{enumerate}
	
	The rest of the paper is organized as follows. In Section~\ref{sec:model}, we present our problem statement. In Section~\ref{sec:appro}, we describe a differentially private standard ADMM-based algorithm and propose our DP-ADMM. In Section~\ref{sec:priv} and Section~\ref{sec:conv}, we theoretically analyze our privacy guarantee and convergence and utility properties of DP-ADMM, respectively. The numerical results of DP-ADMM based on real-world datasets are shown in Section~\ref{sec:impl}. Section~\ref{sec:rela} discusses the related work, and Section~\ref{sec:conc} concludes the paper. 
	
	%---------------------------------------------
	\section{Problem Statement} \label{sec:model}
	%---------------------------------------------
	In this section, we first introduce the problem setting. Then we present the standard ADMM-based distributed learning algorithm and discuss the associated privacy concern. A summary of notations used in this paper is listed in Table~\ref{table1}. 
	
	%Table \uppercase\expandafter{\romannumeral1}.
	
	\subsection{Problem Setting}
	
	We consider a set of agents $[n] := \{1, \ldots, n\}$ and a central aggregator. Each agent $i \in [n]$ has a private training dataset $\mathcal{D}_i := \{(\boldsymbol{a}_{i,j}, \boldsymbol{b}_{i,j}) : \forall j \in [m_i] \}$, where $m_i$ is the number of training samples in the dataset $\mathcal{D}_i$, $\boldsymbol{a}_{i,j} \in \mathbb{R}^d $ is the $d$-dimensional data feature vector of the $j$-th training sample, and $\boldsymbol{b}_{i,j} \in \mathbb{R}^p$ is the corresponding $p$-dimensional data label. In this paper, we consider a star network topology where each agent can communicate with the central aggregator and the aggregator is responsible for message passing and aggregation. Note that our approach can be generalized to other network topologies where agents are connected with their neighbors without a central aggregator, as discussed in \cite{ZhZh17,ZhKh18,ZhangKh18}.
	%% %For each training sample $(x_{j}^{(i)}, y_{j}^{(i)})$, we assume without loss of generality that $\norm{x_{j}^{(i)}}_2 \leq 1$ since any feature vector can be normalized to enforce this assumption. 
	% We assume without loss of generality that each training sample lies in a unit ball. 
	
	%---------------------------------------------
	\begin{table}
		\small
		\caption{List of notations}
		\begin{center}\label{table1}
			\renewcommand\arraystretch{1}
			\begin{tabular}{|c|c|}
				\hline
				%$\mathcal{D}_i$& Dataset of agent $i$ \\ \hline
				$\boldsymbol{a}_{i,j}$ & Data feature vector \\ \hline
				$\boldsymbol{b}_{i, j}$ & Data label \\ \hline
				$\ell(\cdot)$ & Loss function \\\hline
				$R(\cdot)$& Regularizer function \\ \hline
				$\lambda$ &  Regularizer parameter\\\hline
				$\ell^{'}(\cdot)$& Subgradient of loss function\\\hline
				$R^{'}(\cdot)$& Subgradient of regularizer\\\hline
				${\nabla} \ell(\cdot)$ &  Gradient of loss function\\ \hline
				${\nabla} R(\cdot)$ & Gradient of regularizer\\ \hline
				$\boldsymbol{w}$& Global machine learning model\\\hline
				$\boldsymbol{w}_i$ & Local learning model from agent $i$\\\hline
				$ \boldsymbol{\gamma}_i $& Dual variable from agent $i$ \\ \hline
				$\rho $&  Penalty parameter \\ \hline
				$\mathcal{L}_{\rho}(\cdot)$ & Augmented Lagrangian function \\ \hline
				$\hat{\mathcal{L}}_{\rho, k}(\cdot)$ & Approximate augmented Lagrangian function \\ \hline
				$\boldsymbol{w}_i^k$& Primal variable from agent $i$ in $k$-th iteration \\ \hline
				$\boldsymbol{\tilde{w}}_i^{k}$& Noisy version of $\boldsymbol{w}_i^k$ after perturbation\\ \hline
				$\boldsymbol{\gamma}_i^k$ & Dual variable from agent $i$ in $k$-th iteration \\ \hline
				$\boldsymbol{w}^k$& Global variable in $k$-th iteration   \\ \hline
				$\boldsymbol{\xi}_i^{k}$& Sampled noise from agent $i$ in $k$-th iteration \\ \hline
				$\sigma_{i}^{2}$ & Constant variance of Gaussian mechanism \\ \hline
				$\eta_i^{k}$& Time-varying step size in $k$-th iteration \\ \hline
				$\sigma_{i,k}^{2}$& Time-varying variance of Gaussian mechanism \\ \hline
				%$c_w$& $L_2$-norm of the optimal classifier \\ \hline
				%$\boldsymbol{w}^{*}$& Optimal classifier \\ \hline
				%			${\nabla}^2 \ell(\cdot)$ & Second-order derivative of $\ell(\cdot)$\\ \hline
				%			${\nabla}^2 R(\cdot)$ & Second-order derivative of $R(\cdot)$\\ \hline
				%$\mathcal{D}_i^{'}$ & Neighbouring dataset of $\mathcal{D}_i$ \\ \hline
			\end{tabular}
			\vspace{-0.3in}
		\end{center}
	\end{table}
	%---------------------------------------------
	
	The goal of our problem is to train a supervised learning model on the aggregated dataset $\{\mathcal{D}_i\}_{i \in [n]}$, which enables predicting a label for any new data feature vector. The learning objective can be formulated as the following regularized empirical risk minimization problem:
	\begin{align}
	\min_{\boldsymbol{w}} \quad \sum_{i = 1}^{n} \sum_{j=1}^{m_i }\frac{1}{m_i} \ell(\boldsymbol{a}_{i,j}, \boldsymbol{b}_{i,j},\boldsymbol{w})  +\lambda R(\boldsymbol{w}) \label{pro:1},
	\end{align}
	where $\boldsymbol{w} \in \mathbb{R}^{d\times p }$ is the trained machine learning model, $\ell (\cdot):\mathbb{R}^d  \times \mathbb{R}^p  \times \mathbb{R}^{d \times p }  \rightarrow \mathbb{R} $ is the loss function used to measure the quality of the trained model, $R(\cdot) $ refers to the regularizer function introduced to prevent overfitting, and $\lambda > 0$ is the regularizer parameter controlling the impact of regularizer. Note that the problem formulation \eqref{pro:1} can represent a wide range of machine learning tasks by choosing different loss functions. For instance, the loss function of binary logistic regression is:
	\begin{equation}
	\ell(\boldsymbol{a}_{i,j}, \boldsymbol{b}_{i,j}, \boldsymbol{w}) = \ln \big(1+\exp(-\boldsymbol{b}_{i,j}\boldsymbol{w}^{\intercal}\boldsymbol{a}_{i,j})\big),  \label{eq:logi}
	\end{equation}
	and the loss function of multi-class logistic regression is:
	\begin{equation}
	\ell(\boldsymbol{a}_{i,j}, \boldsymbol{b}_{i,j}, \boldsymbol{w}) = \sum_{h = 1}^{p} \boldsymbol{b}_{i,j}^{(h)} \ln \bigg( \frac{\sum_{l=1}^{p} \exp(\boldsymbol{w}^{{(l)}^{\intercal}} \boldsymbol{a}_{i,j})}{\exp(\boldsymbol{w}^{{(h)}^{\intercal}} \boldsymbol{a}_{i,j}) }\bigg). 
	\end{equation}
	
	In this paper, we assume that the loss function $\ell (\cdot)$ and the regularizer function $R(\cdot)$ are both convex but not necessarily smooth. Throughout this paper, we use $\ell^{'}(\cdot)$ and $R^{'}(\cdot)$ to denote the sub-gradient of $\ell(\cdot)$ and $R(\cdot)$ respectively. When we consider smooth functions, we use $\nabla \ell(\cdot)$ and $\nabla R(\cdot)$ instead.

	\subsection{ADMM-Based Distributed Learning Algorithm}
	
	To apply ADMM, we re-formulate the problem \eqref{pro:1} as:
	\begin{subequations}\label{eq:dispro}
		\begin{align}
		\min_{\{\boldsymbol{w}_i\}_{i \in [n]}} \quad & \sum_{i = 1}^{n} \bigg(\sum_{j=1}^{m_i }\frac{1}{m_i} \ell(\boldsymbol{a}_{i,j}, \boldsymbol{b}_{i,j},\boldsymbol{w}_i) + \frac{\lambda}{n} R(\boldsymbol{w}_i)\bigg), \label{objective} \\
		\text{s.t.} \quad \quad & \boldsymbol{w}_i = \boldsymbol{w}, i = 1, \ldots, n, \label{prob:modified}
		\end{align}
	\end{subequations}
	where $\boldsymbol{w}_i \in \mathbb{R}^{d \times p}$ is the local model, and $\boldsymbol{w} \in \mathbb{R}^{d \times p}$ is the global one. The objective function \eqref{objective} is decoupled and each agent only needs to minimize the sub-problem associated with its dataset. Constraints \eqref{prob:modified} enforce that all the local models reach consensus finally.
	
	In standard ADMM, the augmented Lagrangian function associated with the problem \eqref{eq:dispro} is:
	\begin{equation}
	\mathcal{L}_{\rho}(\boldsymbol{w}, \{\boldsymbol{w}_i\}_{i\in[n]}, \{\boldsymbol{\gamma}_i\}_{i\in[n]}) =   \sum_{i = 1}^{n} \mathcal{L}_{\rho,i}( \boldsymbol{w}_i,\boldsymbol{w}, \boldsymbol{\gamma}_i),  \label{eq:lag}   
	\end{equation}
	where
	\begin{equation}
	\begin{split}
	\mathcal{L}_{\rho,i}( \boldsymbol{w}_i,\boldsymbol{w}, \boldsymbol{\gamma}_i) =  & \sum_{j=1}^{m_i }\frac{1}{m_i} \ell(\boldsymbol{a}_{i,j}, \boldsymbol{b}_{i,j},\boldsymbol{w}_i)+\frac{\lambda}{n} R(\boldsymbol{w}_i) \\& \quad - \big\langle \boldsymbol{\gamma}_i, \boldsymbol{w}_i - \boldsymbol{w} \big\rangle  + \frac{\rho}{2} {\Vert \boldsymbol{w}_i - \boldsymbol{w}\Vert}^2.  \label{eq:lag2} 
	\end{split}
	\end{equation}	
	In \eqref{eq:lag2}, $\{\boldsymbol{\gamma}_i\}_{i \in [n]} \in \mathbb{R}^{d \times p \times n}$ are the dual variables associated with constraints~\eqref{prob:modified} and $\rho > 0$ is the penalty parameter. The standard ADMM solves the problem \eqref{eq:dispro} in a Gauss-Seidel manner by minimizing \eqref{eq:lag} w.r.t. $\{\boldsymbol{w}_i\}_{i \in [n]}$ and $\boldsymbol{w}$ alternatively followed by a dual update of $\{\boldsymbol{\gamma}_i\}_{i\in[n]}$. The ADMM-based distributed algorithm is shown in Algorithm~\ref{ag:2}.
	\begin{algorithm}[htbp]
		%	\setstretch{1.05}
		\caption{ADMM-Based Distributed Algorithm}\label{ag:2}
		\begin{algorithmic}[1]
			\STATE Initialize $\boldsymbol{w}^{0}$, $\{\boldsymbol{w}_i^{0}\}_{i \in [n]}$, and $\{\boldsymbol{\gamma}_i^0\}_{i \in [n]}$;
			
			\FOR{$k =  1, 2, \dots, t$}
			\FOR{$i = 1, 2, \dots, n$}
			\STATE  $\boldsymbol{w}_i^{k} \gets \argmin_{\boldsymbol{w}_i}  \mathcal{L}_{\rho,i}(\boldsymbol{w}_i, \boldsymbol{w}^{k-1}, \boldsymbol{\gamma}_i^{k-1})$;
			\ENDFOR
			
			\STATE   $\boldsymbol{w}^{k}  \gets \frac{1}{n}\sum_{i=1}^n\boldsymbol{w}^{k}_i -\frac{1}{n}\sum_{i=1}^n\boldsymbol{\gamma}^{k-1}_i/\rho$;
			
			\FOR{$i = 1, 2, \dots, n$}
			\STATE   $\boldsymbol{\gamma}_{i}^{k} \gets \boldsymbol{\gamma}_{i}^{k-1} - \rho (\boldsymbol{w}_i^{k} - \boldsymbol{w}^{k})$. 
			\ENDFOR
			\ENDFOR
		\end{algorithmic}
	\end{algorithm}
	\vspace*{-0.2in}
	
	\subsection{Privacy Concern}
	
	In Algorithm~\ref{ag:2}, the intermediate parameters $\{\boldsymbol{w}_i^{k}\}_{i \in [n], k \in [t]}$ need to be shared with the aggregator, which may reveal the agents' private information as demonstrated by model inversion attacks \cite{fredrikson2015model}. Thus, we need to develop privacy-preserving methods to control such information leakage. The main goal of this paper is to provide privacy protection against inference attacks from an adversary, who tries to infer sensitive information about the agents' private datasets from the shared messages. We assume that the adversary can neither intrude into the local datasets nor have access to the datasets directly. The adversary could be an outsider who eavesdrops the shared messages, or the honest-but-curious aggregator who follows the protocol honestly but tends to infer the sensitive information. We do not assume any trusted third party, thus a privacy-preserving mechanism should be applied locally by each agent to provide privacy protection.  
	
	%nor pollute the datasets. The communication channels between the agents and the aggregator are reliable, thus the integrity of the transferred information would not be violated.
	
	In order to provide privacy guarantee against such attacks, we define our privacy model formally by the notion of differential privacy \cite{DwMc06}. 
	%Differential privacy is a widely used definition of privacy, providing the protection of the dataset from attackers with arbitrary knowledge. %Specifically, differential privacy guarantees that when given a sanitized result from each individual, the adversary does not know whether the data target is in the dataset or not. 
	% There are two definitions including pure differential privacy: $\epsilon$-differential privacy and relaxed differential privacy: $(\epsilon,\delta)$-differential privacy. Compared with pure differential privacy, $(\epsilon,\delta)$-differential privacy is used for the analysis of the privacy guarantee of Gaussian mechanisms \cite{DwKe06}, by which generated noise has the same distribution model as the natural noise. Especially, $(\epsilon,\delta)$-differential privacy is preferred in the application of advanced composition theorem, for example, the adaptive algorithm, which takes the output of previous step as the input of the current step.
	Specifically, we adopt the $(\epsilon,\delta)$-differential privacy defined as follows:
	
	\begin{definition}[$(\epsilon,\delta)$-Differential Privacy] \label{de:1}
		A randomized mechanism $\mathcal{M}$ is $(\epsilon,\delta)$-differentially private if for any two neighbouring datasets $\mathcal{D}$ and $\mathcal{D}^{'}$ differing in only one tuple, and for any subsets of outputs $\mathcal{O} \subseteq$ range($\mathcal{M}$):
		\begin{equation}
		\Pr[\mathcal{M}(\mathcal{D})\in\mathcal{O}] \leq e^{\epsilon} \cdot \Pr[\mathcal{M}(\mathcal{D}^{'})\in\mathcal{O}] + \delta, 
		\end{equation}
		which means, with probability of at least $1-\delta$, the ratio of the probability distributions for two neighboring datasets is bounded by $e^{\epsilon}$.
	\end{definition}
	
	In Definition~\ref{de:1}, the parameters $\delta$ and $\epsilon$ are privacy budgets indicating the strength of privacy protection from the mechanism. Smaller $\epsilon$ or $\delta$ indicates better privacy protection. Gaussian mechanism is a common randomization method used to guarantee $(\epsilon,\delta)$-differential privacy, where noise sampled from normal distribution is added to the output. In this paper, we use $\mathcal{MN}_{d,p}(0, \sigma^{2}\boldsymbol{\mathrm{I}}_d, \sigma^{2}\boldsymbol{\mathrm{I}}_p )$ to denote the matrix normal distribution with variance $\sigma^{2}$.
	
	\section{ADMM with Differential Privacy} \label{sec:appro}
	
	In this section, we achieve differential privacy under the framework of ADMM. First, we introduce an intuitive method by directly combining standard ADMM and primal variable perturbation (PVP) and discuss the weaknesses of this method. Then we propose our new approach to achieving differential privacy in ADMM with an improved utility-privacy tradeoff.
	
	\subsection{ADMM with Primal Variable Perturbation (PVP)}
	
	As described in Section~\ref{sec:model}, we need to use a local privacy-preserving mechanism in order to guarantee $(\epsilon,\delta)$-differential privacy for each agent. An intuitive way to achieve this goal is to combine the primal variable perturbation mechanism (PVP) and standard ADMM directly as proposed in \cite{ZhZh17}. Specifically, as given in Algorithm~\ref{ag:3}, at the $k$-th iteration, after obtaining the local primal variable $\boldsymbol{w}_i^{k}$, we apply Gaussian mechanism with a pre-defined variance $\sigma_{i}^{2}$ to perturb it and share the noisy primal variable $\tilde{\boldsymbol{w}}_i^{k}$, which can guarantee differential privacy. According to \cite{DwRo14,ChMo11}, by assuming the smoothness of loss function $l(\cdot)$ and regularizer function $R(\cdot)$, strongly convexity of regularizer $R(\cdot)$, and the bounded $l_2$ norm of the derivative of loss function by $c_1$, the $l_2$ sensitivity of $\boldsymbol{w}_i^k$ update function in standard ADMM is $2 c_1 /\big(m_i (\lambda/n+\rho)\big)$ as proved in Appendix \ref{ap:sen}. Therefore, the noise magnitude $\sigma_{i} = 2 c_1  \sqrt{2\ln(1.25/\delta)}/\big(( \lambda/n+\rho)m_i \epsilon\big) $ can achieve $(\epsilon,\delta)$-differential privacy in each iteration. 
	
	\begin{algorithm}
		%		\setstretch{1.05}
		\caption{ADMM with PVP}\label{ag:3}
		\begin{algorithmic}[1]
			\STATE Initialize  $\boldsymbol{w}^{0} $, $\{\boldsymbol{w}_i^{0}\}_{i \in [n]}$, and $\{\boldsymbol{\gamma}_i^0\}_{i \in [n]} $.
			
			\FOR{$k =  1, 2, \dots, t$}
			\FOR{$i = 1, 2, \dots, n$}
			\STATE  $\boldsymbol{w}_i^{k} \gets \argmin_{\boldsymbol{w}_i}  \mathcal{L}_{\rho,i}(\boldsymbol{w}_i, \boldsymbol{w}^{k-1}, \boldsymbol{\gamma}_i^{k-1})$.
			\STATE  $\boldsymbol{\tilde{w}}_i^{k} \gets \boldsymbol{w}_i^{k} + \mathcal{MN}_{d,p}(0, \sigma_{i}^{2}\boldsymbol{\mathrm{I}}_d, \sigma_{i}^{2}\boldsymbol{\mathrm{I}}_p )$.
			\ENDFOR
			\STATE  $\boldsymbol{w}^{k}  \gets \frac{1}{n}\sum_{i=1}^n\boldsymbol{\tilde{w}}^{k}_i -\frac{1}{n}\sum_{i=1}^n\boldsymbol{\gamma}^{k-1}_i/\rho$.
			\FOR{$i = 1, 2, \dots, n$}
			\STATE $\boldsymbol{\gamma}_i^{k } \gets \boldsymbol{\gamma}_{i}^{k-1} - \rho (\boldsymbol{\tilde{w}}_i^{k} - \boldsymbol{w}^{k})$. 
			\ENDFOR
			\ENDFOR
		\end{algorithmic}
	\end{algorithm}
	
	However, the added noise from the perturbation mechanism would disrupt the learning process, break the convergence property of the iterative process, and lead to a trained model with poor performance. This is especially the case when the privacy budget is small. Specifically, when the iteration number $k$ is large, the trained model would keep changing dramatically due to the existence of large noise. Besides, the above perturbation method can only be applied when the objective function is smooth and the regularizer is strongly convex \cite{ChMo11,ZhZh17}. In order to address such problems, we need to consider an alternative way to preserving differential privacy of ADMM-based distributed learning algorithms.

	\subsection{Our Approach}
	
	Our approach is inspired by the intuition that it is not necessary to solve the problem up to a very high precision in each iteration in order to guarantee the overall convergence. In our approach, instead of using the exact augmented Lagrangian function, we employ its first-order approximation with a scalar $l_2$-norm prox-function. Here we define:
	\begin{equation}
	\begin{split}
	&\hat{\mathcal{L}}_{\rho, k, i}(\boldsymbol{w}_i,\boldsymbol{\tilde{w}}_i^{k-1}, \boldsymbol{w}, \boldsymbol{\gamma}_i) \\   &\quad = \sum_{j=1}^{m_i }\frac{1}{m_i} \ell(\boldsymbol{a}_{i,j}, \boldsymbol{b}_{i,j},\boldsymbol{\tilde{w}}_i^{k-1}) + \frac{\lambda}{n} R(\boldsymbol{\tilde{w}}_i^{k-1}) \\& \quad + \big\langle \sum_{j=1}^{m_i }\frac{1}{m_i} \ell^{'}(\boldsymbol{a}_{i,j}, \boldsymbol{b}_{i,j},\boldsymbol{\tilde{w}}_i^{k-1}) + \frac{\lambda}{n} R^{'}(\boldsymbol{\tilde{w}}_i^{k-1}),  \boldsymbol{w}_i - \boldsymbol{\tilde{w}}_i^{k-1} \big\rangle
	\\ &\quad - \big\langle \boldsymbol{\gamma}_i ,\boldsymbol{w}_i- \boldsymbol{w} \big\rangle + \frac{\rho}{2} {\Vert\boldsymbol{w}_i- \boldsymbol{w}\Vert}^2 + \frac{{\Vert \boldsymbol{w}_i - \boldsymbol{\tilde{w}}_i^{k-1}\Vert}^2}{2\eta_i^{k}}, \label{eq:apro}
	\end{split}
	\end{equation}
	where $\eta_i^{k} \in \mathbb{R}$ is the time-varying step size, and it decreases as the iteration number $k$ increases.
	
	The proposed approximate augmented Lagrangian function used in our approach is defined by:
	\begin{equation}
	\begin{split}
	&\hat{\mathcal{L}}_{\rho,k}(\{\boldsymbol{w}_i\}_{i\in[n]},\{\boldsymbol{\tilde{w}}_i^{k-1}\}_{i\in[n]},\boldsymbol{w}, \{\boldsymbol{\gamma}_i\}_{i \in [n]} ) \\    
	&\quad\quad\quad\quad\quad\quad\quad\quad = \sum_{i = 1}^{n} \hat{\mathcal{L}}_{\rho, k, i}(\boldsymbol{w}_i,\boldsymbol{\tilde{w}}_i^{k-1}, \boldsymbol{w}, \boldsymbol{\gamma}_i).\label{eq:aprolag}
	\end{split}
	\end{equation}
	Our approach minimizes \eqref{eq:aprolag} in a Gauss-Seidel manner and adds zero-mean Gaussian noise with time-varying variance $\sigma_{i,k}^{2}$ that decreases as the iteration number $k$ increases.  
	
	The resulting ADMM steps that provide differential privacy are as follows:
	\begin{subequations}
		\begin{align}
		\boldsymbol{w}_i^{k} = & \argmin_{\boldsymbol{w}_i} ~\hat{\mathcal{L}}_{\rho, k, i}(\boldsymbol{w}_i,\boldsymbol{\tilde{w}}_i^{k-1}, \boldsymbol{w}^{k-1}, \boldsymbol{\gamma}_i^{k-1}),\label{eqite1}  \\\boldsymbol{\tilde{w}}_i^{k}= &\boldsymbol{w}_i^{k} + \mathcal{MN}_{d,p}(0, \sigma_{i,k}^{2}\boldsymbol{\mathrm{I}}_d, \sigma_{i,k}^{2}\boldsymbol{\mathrm{I}}_p )  , \label{eqper}\\
		\boldsymbol{w}^{k} = & \frac{1}{n}\sum_{i=1}^n\boldsymbol{\tilde{w}}^{k}_i - \frac{1}{n}\sum_{i=1}^n\boldsymbol{\gamma}^{k-1}_i/\rho, \label{eqite3}\\
		\boldsymbol{\gamma}_i^{k} = & \boldsymbol{\gamma}_{i}^{k-1} - \rho (\boldsymbol{\tilde{w}}_i^{k} - \boldsymbol{w}^{k}), \label{eqite2}
		\end{align}
	\end{subequations}
	where \eqref{eqite3} is computed at the aggregator while \eqref{eqite1}, \eqref{eqper} and \eqref{eqite2} are performed at each agent. 
	
	The details are given in Algorithm~\ref{ag:1}. The central aggregator firstly initializes the global variable $\boldsymbol{w}^0$, and the agents also initialize their noisy primal variables $\{ \boldsymbol{\tilde{w}}_i^0 \}_{i \in [n]}$ and dual variables $\{ \boldsymbol{\gamma}_i^0 \}_{i \in [n]}$. At the beginning of each iteration $k$, each agent $i$ first samples a zero-mean Gaussian noise $\boldsymbol{\xi}_i^{k}$ with variance $\sigma_{i,k}^{2}$ and updates the noisy primal variable $ \boldsymbol{\tilde{w}}_i^{k} $ based on \eqref{eqite1} and \eqref{eqper}. Then the aggregator receives the noisy primal variables $\{\boldsymbol{\tilde{w}}_i^{k} \}_{i \in [n]}$ and the dual variables $\{ \boldsymbol{\gamma}_i^{k-1} \}_{i \in [n]}$ from the agents, and uses them to update the global variable $\boldsymbol{w}^{k} $ according to \eqref{eqite3}. After that, agents receive the updated global variable $ \boldsymbol{w}^{k} $ from the aggregator and continue to update the dual variables $\{ \boldsymbol{\gamma}_{i}^{k} \}_{i \in [n]}$ by \eqref{eqite2}. The iterative process will continue until reaching $t$ iterations. 
	%The agents and the aggregator iteratively exchanges variables until the end of their communication. 
	
	\begin{algorithm}[htbp] 
		%\setstretch{1.05}
		\caption{DP-ADMM}\label{ag:1}
		\begin{algorithmic}[1]
			\STATE Initialize  $\boldsymbol{w}^{0} $, $\{\tilde{\boldsymbol{w}}_i^{0}\}_{i \in [n]}$, and $\{\boldsymbol{\gamma}_i^0\}_{i \in [n]} $.
			\FOR{$k = 1, 2, \dots, t$}
			\FOR{$i = 1, 2, \dots, n$}
			\STATE   $\boldsymbol{w}_i^{k} \gets  \argmin_{\boldsymbol{w}_i} ~\hat{\mathcal{L}}_{\rho, k, i}(\boldsymbol{w}_i,\boldsymbol{\tilde{w}}_i^{k-1}, \boldsymbol{w}^{k-1}, \boldsymbol{\gamma}_i^{k-1})$.
			\STATE   $\boldsymbol{\xi}_i^{k} \gets \mathcal{MN}_{d,p}(0, \sigma_{i,k}^{2}\boldsymbol{\mathrm{I}}_d, \sigma_{i,k}^{2}\boldsymbol{\mathrm{I}}_p ).$
			
			\STATE  $\boldsymbol{\tilde{w}}_i^{k} \gets \boldsymbol{w}_i^{k} + \boldsymbol{\xi}_i^{k}.$
			\ENDFOR
			\STATE   $\boldsymbol{w}^{k}  \gets \frac{1}{n}\sum_{i=1}^n\boldsymbol{\tilde{w}}^{k}_i -\frac{1}{n}\sum_{i=1}^n\boldsymbol{\gamma}^{k-1}_i/\rho$.
			\FOR{$i = 1, 2, \dots, n$}
			\STATE  $\boldsymbol{\gamma}_i^{k} \gets \boldsymbol{\gamma}_{i}^{k-1} - \rho (\boldsymbol{\tilde{w}}_i^{k} - \boldsymbol{w}^{k})$. 
			\ENDFOR
			\ENDFOR
		\end{algorithmic}
	\end{algorithm}
	
	Algorithm~\ref{ag:1} is different from Algorithm~\ref{ag:3} in three aspects. Firstly, the approximate augmented Lagrangian function used in this approach replaces the objective function with its first-order approximation at $\boldsymbol{\tilde{w}}_i^{k-1}$, which is similar to the stochastic mirror descent \cite{NeJu09}. This approximation enforces the smoothness of the Lagrangian function and makes it easy to solve~\eqref{eqite1}. Even when the objective function is non-smooth, we can still get a closed-form solution to \eqref{eqite1}, which achieves fast computation. More importantly, this approximation can lead to a bounded $l_2$ sensitivity in differential privacy guarantee without the limitation that the objective function should be smooth and strongly convex. Thus our approach can be applied to any convex problems. We demonstrate this in Section \ref{sec:priv}.
	
	Secondly, similar to linearized ADMM \cite{YaYu13,LinLiu11}, there is an $l_2$-norm prox-function ${\Vert \boldsymbol{w}_i - \boldsymbol{\tilde{w}}_i^{k-1}\Vert}^2$ but scaled by $1/2\eta_i^{k}$ added in \eqref{eq:apro}, where the step size $\eta_i^{k}$ decreases when the iteration number $k$ increases. Such additional part can guarantee the consistency between the updated model $ \boldsymbol{w}_i^{k}$ and the previous one, especially when $k$ is large. Thus, as $k$ increases, the updated model would change more smoothly. Note that the time-varying step-size $\eta_i^{k}$ is significant for the overall convergence guarantee. In Section~\ref{sec:conv}, we will define $\eta_i^{k}$ and show its importance in algorithmic convergence. 
	
	Lastly, the variance $\sigma_{i,k}^{2}$ of Gaussian mechanism used in Algorithm~\ref{ag:1} is time-varying rather than constant as adopted in prior studies \cite{AbCh16}. It decreases when the iteration number $k$ increases. The motivation of using Gaussian mechanism with time-varying variance is to mitigate the negative effect from noise and guarantee the convergence property of our approach. As explained before, the added noise would disrupt the learning process. By using the Gaussian mechanism with time-varying variance, the added noise will decrease when the iteration number $k$ increases. Therefore, the negative affect from the added noise will be mitigated, enabling the updates to be stable. In Section \ref{sec:priv}, we would define the magnitude of time-varying variance $\sigma_{i,k}^{2}$ to achieve differential privacy.
	
	%\hl{In addition, the time-varying variance is necessary to guarantee that each iteration of our approach is $(\epsilon, \delta)$-differentially private.} As will be shown in Section~\ref{sec:priv}, the $l_2$ sensitivity of the primal variable update changes with the time-varying step size $\eta_i^{k}$: the larger $k$ is, the smaller the step size $\eta_i^{k}$ is, and the lower the $l_2$ sensitivity becomes. Since the variance depends on the time-varying $l_2$ sensitivity, Gaussian mechanism with time-varying variance is required in our approach.} 
	%In Section~\ref{sec:priv}, we will select the appropriate $\sigma_{i,k}^{2}$ and prove the end-to-end privacy guarantee of Algorithm~\ref{ag:1}. 
	
	\section{Privacy Guarantee}\label{sec:priv}
	
	In this section, we analyze the privacy guarantee of the proposed DP-ADMM. In DP-ADMM, the shared messages $\{\boldsymbol{\tilde{w}}_i^{k}\}_{ k \in[t]}$ may reveal the sensitive information of agent $i$, which has been discussed in Section \ref{sec:model}. Thus, we need to demonstrate that DP-ADMM guarantees differential privacy with outputs $\{\boldsymbol{\tilde{w}}_i^{k} \}_{k \in [t]}$. We first estimate the $l_2$ norm sensitivity of $\boldsymbol{w}_i^{k}$ update function, then analyze the privacy leakage from the shared primal variable $\boldsymbol{\tilde{w}}_i^{k}$ in each iteration, and finally compute the end-to-end differential privacy guarantee across $t$ iterations using the moments accountant method. Here we use $\boldsymbol{w}_{i,\mathcal{D}_i}^{k}$ and $	\boldsymbol{w}_{i,\mathcal{D}^{'}_i}^{k} $ to denote the local primal variables updated from two neighboring datasets $\mathcal{D}_i$ and $\mathcal{D}_i^{'}$.
	
	\subsection{$L_2$-norm Sensitivity}
	
	In our approach, we apply Gaussian mechanism to add noise whose magnitude is calibrated by the $l_2$-norm sensitivity. Note that compared with Algorithm~\ref{ag:3} and prior works \cite{ZhZh17,ZhKh18,ZhangKh18}, the derivation of the sensitivity in our proposed algorithm does not require the assumption of smoothness and strong convexity of the objective function due to the first-order approximation used in the approximate augmented Lagrangian function. 
	
	\begin{lemma} \label{lem:s}
		Assume that $\Vert \ell^{'}(\cdot) \Vert \leq c_1$. The $l_2$-norm sensitivity of local primal variable $\boldsymbol{w}_i^{k}$ update function is given by:
		\begin{equation}
		\begin{split}
		\max_{\mathcal{D}_i, \mathcal{D}_i^{'}} \Vert \boldsymbol{w}_{i,\mathcal{D}_i}^{k}  -\boldsymbol{w}_{i,\mathcal{D}_i^{'}}^{k}  \Vert =  \frac{2c_1}{m_i(\rho+1/\eta_i^{k})} .
		\end{split}
		\end{equation}
	\end{lemma}
	\begin{proof}
		%Firstly, we require to obtain the closed-form expression of $\boldsymbol{w}_{i,\mathcal{D}_i}^{k}$ and $	\boldsymbol{w}_{i,\mathcal{D}^{'}_i}^{k}$. 
		Since $\hat{\mathcal{L}}_{\rho, k, i}(\boldsymbol{w}_i,\boldsymbol{\tilde{w}}_i^{k-1}, \boldsymbol{w}^{k-1}, \boldsymbol{\gamma}_i^{k-1})$ in the first step of DP-ADMM \eqref{eqite1} is a quadratic function w.r.t. $\boldsymbol{w}_i$ and therefore convex, we could obtain that:
		%$\boldsymbol{w}_{i,\mathcal{D}_i}^{k}$ and $\boldsymbol{w}_{i,\mathcal{D}^{'}_i}^{k}$ as the two versions of $\boldsymbol{\tilde{w}}_i^{k}$ from two neighbouring datasets $\mathcal{D}_i$ and $\mathcal{D}_i^{'}$:
		\begin{subequations}
			\begin{align}
			\boldsymbol{w}_{i,\mathcal{D}_i}^{k} = & \bigg(-\sum_{j=1}^{m_i }\frac{1}{m_i} \ell^{'}(\boldsymbol{a}_{i,j}, \boldsymbol{b}_{i,j},\boldsymbol{\tilde{w}}_i^{k-1})  -  \frac{\lambda}{n} R^{'}(\boldsymbol{\tilde{w}}_i^{k-1})   \nonumber\\ +& \boldsymbol{\gamma}_i^{k-1}  + \rho\boldsymbol{w}^{k-1}+ \frac{ \boldsymbol{\tilde{w}}_i^{k-1}}{\eta_i^{k}}\bigg)\bigg(\rho+1/\eta_i^{k}\bigg)^{-1}, \label{eq:wkid}\\
			\boldsymbol{w}_{i,\mathcal{D}^{'}_i}^{k}  = & \bigg(-\sum_{j=1}^{m_i-1}\frac{1}{m_i} \ell^{'}(\boldsymbol{a}_{i,j}, \boldsymbol{b}_{i,j},\boldsymbol{\tilde{w}}_i^{k-1})   \nonumber\\ -&  \frac{1}{m_i}\ell^{'}(a^{'}_{i,m_i}, b^{'}_{i,m_i},\boldsymbol{\tilde{w}}_i^{k-1})-\frac{\lambda}{n} R^{'}(\boldsymbol{\tilde{w}}_i^{k-1}) \nonumber\\ +&     \boldsymbol{\gamma}_i^{k-1} + \rho\boldsymbol{w}^{k-1} + \frac{ \boldsymbol{\tilde{w}}_i^{k-1}}{\eta_i^{k}}\bigg)\bigg(\rho+1/\eta_i^{k}\bigg)^{-1}, \label{eq:wkid1}
			\end{align}
		\end{subequations}
		by computing the derivative of \eqref{eq:apro} with inputs $\boldsymbol{w}^{k-1}$ and $\boldsymbol{\gamma}_i^{k-1}$
		%$\hat{\mathcal{L}}_{\rho, k, i}(\boldsymbol{w}_i,\boldsymbol{\tilde{w}}_i^{k-1}, \boldsymbol{w}^{k-1}, \boldsymbol{\gamma}_i^{k-1})$: 
		%\begin{equation}
		%\begin{split}
		%& \nabla \hat{\mathcal{L}}_{\rho, k, i}(\boldsymbol{w}_i,\boldsymbol{\tilde{w}}_i^{k-1}, %\boldsymbol{w}^{k-1}, \boldsymbol{\gamma}_i^{k-1}) =\\ & \quad \quad \quad\sum_{j=1}^{m_i }\frac{1}{m_i} %\ell^{'}(\boldsymbol{a}_{i,j}, \boldsymbol{b}_{i,j},\boldsymbol{\tilde{w}}_i^{k-1})  + \frac{\lambda}{n} %R^{'}(\boldsymbol{\tilde{w}}_i^{k-1})    -\\& \quad \quad  \quad\quad\quad\boldsymbol{\gamma}_i^{k-1} + %\rho(\boldsymbol{w}_i-\boldsymbol{w}^k)+ %\frac{1}{\eta_i^{k}}(\boldsymbol{w}_i-\boldsymbol{\tilde{w}}_i^{k-1} ),
		%\end{split}
		%\end{equation}
		and letting $\nabla \hat{\mathcal{L}}_{\rho, k, i}(\boldsymbol{w}_i,\boldsymbol{\tilde{w}}_i^{k-1}, \boldsymbol{w}^{k-1}, \boldsymbol{\gamma}_i^{k-1})$ to be $0$. 

		With $	\boldsymbol{w}_{i,\mathcal{D}_i}^{k}$ and $\boldsymbol{w}_{i,\mathcal{D}^{'}_i}^{k}$ calculated by \eqref{eq:wkid} and \eqref{eq:wkid1} respectively, the $l_2$-norm sensitivity of primal variable $\boldsymbol{w}_i^{k}$ update function is defined by:
		\begin{equation}
		\begin{split}
		&\max_{\mathcal{D}_i, \mathcal{D}_i^{'}} \Vert \boldsymbol{w}_{i,\mathcal{D}_i}^{k}  -\boldsymbol{w}_{i,\mathcal{D}_i^{'}}^{k}  \Vert \\
		= & \max_{\mathcal{D}_i, \mathcal{D}_i^{'}} \frac{\big\Vert\ell^{'}(\boldsymbol{a}_{i,m_i}, \boldsymbol{b}_{i,m_i},\boldsymbol{\tilde{w}}_i^{k-1}) - \ell^{'}\small(\boldsymbol{a}^{'}_{i,m_i}, \boldsymbol{b}^{'}_{i,m_i},\boldsymbol{\tilde{w}}_i^{k-1}\small) \big\Vert}{m_i(\rho+1/\eta_i^{k})}  .
		\end{split}
		\end{equation}
		%With $\mathcal{D}_i$ and $\mathcal{D}_i^{'}$ differing in only one data entry $\boldsymbol{d}_{i,m}$, we define $\sum_{j=1}^{m_i }\frac{1}{m_i} \ell^{'}(\boldsymbol{a}_{i,j}, \boldsymbol{b}_{i,j},\boldsymbol{w}_i^k) $ and $\ell^{'}(\mathcal{D}_i^{'},\boldsymbol{w}_i^k)$ by:
		%\begin{subequations}
		%	\begin{align}
		%	\sum_{j=1}^{m_i }\frac{1}{m_i} \ell^{'}(\boldsymbol{a}_{i,j}, \boldsymbol{b}_{i,j},\boldsymbol{w}_i^k) =& \frac{1}{m} \sum_{j=1}^{m-1} \ell^{'}(\boldsymbol{d}_{i,j},\boldsymbol{w}_i^k)+\frac{1}{m}\ell^{'}(\boldsymbol{d}_{i,m},\boldsymbol{w}_i^k),  \\
		%	\ell^{'}(\mathcal{D}^{'}_i,\boldsymbol{w}_i^k) =& \frac{1}{m} \sum_{j=1}^{m-1} \ell^{'}(\boldsymbol{d}_{i,j},\boldsymbol{w}_i^k)+\frac{1}{m}\ell^{'}(\boldsymbol{d}^{'}_{i,m},\boldsymbol{w}_i^k). 
		%	\end{align}
		%\end{subequations}
		Since $\Vert \ell^{'}(\cdot) \Vert$ is bounded by $c_1$, the sensitivity of $\boldsymbol{w}_i^{k}$ update function is given by $2c_1/\big(m_i(\rho+1/\eta_i^{k})\big)$.
		% 		\begin{equation}
		% 		\begin{split}
		% 		&\max_{\mathcal{D}_i, \mathcal{D}_i^{'}} \Vert \boldsymbol{w}_{i,\mathcal{D}_i}^{k}  -\boldsymbol{w}_{i,\mathcal{D}_i^{'}}^{k}  \Vert \\
		% 		=&\max_{\mathcal{D}_i, \mathcal{D}_i^{'}} \frac{\big\Vert \frac{1}{m_i}\ell^{'}(a_{i,m_i}, b_{i,m_i},\boldsymbol{\tilde{w}}_i^{k-1})- \frac{1}{m_i}\ell^{'}\small(a^{'}_{i,m_i}, b^{'}_{i,m_i},\boldsymbol{\tilde{w}}_i^{k-1}\small) \big\Vert}{\rho+1/\eta_i^{k}}  \\ = & \frac{2c_1}{m_i(\rho+1/\eta_i^{k})} .
		% 		\end{split}
		% 		\end{equation}
	\end{proof}
	Lemma~\ref{lem:s} shows that the sensitivity of $\boldsymbol{w}_i^{k}$ update function in our approach is affected by the time-varying $\eta_i^{k}$. When we set $\eta_i^{k}$ to decrease with increasing $k$, the sensitivity becomes smaller with larger $k$, then the noise added would be smaller when $\epsilon$ is fixed. Thus, the updates would be stable in spite of the existence of the noise.
	
	\subsection{$(\epsilon, \delta)$-Differential Privacy Guarantee}
	
	In this section, we prove that each iteration of Algorithm \ref{ag:1} guarantees $(\epsilon, \delta)$-differential privacy. 
	
	\begin{theorem} \label{theo:1}
		Assume that $\Vert \ell^{'}(\cdot) \Vert \leq c_1 $. Let $\epsilon \in (0,1]$ be arbitrary and $\boldsymbol{\xi}^{k}_i$ be the noise sampled from Gaussian mechanism with variance $\sigma_{i,k}^{2}$ where
		\begin{equation}
		\sigma_{i,k} = \frac{2 c_1\sqrt{2\ln(1.25/\delta)} }{m_i \epsilon (\rho+1/\eta_i^{k})}. 
		\end{equation}
		Each iteration of DP-ADMM guarantees $(\epsilon,\delta)$-differential privacy. Specifically, for any neighboring datasets $\mathcal{D}_i$ and $\mathcal{D}^{'}_i$, for any output $\boldsymbol{\tilde{w}}_i^{k}$, the following inequality always holds:
		\begin{equation}
		\Pr[\boldsymbol{\tilde{w}}_i^{k}\vert \mathcal{D}_i ] \leq e^{\epsilon} \cdot \Pr[\boldsymbol{\tilde{w}}_i^{k}\vert \mathcal{D}^{'}_i] +\delta .
		\end{equation}
	\end{theorem}
	
	\begin{proof}
		The privacy loss from $\boldsymbol{\tilde{w}}_i^{k}$ is calculated as
		\begin{equation}
		\bigg\lvert \ln \frac{\Pr[\boldsymbol{\tilde{w}}_i^{k}\rvert \mathcal{D}_i ]}{\Pr[\boldsymbol{\tilde{w}}_i^{k}\vert \mathcal{D}^{'}_i]} \bigg\vert
		%=  \bigg\vert \ln \frac{\Pr(\boldsymbol{w}_{i,\mathcal{D}_i}^{k}  + \boldsymbol{\xi}_{i}^{k})}{\Pr(\boldsymbol{w}_{i,\mathcal{D}_i^{'}}^{k}  +\boldsymbol{\xi}_{i}^{k,'})} \bigg\vert
		=	\bigg\lvert \ln \frac{\Pr[\tilde{w}_i^{k^{(h,l)}}\rvert \mathcal{D}_i ]}{\Pr[\tilde{w}_i^{k^{(h,l)}}\vert \mathcal{D}^{'}_i]} \bigg\vert
		= \bigg\vert \ln \frac{\Pr[\xi_{i}^{k^{(h,l)}}]}{\Pr[\xi_{i}^{{k,'}^{(h,l)}}]} \bigg\vert,
		\end{equation}
		where $\xi_{i}^{{k}^{(h,l)}}$ and $\xi_{i}^{{k,'}^{(h,l)}}$ are the $(h,l)$-entry of $\boldsymbol{\xi}_{i}^{{k}}$ and $\boldsymbol{\xi}_{i}^{{k,'}}$, and are sampled from $\mathcal{N}(0,\sigma_{i,k}^{2})$. This leads to:
		\begin{equation}
		\begin{split}
		&	\bigg\lvert \ln \frac{\Pr[\boldsymbol{\tilde{w}}_i^{k}\rvert \mathcal{D}_i ]}{\Pr[\boldsymbol{\tilde{w}}_i^{k}\vert \mathcal{D}^{'}_i]} \bigg\vert
		=  \big\vert{\frac{1}{2\sigma_{i,k}^{2}} \big(\big\Vert \xi_{i}^{k^{(h,l)}} \big\Vert}^2- {\big\Vert \xi_{i}^{{k,'}^{(h,l)}} \big\Vert}^2\big)\big\vert \\
		= & \big\vert\frac{1}{2\sigma_{i,k}^{2}} \big({\Vert \xi_{i}^{{k}^{(h,l)}} \Vert}^2- {\Vert \xi_{i}^{{k}^{(h,l)}}+(w_{i,\mathcal{D}_i}^{k^{(h,l)}}  -w_{i,\mathcal{D}_i^{'}}^{k^{(h,l)}}) \Vert}^2 \big)\big\vert \\
		=  & \big\vert \frac{1}{2\sigma_{i,k}^{2}}\big(2 \xi_{i}^{k^{(h,l)}} \Vert w_{i,\mathcal{D}_i}^{k^{(h,l)}}  -w_{i,\mathcal{D}_i^{'}}^{k^{(h,l)}} \Vert + {\Vert w_{i,\mathcal{D}_i}^{k^{(h,l)}}  -w_{i,\mathcal{D}_i^{'}}^{k^{(h,l)}} \Vert}^2 \big) \big\vert.  \label{eq:211}
		\end{split}
		\end{equation}
		Since $\Vert \ell^{'}(\cdot) \Vert \leq c_1 $, according to Lemma \ref{lem:s}, we have
		$\Vert w_{i,\mathcal{D}_i}^{k^{(h,l)}}  -w_{i,\mathcal{D}_i^{'}}^{k^{(h,l)}} \Vert<\Vert \boldsymbol{w}_{i,\mathcal{D}_i}^{k}  -\boldsymbol{w}_{i,\mathcal{D}_i^{'}}^{k} \Vert \leq 2c_1/\big(m_i(\rho+1/\eta_i^{k})\big)$. 
		Thus, by letting $\sigma_{i,k} = 2 c_1\sqrt{2\ln(1.25/\delta)} /\big(m_i \epsilon (\rho+1/\eta_i^{k})\big)$, we have
		\begin{equation}
		\bigg\vert \ln \frac{\Pr[\boldsymbol{\tilde{w}}_i^{k}\vert \mathcal{D}_i ]}{\Pr[\boldsymbol{\tilde{w}}_i^{k}\vert \mathcal{D}^{'}_i]} \bigg\vert 	\leq  \bigg\vert \frac{ \xi_{i}^{k^{(h,l)}}   m_i (\rho+1/\eta_i^{k}) + c_1 }{4\ln(1.25/\delta)  c_1/{\epsilon}^2} \bigg\vert.\\
		%	=  & \bigg\vert \frac{2 \boldsymbol{\xi}_{i}^{k} \Vert \boldsymbol{w}_{i,\mathcal{D}_i}^{k}  -\boldsymbol{w}_{i,\mathcal{D}_i^{'}}^{k} \Vert + {\Vert \boldsymbol{w}_{i,\mathcal{D}_i}^{k}  -\boldsymbol{w}_{i,\mathcal{D}_i^{'}}^{k} \Vert}^2 }{2\sigma_{i,k}^{2}} \bigg\vert\\
		%	\leq & \bigg\vert \frac{ \boldsymbol{\xi}_{i}^{k}   m_i (\rho+1/\eta_i^{k}) + c_1 }{4\ln(1.25/\delta)  c_1/{\epsilon}^2} \bigg\vert.
		\end{equation}
		When $\vert \xi_{i}^{k^{(h,l)}}  \vert \leq \big(4\ln(1.25/\delta)c_1/\epsilon-c_1\big) /\big(\epsilon m_i (\rho+1/\eta_i^{k})  \big) $, $\big\vert \ln \big(\Pr[\boldsymbol{\tilde{w}}_i^{k}\vert \mathcal{D}_i ]/\Pr[\boldsymbol{\tilde{w}}_i^{k}\vert \mathcal{D}^{'}_i]\big) \big\vert$ is bounded by $\epsilon$. Next, we need to prove that $\Pr\big[\vert\xi_{i}^{k^{(h,l)}}  \vert > \big(4\ln(1.25/\delta)c_1/\epsilon-c_1\big) /\big(\epsilon m_i (\rho+1/\eta_i^{k})  \big)\big] \leq \delta $, which requires $\Pr\big[\xi_{i}^{k^{(h,l)}} > \big(4\ln(1.25/\delta)c_1/\epsilon-c_1\big) /\big(\epsilon m_i (\rho+1/\eta_i^{k})  \big)\big] \leq \delta/2 $. According to the tail bound of normal distribution $\mathcal{N}(0,\sigma_{i,k}^{2})$, we have
		\begin{equation}
		\Pr\big[\xi_{i}^{k^{(h,l)}} >r\big] \leq \frac{\sigma_{i,k}}{r\sqrt{2 \pi}}e^{-r^2/2\sigma_{i,k}^{2}}.
		\end{equation}
		By letting $r =\big(4\ln(1.25/\delta)c_1/\epsilon-c_1\big) /\big(\epsilon m_i (\rho+1/\eta_i^{k})  \big) $ in the above inequality, we have:
		\begin{equation}
		\begin{split}
		& \Pr\bigg[\xi_{i}^{k^{(h,l)}} > \frac{4\ln(1.25/\delta) c_1/\epsilon-c_1}{ m_i (\rho+1/\eta_i^{k})} \bigg] \\ \leq &  \frac{2  \sqrt{2\ln(1.25/\delta)}  }{(4\ln(1.25/\delta) - \epsilon ) \sqrt{2 \pi}}\exp\bigg(-{\frac{{( 4\ln(1.25/\delta) - \epsilon)}^2 }{ 8\ln(1.25/\delta)}\bigg)}. \label{eq:213} 
		\end{split}
		\end{equation}
		When $\delta$ is small ($\leq 0.01$) and let $\epsilon \leq 1$, we have
		\begin{equation}
		%		\begin{split}
		\frac{2\sqrt{2\ln(1.25/\delta)} }{(4\ln(1.25/\delta) - \epsilon ) \sqrt{2 \pi}} 
		%		\leq \frac{\sqrt{2\ln(1.25/\delta)} 2 }{(4\ln(1.25/\delta) - 1 ) \sqrt{2 \pi}} 
		< \frac{1}{\sqrt{2 \pi}},\label{eq:214}
		%		\end{split}
		\end{equation}
		and
		\begin{equation}
		%		\begin{split}
		-{\frac{{\big( 4\ln(1.25/\delta) - \epsilon\big)}^2 }{ 8\ln(1.25/\delta)}} 
		%\leq -{\frac{{( 4\ln(1.25/\delta) - 1)}^2 }{ 8\ln(1.25/\delta)}} \\
		%<  - 2 \ln(1.25/\delta) + \frac{8}{9}  
		< \ln(\sqrt{2\pi}\frac{\delta}{2}). 
		%		\end{split}
		\end{equation}
		As a result, we have:
		\begin{equation}
		%	\begin{split}
		\Pr\bigg[\xi_{i}^{k^{(h,l)}} >  \frac{4\ln(1.25/\delta) c_1/\epsilon-c_1}{ m_i (\rho+1/\eta_i^{k})} \bigg] \\ 
		%\leq &  \frac{\sqrt{2\ln(1.25/\delta)} 2 }{(4\ln(1.25/\delta) - \epsilon ) \sqrt{2 \pi}}\exp(-{\frac{{( 4\ln(1.25/\delta) - \epsilon)}^2 }{ 8\ln(1.25/\delta)})} \\ 
		%< & \frac{1}{\sqrt{2 \pi}}\exp\big(\ln(\sqrt{2\pi}\frac{\delta}{2})\big) 
		<  \frac{\delta}{2}. 
		%\end{split}
		\end{equation}
		So far we have proved that  $\Pr\big[\xi_{i}^{k^{(h,l)}} > \big(4\ln(1.25/\delta)c_1/\epsilon-c_1\big) /\big(\epsilon m_i (\rho+1/\eta_i^{k})  \big)\big] \leq \delta/2 $, thus we can prove that  $\Pr\big[\vert \xi_{i}^{k^{(h,l)}}  \vert > \big(4\ln(1.25/\delta)c_1/\epsilon-c_1\big) /\big(\epsilon m_i (\rho+1/\eta_i^{k})  \big)\big] \leq \delta $. We define:
		\begin{subequations}
			\begin{align}
			\mathbb{A}_1 = & \{ \xi_{i}^{k^{(h,l)}}: \vert \xi_{i}^{k^{(h,l)}} \vert \leq  \frac{4\ln(1.25/\delta) c_1/\epsilon-c_1}{ m_i (\rho+1/\eta_i^{k})}\},  \\
			\mathbb{A}_2 = & \{ \xi_{i}^{k^{(h,l)}}: \vert \xi_{i}^{k^{(h,l)}}  \vert >  \frac{4\ln(1.25/\delta) c_1/\epsilon-c_1}{ m_i (\rho+1/\eta_i^{k})}\}. 
			\end{align}
		\end{subequations}
		Therefore, we obtain the result:
		\begin{equation}
		\begin{split}
		\Pr[\boldsymbol{\tilde{w}}_i^{k}\vert \mathcal{D}_i] = & \Pr[w_{i,\mathcal{D}_i}^{k^{(h,l)}}  + \xi_{i}^{k^{(h,l)}}: \xi_i^{k^{(h,l)}} \in \mathbb{A}_1 ]\\ & +\Pr[w_{i,\mathcal{D}_i}^{k^{(h,l)}}  + \xi_{i}^{k^{(h,l)}}: \xi_i^{k^{(h,l)}} \in \mathbb{A}_2 ] \\
		%	< &  e^{\epsilon} \cdot \Pr(\boldsymbol{w}_{i,\mathcal{D}_i^{'}}^{k}  +\boldsymbol{\xi}_{i}^{k,'}) + \delta \\
		< & e^{\epsilon} \cdot \Pr[\boldsymbol{\tilde{w}}_i^{k}\vert \mathcal{D}^{'}_i] +\delta,
		\end{split}
		\end{equation}
		which proves that each iteration of DP-ADMM guarantees $(\epsilon,\delta)$-differential privacy.
	\end{proof}
	%According to Theorem~\ref{theo:1}, the privacy guarantee of our proposed algorithm does not rely on the assumption that the loss function and regularizer are smooth, and the regularizer is strongly convex, which is needed in previous works \cite{ZhZh17, ZhKh18, ZhangKh18}. Thus, our approach to have wider applicability.
	
	\subsection{Total Privacy Leakage}
	We have proved that each iteration of the proposed algorithm is $(\epsilon, \delta)$-differentially private. Here we focus on the total privacy leakage of our algorithm. Since Algorithm~\ref{ag:1} is a $t$-fold adaptive algorithm, we follow prior studies \cite{AbCh16,Miro17} and use the moments accountant method to analyze the total privacy leakage.
	\begin{theorem}[Advanced Composition Theorem] \label{theo:4}
		Assume $\Vert \ell^{'}(\cdot) \Vert \leq c_1 $. Let $\epsilon \in (0,1]$ be arbitrary and $\boldsymbol{\xi}^{k}_i$ be sampled from Gaussian mechanism with variance $\sigma_{i,k}^{2}$ where
		\begin{equation}
		\sigma_{i,k} = \frac{2 c_1 \sqrt{2\ln(1.25/\delta)} }{m_i \epsilon (\rho+1/\eta_i^{k})}. 
		\end{equation}
		Then Algorithm~\ref{ag:1} guarantees $(\bar{\epsilon}, \delta)$-differential privacy, where $\bar{\epsilon} = c_0 \sqrt{t}\epsilon$ for some constant $c_0$. 
	\end{theorem}
	\begin{proof}
		See Appendix \ref{ap:a}.
	\end{proof}

	%	In our simulation, we would give the total privacy leakage based on the moments accountant method to evaluate our approach. 
	
	\section{Convergence Analysis}   \label{sec:conv}
	
	In this section, we analyze the convergence of the proposed DP-ADMM. Let $\boldsymbol{w}^*$ denote the optimal solution of problem \eqref{eq:dispro}, and $c_w $ denote $\Vert \boldsymbol{w}^{*} \Vert$. Firstly, we analyze the convergence property based on the general assumption that the objective function is convex and non-smooth. Secondly, we refine the convergence property under a stricter assumption that the objective function is convex and smooth. 
	
	We define the following notations to be used for the analysis:
	\begin{subequations}
		\begin{align}
		& f_i(\boldsymbol{w}_i)  = \sum_{j=1}^{m_i }\frac{1}{m_i} \ell(\boldsymbol{a}_{i,j}, \boldsymbol{b}_{i,j},\boldsymbol{w}_i)  +\frac{\lambda}{n} R(\boldsymbol{w}_i),  \nonumber\\
		&	\boldsymbol{\bar{w}}^t  = \frac{1}{t} \sum_{k=1}^t \boldsymbol{w}^{k}, \quad \boldsymbol{\bar{\gamma}}^t_i  = \frac{1}{t} \sum_{k=1}^t \boldsymbol{\gamma}^k_i, \quad \boldsymbol{\bar{w}}_i^t  =  \frac{1}{t} \sum_{k=0}^{t-1} \boldsymbol{\tilde{w}}_i^{k},  \nonumber\\  &\boldsymbol{u}_i^{k}   = \begin{bmatrix}
		\boldsymbol{\tilde{w}}_i^{k}\\
		\boldsymbol{w}^k \\
		\boldsymbol{\gamma}^k_i 
		\end{bmatrix}, \quad \quad 
		\boldsymbol{u}_i   = \begin{bmatrix}
		\boldsymbol{w}_i\\
		\boldsymbol{w}\\
		\boldsymbol{\gamma}_i 
		\end{bmatrix}, \quad  F(\boldsymbol{u}_i^{k})  =  \begin{bmatrix}
		-\boldsymbol{\gamma}^{k}_i \\
		\boldsymbol{\gamma}^{k}_i \\
		\boldsymbol{\tilde{w}}_i^{k}-\boldsymbol{w}^k
		\end{bmatrix}. \nonumber
		\end{align}
	\end{subequations}
	We show that DP-ADMM achieves an $O(1/\sqrt{t})$ rate of convergence in terms of both the objective value and the constraint violation: $  \sum_{i=1}^{n}\big( f_i(\boldsymbol{\bar{w}}_i^{t})-f_i(\boldsymbol{w}^*)+ \beta \Vert \boldsymbol{\bar{w}}_i^{t} - \boldsymbol{\bar{w}}^{t} \Vert\big) $, where $ \sum_{i=1}^{n}\big( f_i(\boldsymbol{\bar{w}}_i^{t})-f_i(\boldsymbol{w}^*)\big)$ represents the distance between the current objective value and the optimal value while $ \sum_{i=1}^{n} \beta \Vert \boldsymbol{\bar{w}}_i^{t} - \boldsymbol{\bar{w}}^{t} \Vert$ measures the difference between the local model and the global one. Therefore, when we have $  \sum_{i=1}^{n}\big( f_i(\boldsymbol{\bar{w}}_i^{t})-f_i(\boldsymbol{w}^*)+ \beta \Vert \boldsymbol{\bar{w}}_i^{t} - \boldsymbol{\bar{w}}^{t} \Vert\big)  = 0$, our training result converges to the optimal one and all local models reach consensus. 
	
	\subsection{Non-Smooth Convex Objective Function} \label{sec:cof}
	
	In this section, we analyze the convergence when the objective function is convex but non-smooth. 
	%Based on this assumption and Lemma \ref{lem:1}, we give the convergence analysis.
	%
	We firstly analyze a single iteration of our algorithm in Lemma \ref{lem:2} and then give the convergence result of DP-ADMM in Theorem \ref{the:2}.
	\begin{lemma} \label{lem:2}
		Assume $\ell(\cdot)$ and $R(\cdot)$ are convex. For any $k \geq 1$, we have:
		\begin{equation}
		\begin{split}
		&\sum_{i=1}^{n}\bigg( f_i(\boldsymbol{\tilde{w}}_i^{k-1})-f_i(\boldsymbol{w}_i) + {(\boldsymbol{u}_i^{k}-\boldsymbol{u}_i)}^{\intercal} F(\boldsymbol{u}_i^{k}) \bigg)\\
		\leq & \sum_{i=1}^{n}\bigg(\frac{\eta_i^{k}}{2}{\big\Vert f_i^{'}(\boldsymbol{\tilde{w}}_i^{k-1})-(\rho+1/\eta_i^{k}) \boldsymbol{\xi}_i^{k}  \big\Vert}^2   - \frac{\rho}{2}{\Vert \boldsymbol{w}_i - \boldsymbol{w}^{k} \Vert}^2  \\& \quad \quad +\frac{\rho}{2}{\Vert \boldsymbol{w}_i - \boldsymbol{w}^{k-1} \Vert}^2-\big (\rho+1/\eta_i^{k}\big)\big\langle  \boldsymbol{\xi}_i^{k}, \boldsymbol{w}_i - \boldsymbol{\tilde{w}}_i^{k-1} \big\rangle   \\ & \quad \quad + \frac{1}{2\eta_i^{k}}{\Vert \boldsymbol{w}_i - \boldsymbol{\tilde{w}}_i^{k-1} \Vert}^2  -\frac{1}{2\eta_i^{k}} {\Vert \boldsymbol{w}_i - \boldsymbol{\tilde{w}}_i^{k} \Vert}^2 \\ & \quad \quad  +  \frac{1}{2\rho}{\Vert \boldsymbol{\gamma}_i-\boldsymbol{\gamma}_i^{k-1} \Vert}^2 - \frac{1}{2\rho}{\Vert \boldsymbol{\gamma}_i-\boldsymbol{\gamma}_{i}^{k} \Vert}^2\bigg). 
		\end{split}
		\end{equation}
	\end{lemma}
	
	\begin{proof}
		See Appendix \ref{ap:b}.
	\end{proof}
	
	Based on Lemma \ref{lem:2}, we give the following convergence theorem. 
	\begin{theorem} \label{the:2}
		Assume $\ell(\cdot)$ and  $R(\cdot)$ are convex, $\Vert \ell^{'}(\cdot) \Vert \leq c_1$, and $\Vert R^{'}(\cdot) \Vert \leq c_2$. Let
		\begin{equation}
		\eta_i^{k} = \frac{c_w}{\sqrt{2k}}{\bigg((c_1+\lambda c_2/n)^2+\frac{8dpc_1^2\ln{(1.25/\delta)}}{m_i^2 \epsilon^2} \bigg)}^{-\frac{1}{2}}.
		\end{equation}
		Define
		\begin{equation}
		M_1(\epsilon,\delta) = \sum_{i=1}^n c_w\sqrt{ 2(c_1+\lambda c_2/n)^2+\frac{16dpc_1^2\ln{(1.25/\delta)}}{m_i^2 \epsilon^2}  } , 
		\end{equation}
		and 
		\begin{equation}
		M_2 = \frac{n(\rho c_w^2 + \beta^2/\rho)}{2}.
		\end{equation}
		For any $t \geq 1$ and $\beta$, we have:
		\begin{equation}
		\begin{split}
		& \mathbb{E}\bigg[ \sum_{i=1}^{n}\bigg( f_i(\boldsymbol{\bar{w}}_i^{t})-f_i(\boldsymbol{w}^*)+ \beta \Vert \boldsymbol{\bar{w}}_i^{t} - \boldsymbol{\bar{w}}^{t} \Vert \bigg)\bigg] \\
		\leq  & \quad    \frac{  M_1(\epsilon,\delta)}{ \sqrt{t}} + \frac{M_2}{t}.
		\end{split}
		\end{equation}
	\end{theorem}
	\begin{proof}
		See Appendix \ref{ap:c}.
	\end{proof}
	Theorem \ref{the:2} shows an explicit utility-privacy trade-off of our approach: when privacy guarantee is weaker (larger $\epsilon$ and $\delta$), our approach has better utility. In addition, it demonstrates that our algorithm converges at a rate of $O(1/\sqrt{t})$.

	\subsection{Smooth Convex Objective Function}
	
	In this section, we refine Theorem \ref{the:2} under a stricter assumption that $\ell(\cdot)$ and $R(\cdot)$ are both convex and smooth. %Under this stricter assumption, Algorithm \ref{ag:1} also achieves an $O(1/\sqrt{t})$ rate of convergence.
	Here, we replace the definition of $\boldsymbol{\bar{w}}_i^t$: $\boldsymbol{\bar{w}}_i^t =  \frac{1}{t} \sum_{k=0}^{t-1} \boldsymbol{\tilde{w}}_i^{k}$ by $\boldsymbol{\bar{w}}_i^t = \frac{1}{t} \sum_{k=1}^{t} \boldsymbol{\tilde{w}}_i^{k}$. %so that we could obtain a closed-form convergence result. 
	Similar to Section \ref{sec:cof}, we first focus on a single iteration and then give the final convergence result. 
	
	\begin{lemma} \label{lem:3}
		Assume $\ell(\cdot)$ and  $R(\cdot)$ are convex and smooth, $\Vert \nabla^2 \ell(\cdot) \Vert \leq c_3$, and $\Vert \nabla^2 R(\cdot) \Vert \leq c_4$. For any $k \geq 1$, we have:
		\begin{equation}
		\begin{split}
		&\sum_{i=1}^{n} \bigg( f_i(\boldsymbol{\tilde{w}}_i^{k})-f_i(\boldsymbol{w}_i) + {(\boldsymbol{u}_i^{k}-\boldsymbol{u}_i)}^{\intercal} F(\boldsymbol{u}_i^{k}) \bigg)\\
		\leq & \sum_{i=1}^{n}\bigg(\frac{\big(\rho+1/\eta_i^{k}\big)^2}{2/\eta_i^{k}-2(c_3+\lambda c_4/n)}  {\big\Vert  \boldsymbol{\xi}_i^{k}  \big\Vert}^2  - \frac{1}{2\eta_i^{k}}{\Vert \boldsymbol{w}_i - \boldsymbol{\tilde{w}}_i^{k} \Vert}^2\\
		& \quad \quad +  \frac{1}{2\eta_i^{k}}{\Vert \boldsymbol{w}_i - \boldsymbol{\tilde{w}}_i^{k-1} \Vert}^2- \big(\rho+1/\eta_i^{k}\big)\big\langle  \boldsymbol{\xi}_i^{k}, \boldsymbol{w}_i - \boldsymbol{\tilde{w}}_i^{k-1} \big\rangle\\
		& \quad \quad +\frac{\rho}{2}{\Vert \boldsymbol{w}_i - \boldsymbol{w}^{k-1} \Vert}^2 -\frac{\rho}{2} {\Vert \boldsymbol{w}_i - \boldsymbol{w}^{k} \Vert}^2 \\& \quad \quad + \frac{1}{2\rho}{\Vert \boldsymbol{\gamma}_i-\boldsymbol{\gamma}_i^{k-1} \Vert}^2 -\frac{1}{2\rho} {\Vert \boldsymbol{\gamma}_i-\boldsymbol{\gamma}_{i}^{k} \Vert}^2\bigg). 
		\end{split}
		\end{equation}
	\end{lemma}
	\begin{proof}
		See Appendix \ref{ap:d}.
	\end{proof}
	
	Based on Lemma \ref{lem:3}, we give the following theorem.
	\begin{theorem} \label{the:3}
		Assume $\ell(\cdot)$ and  $R(\cdot)$ are convex and smooth, $\Vert \nabla^2 \ell(\cdot) \Vert \leq c_3$, and $\Vert \nabla^2 R(\cdot) \Vert \leq c_4$. Let 
		\begin{equation}
		\eta_i^{k} = {\bigg(c_3+\lambda c_4/n+\frac{4c_1\sqrt{d p k\ln(1.25/\delta)}}{m_i\epsilon c_w}}\bigg)^{-1}.
		\end{equation}
		Define
		\begin{equation}
		M_3(\epsilon,\delta)= \sum_{i=1}^{n} \frac{ 4  c_w c_1 \sqrt{dp\ln(1.25/\delta)}  }{m_i \epsilon  },
		\end{equation}
		and
		\begin{equation}
		M_4 = \frac{n c_w^2(c_3+ \lambda c_4/n+\rho)+n \beta^2/\rho}{2}.
		\end{equation}
		For any $t \geq 1$ and $\beta$, we have:
		\begin{equation}
		\mathbb{E}\bigg[ \sum_{i=1}^{n}\bigg( f_i(\boldsymbol{\bar{w}}_i^{t})-f_i(\boldsymbol{w}^*)+ \beta \Vert \boldsymbol{\bar{w}}_i^{t} - \boldsymbol{\bar{w}}^{t} \Vert\bigg) \bigg]  
		\leq    \frac{M_3(\epsilon,\delta) }{\sqrt{t} }+\frac{M_4}{t}. 
		\end{equation}
	\end{theorem}
	\begin{proof}
		See Appendix \ref{ap:e}.
	\end{proof}
	
	Theorem \ref{the:3} also shows an explicit relation between the privacy budget (i.e., $\epsilon$ and $\delta$) and the utility of our approach with smoothness, and demonstrates that the result from our algorithm converges to the optimal result at a rate of $O(1/\sqrt{t})$.
	
	%	Based on different assumptions, we give two theorems about the convergence of Algorithm \ref{ag:1}. Specifically, Theorem \ref{the:2} holds for the optimization problem where loss function and regularizer are $l_1$-norm or $l_2$-norm. Theorem \ref{the:3} gives tighter convergence analysis but only holds for the problem where loss function and regularizer are both $l_2$-norm.

	\section{Performance Evaluation}   \label{sec:impl}

	\begin{figure*}[t]
		\centering
		\subfloat[$\epsilon = 0.05, \bar{\epsilon}=0.5009,\delta = 10^{-3}$]{\includegraphics[width=1.7in]{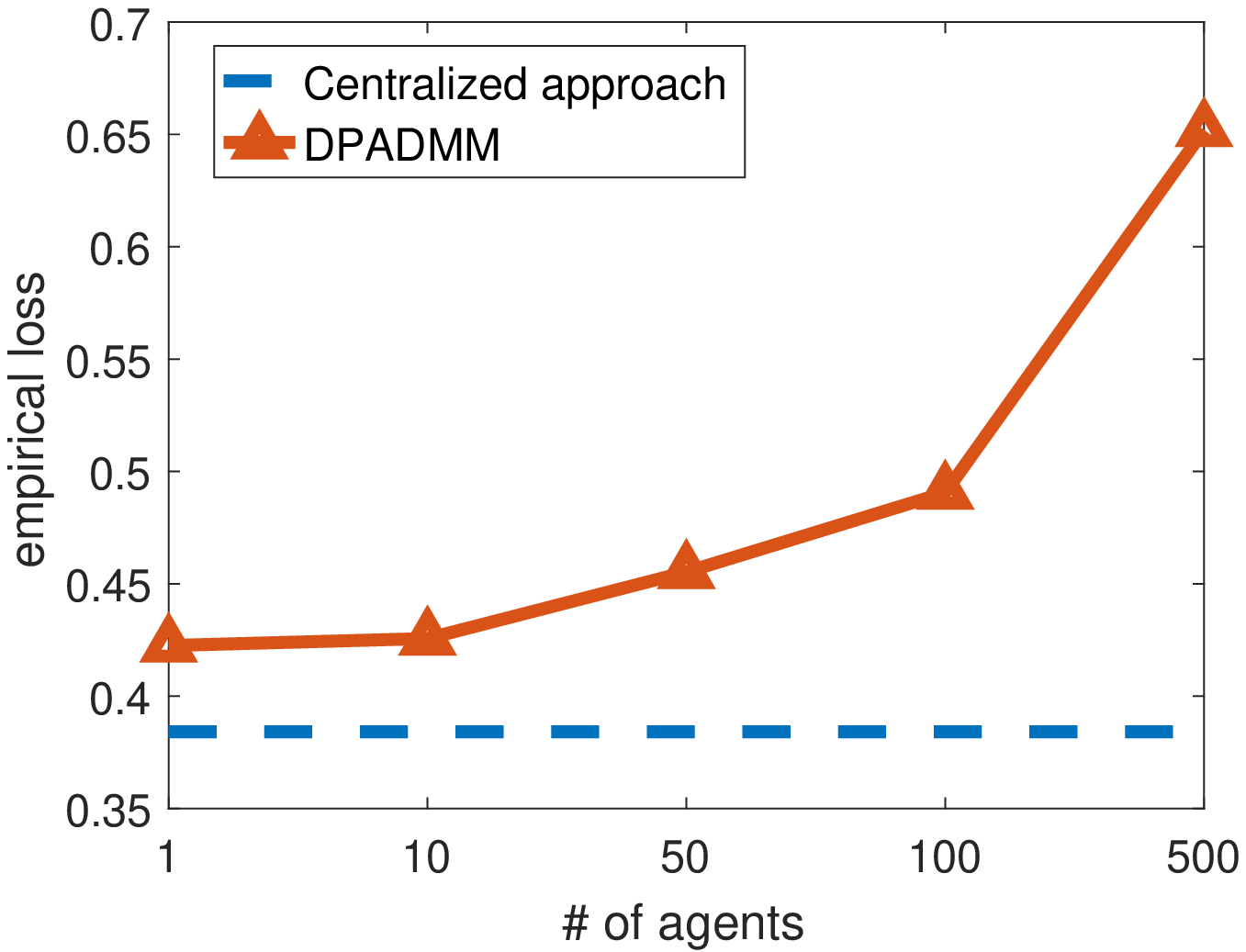}%
			\label{fig:11}	}
		\hfil
		\subfloat[$\epsilon = 0.1, \bar{\epsilon}=1.0193,\delta = 10^{-3}$]{\includegraphics[width=1.7in]{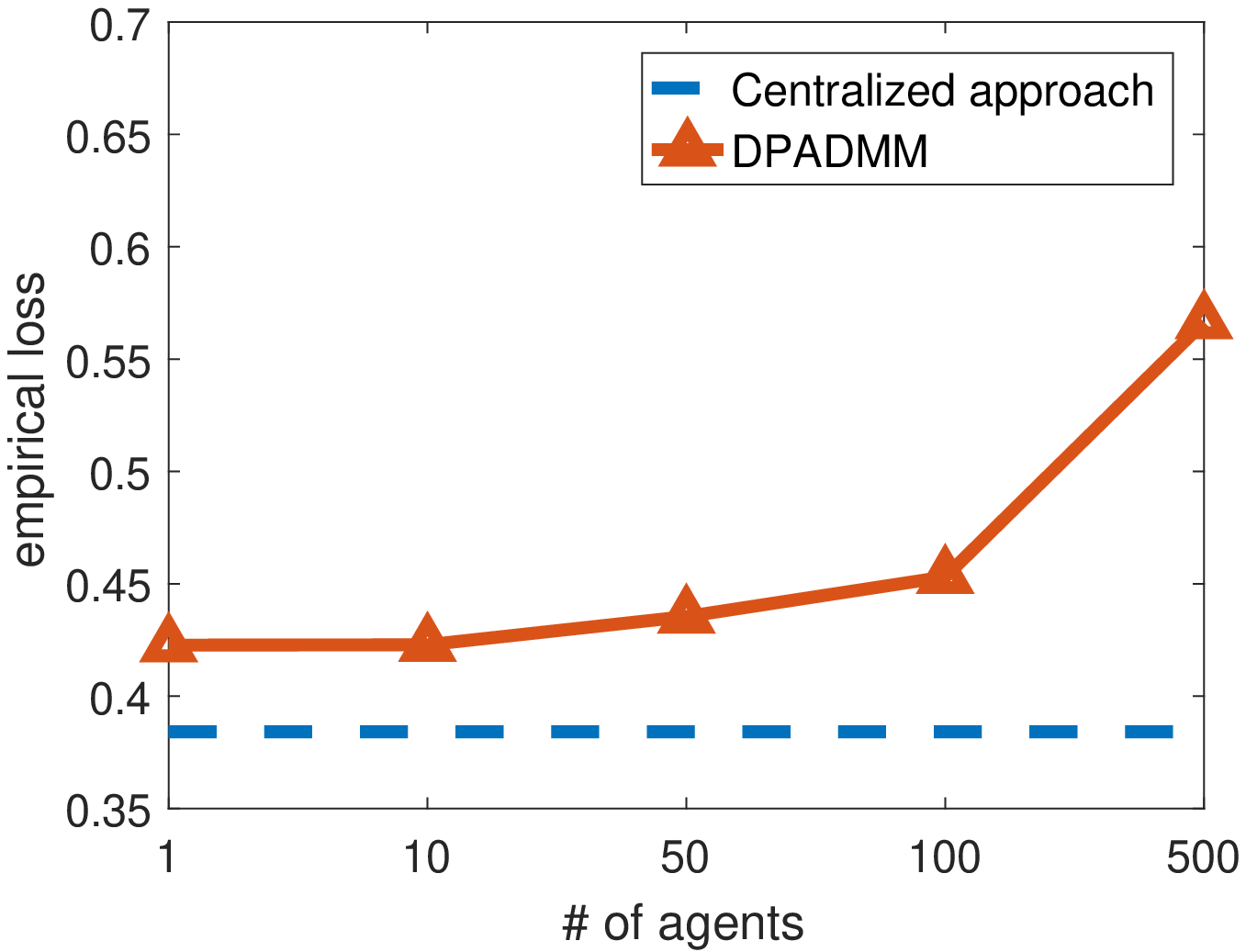}%
			\label{fig:12}	}
		\hfil
		\subfloat[$\epsilon = 0.05, \bar{\epsilon}=0.5009,\delta = 10^{-3}$]{\includegraphics[width=1.7in]{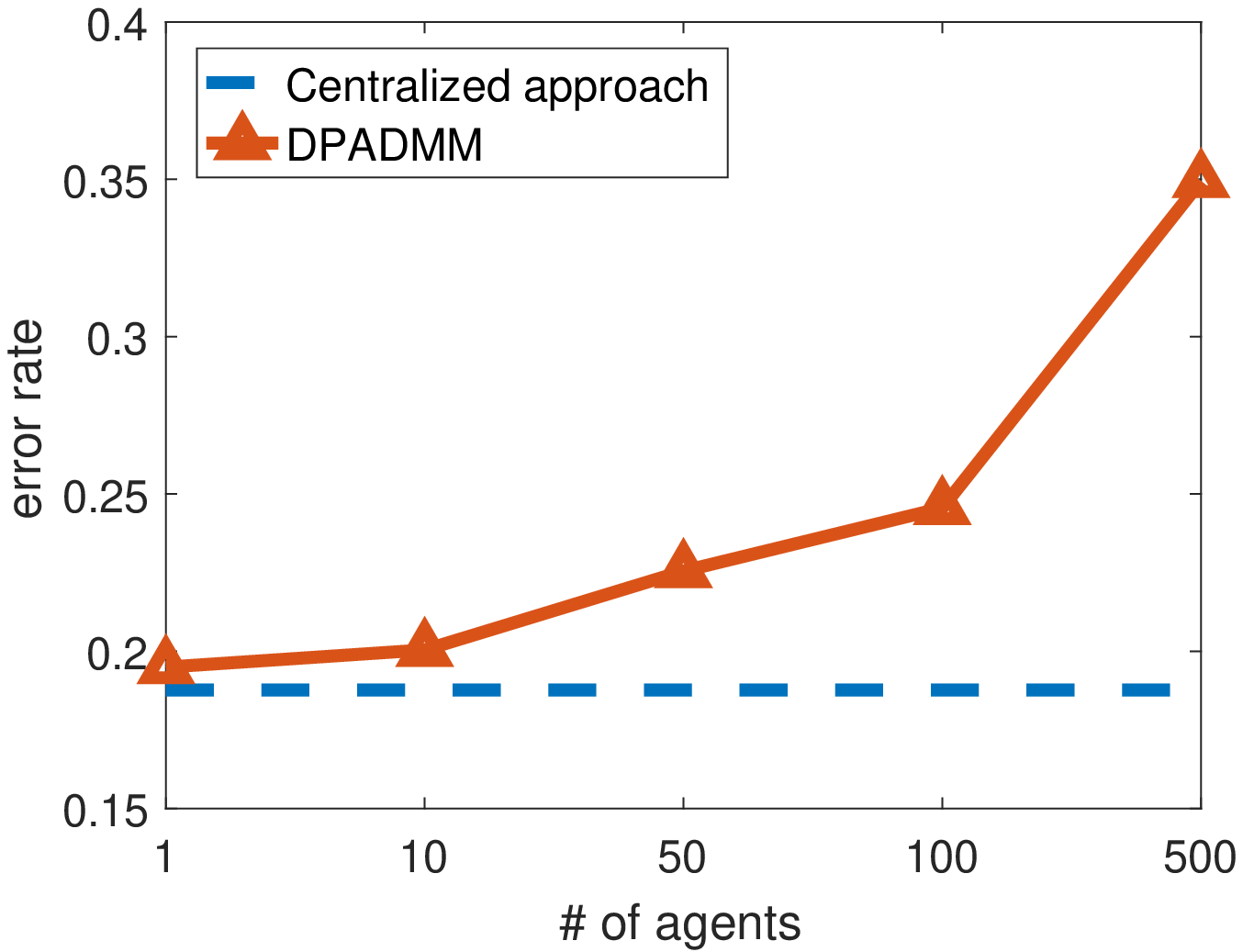}%
			\label{fig:13}	}
		\hfil
		\subfloat[$\epsilon = 0.1, \bar{\epsilon}=1.0193,\delta = 10^{-3}$]{\includegraphics[width=1.7in]{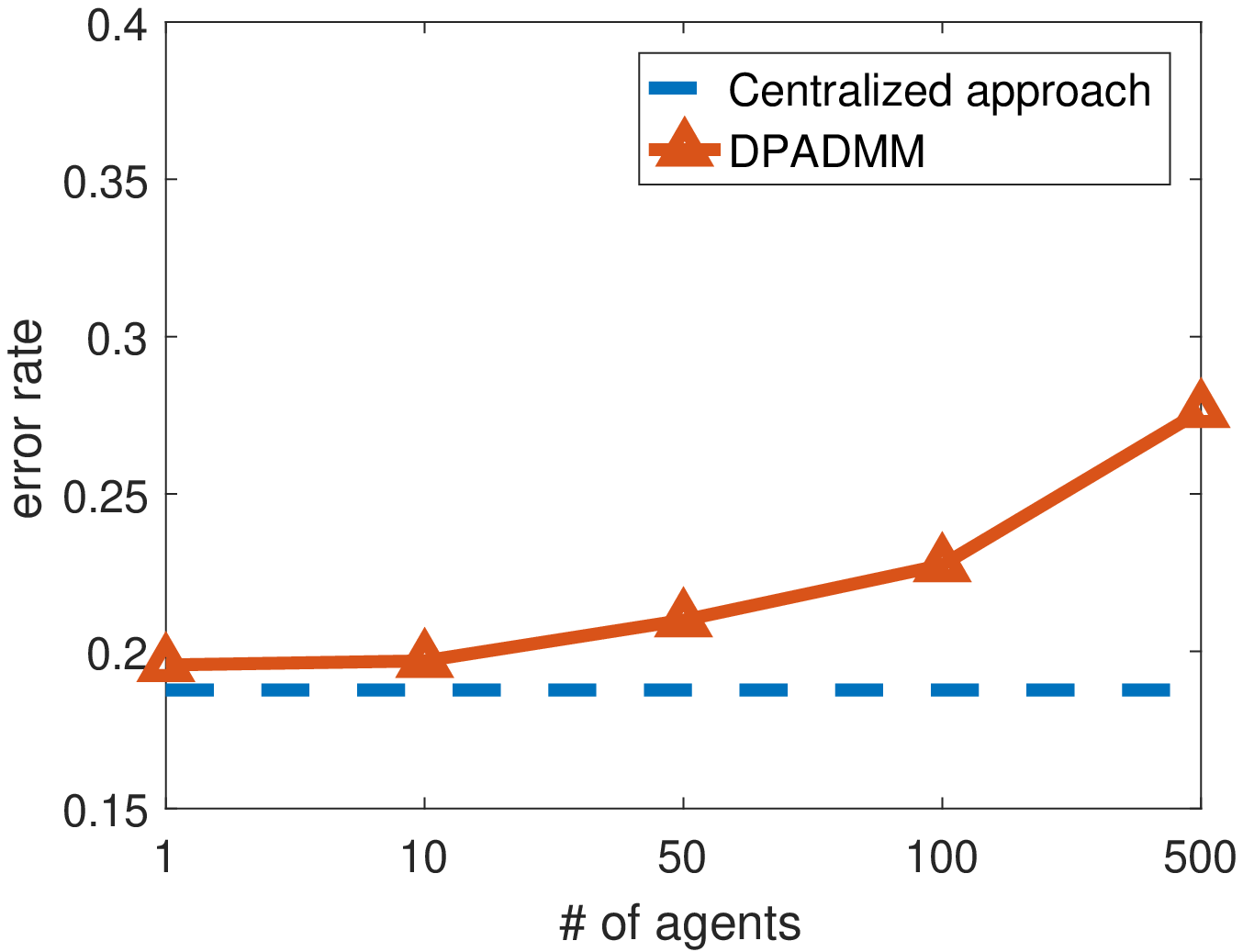}%
			\label{fig:14}	}
		\hfil
		\caption{Impact of distributed data source number on DP-ADMM ($l_1$-regularized logistic regression).}
		\label{fig:1}
		\vspace*{-.15in}
	\end{figure*}

	\begin{figure*}[t]
		\centering
		\subfloat[ $\delta = 10^{-3}$]{\includegraphics[width=1.7in]{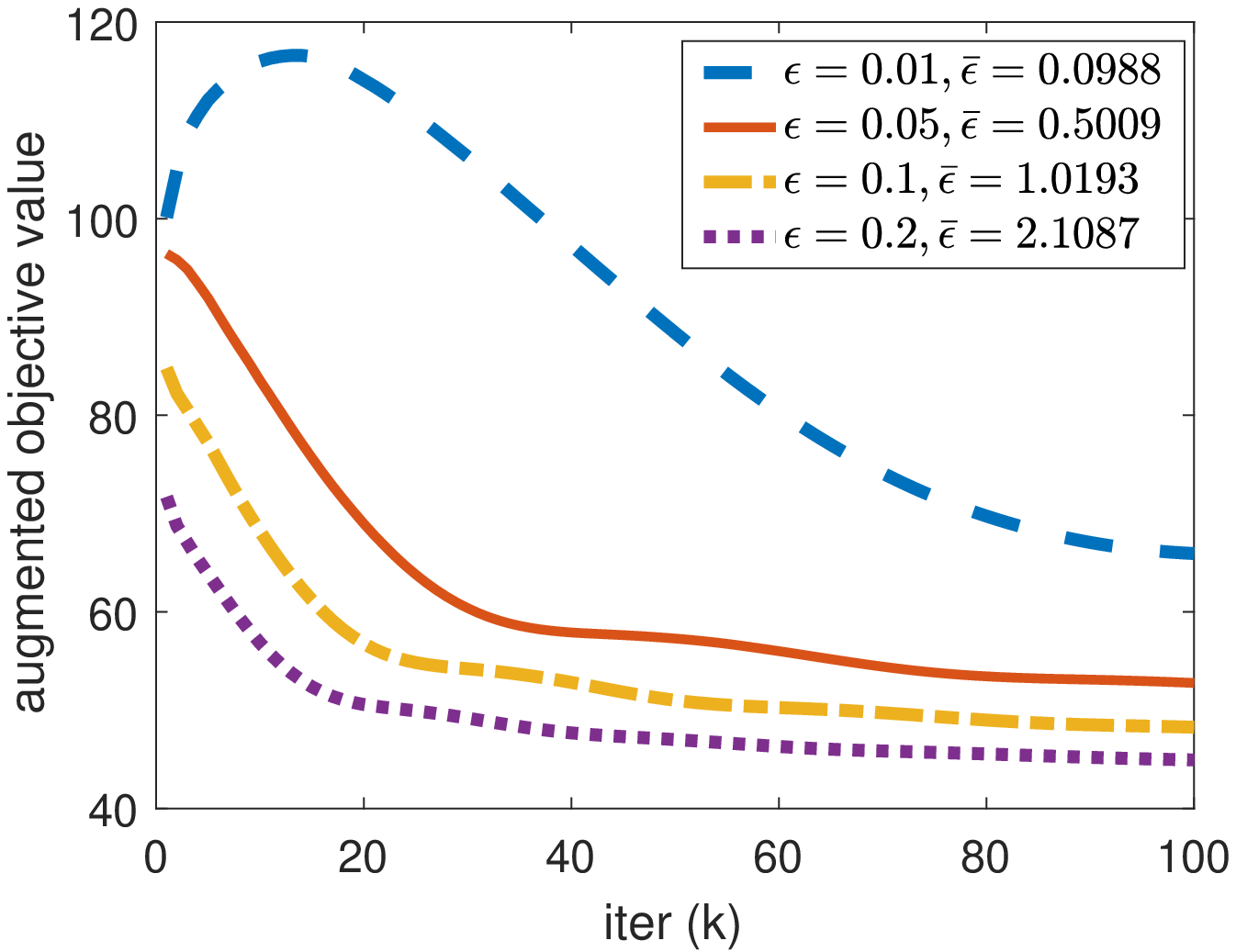}%
		}
		\hfil
		\subfloat[ $\delta = 10^{-4}$]{\includegraphics[width=1.7in]{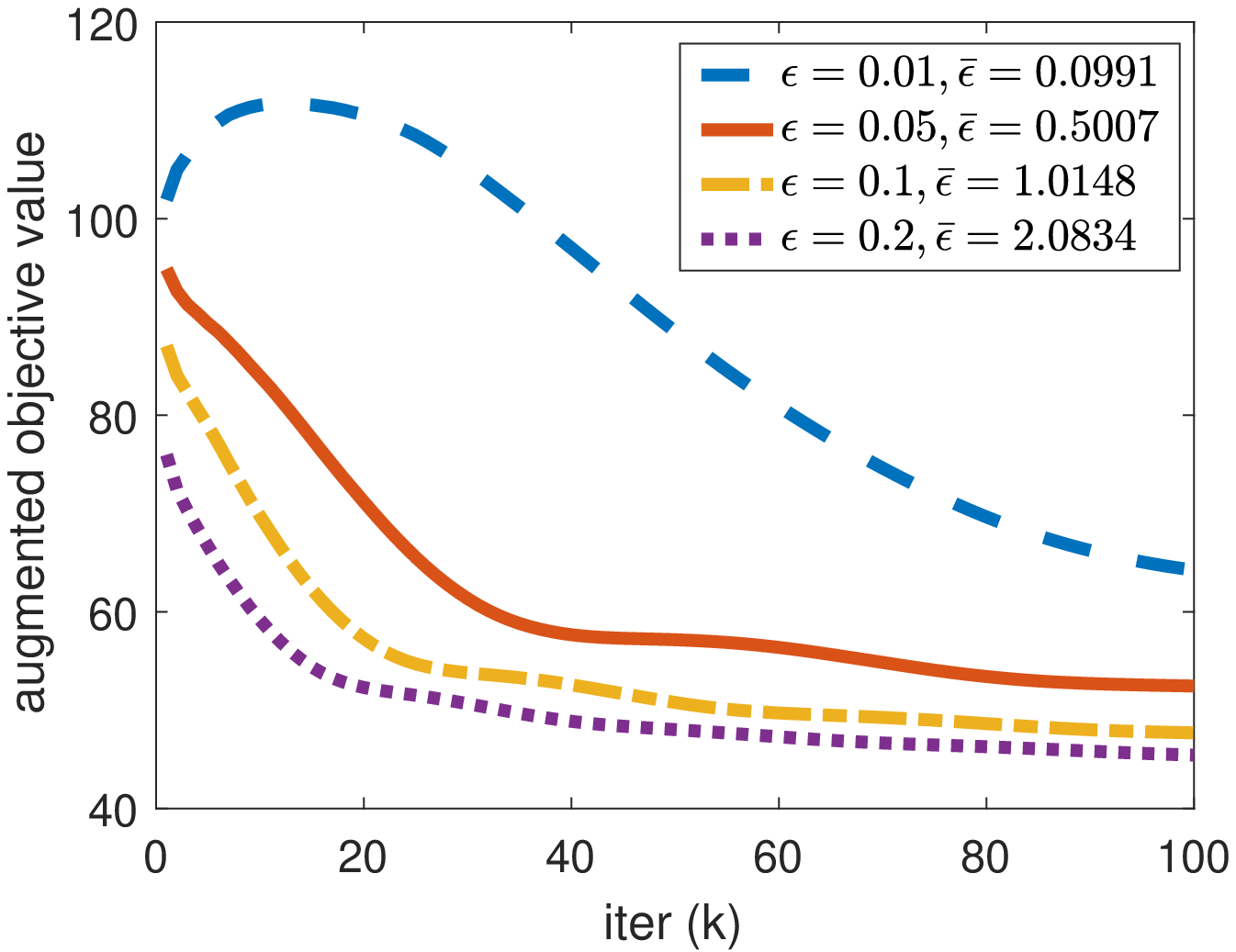}%
		}
		\hfil
		\subfloat[ $\delta = 10^{-5}$]{\includegraphics[width=1.7in]{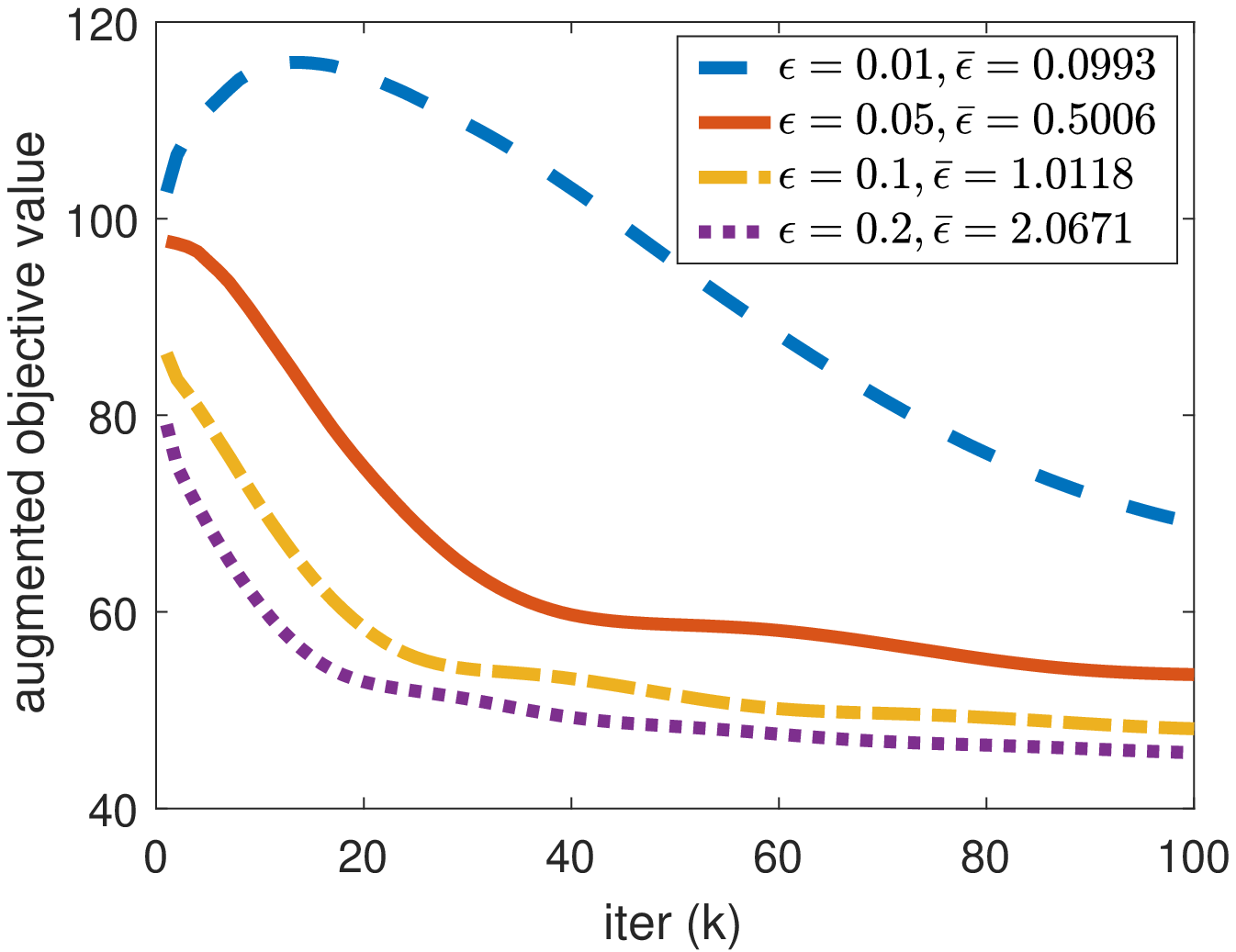}%
		}
		\hfil
		\subfloat[$\delta = 10^{-6}$]{\includegraphics[width=1.7in]{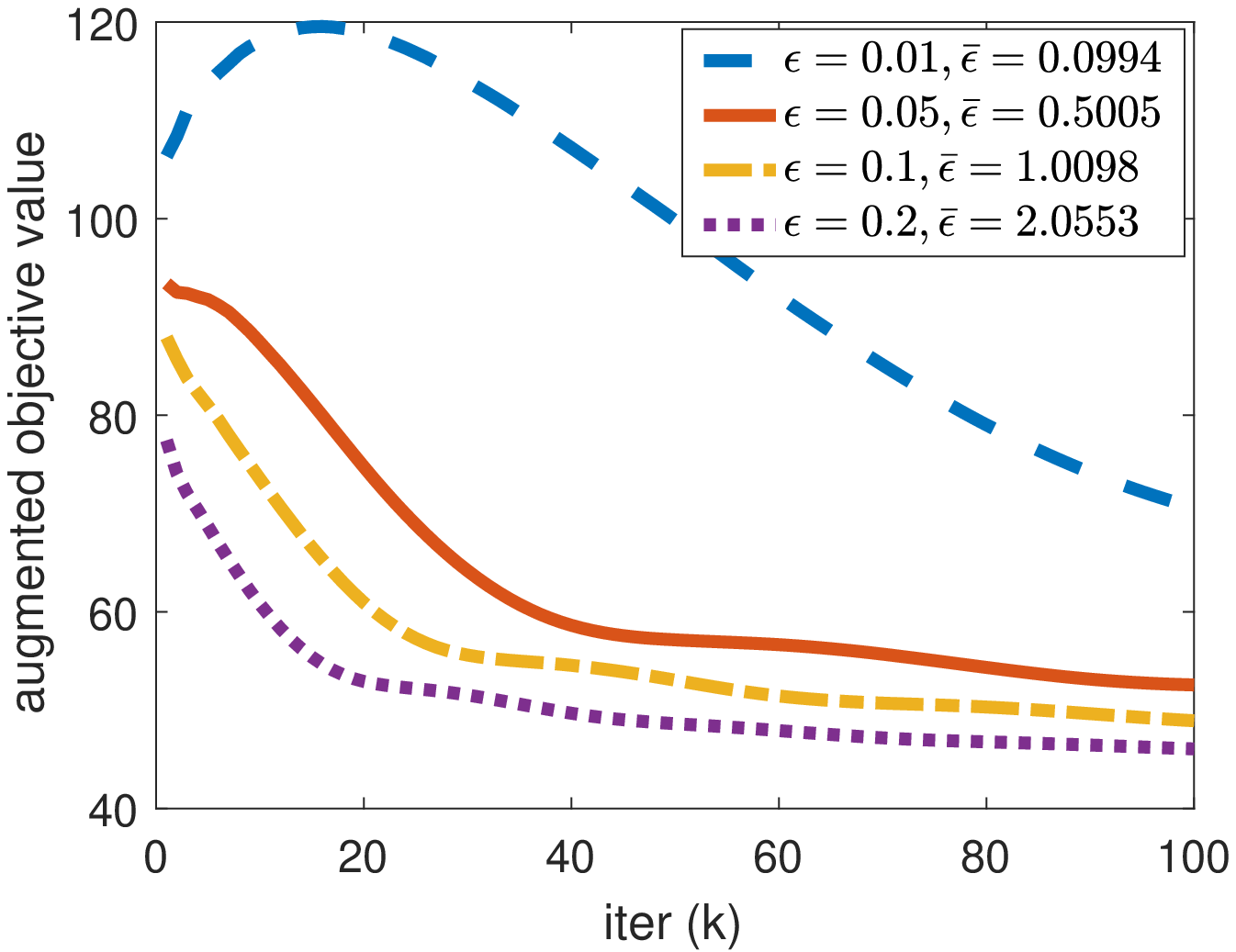}%
		}
		\hfil
		\caption{Convergence properties of DP-ADMM ($l_1$-regularized logistic regression).}
		\label{fig:2}
		\vspace*{-.15in}
	\end{figure*}
	%		\caption{$L_2$-regularized logistic regression: DP-ADMM vs the existing algorithm on convergence}
	
	\begin{figure*}[t]
		\centering
		\subfloat[$\epsilon = 0.05$, $\bar{\epsilon} = 0.5009$, $\delta = 10^{-3}$]{\includegraphics[width=1.7in]{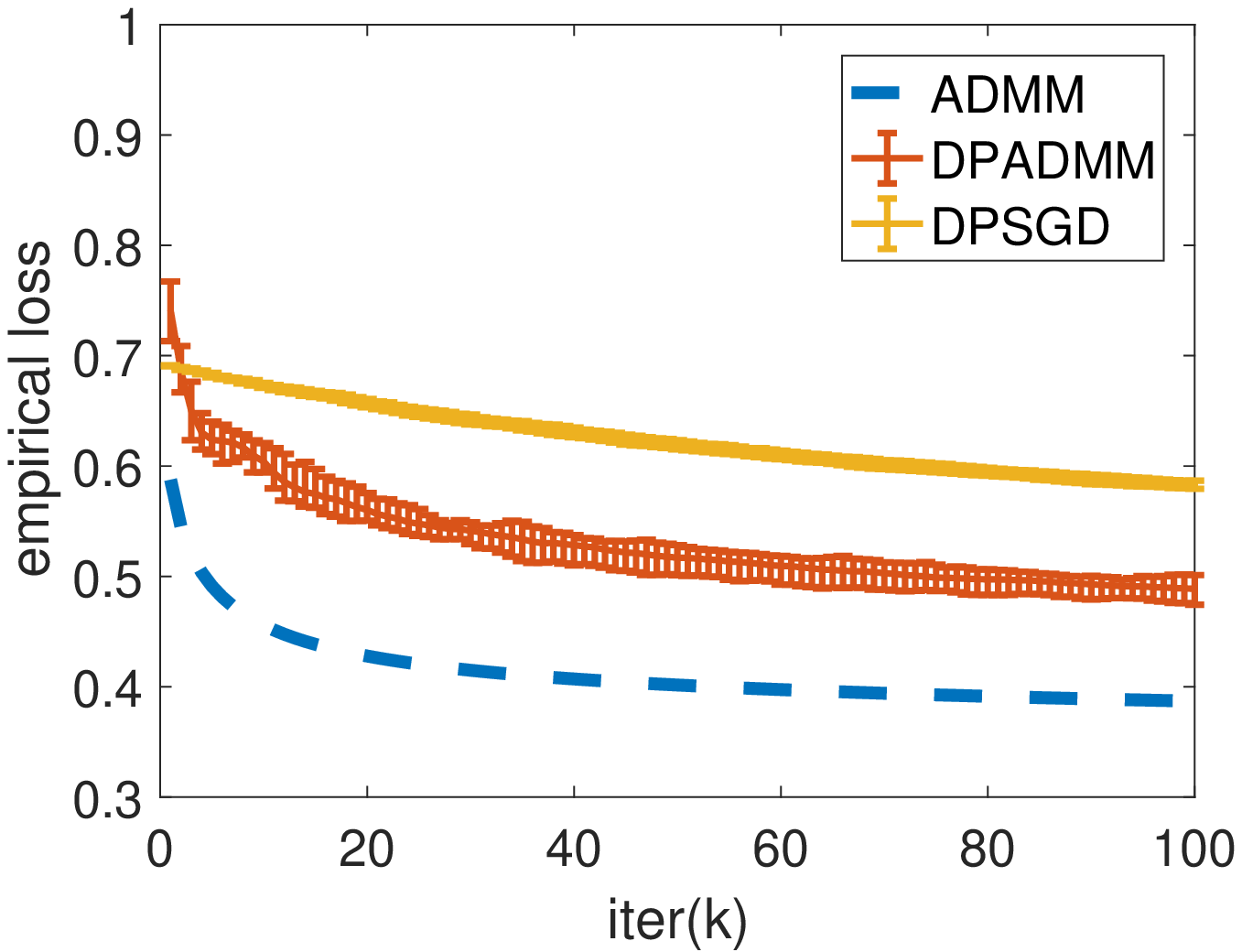}%
		}
		\hfil
		\subfloat[$\epsilon = 0.1$, $\bar{\epsilon} = 1.0193$, $\delta = 10^{-3}$]{\includegraphics[width=1.7in]{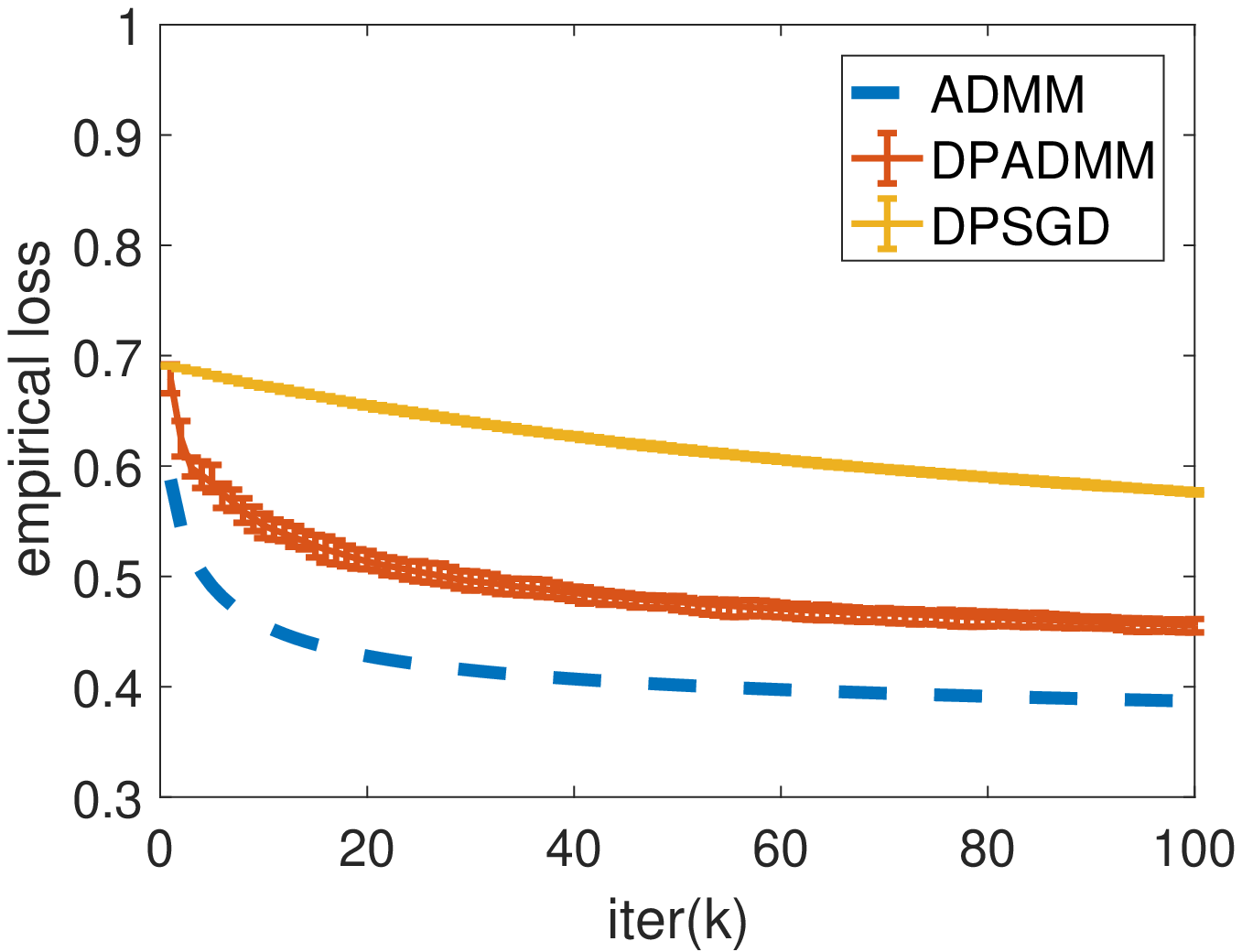}%
		}
		\hfil
		\subfloat[$\delta = 10^{-3}$]{\includegraphics[width=1.7in]{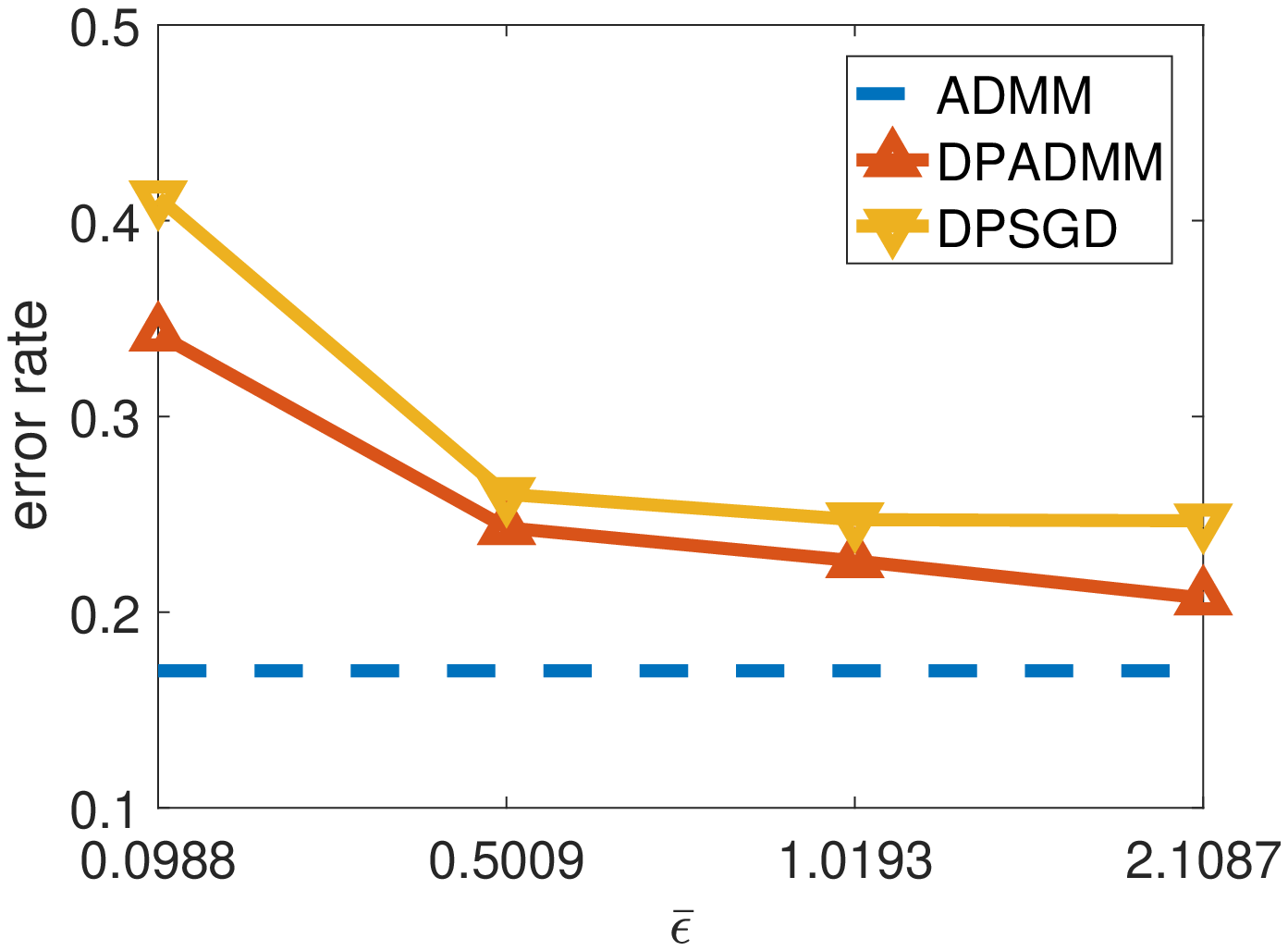}%
		}
		\hfil
		\subfloat[$\epsilon = 0.1$]{\includegraphics[width=1.7in]{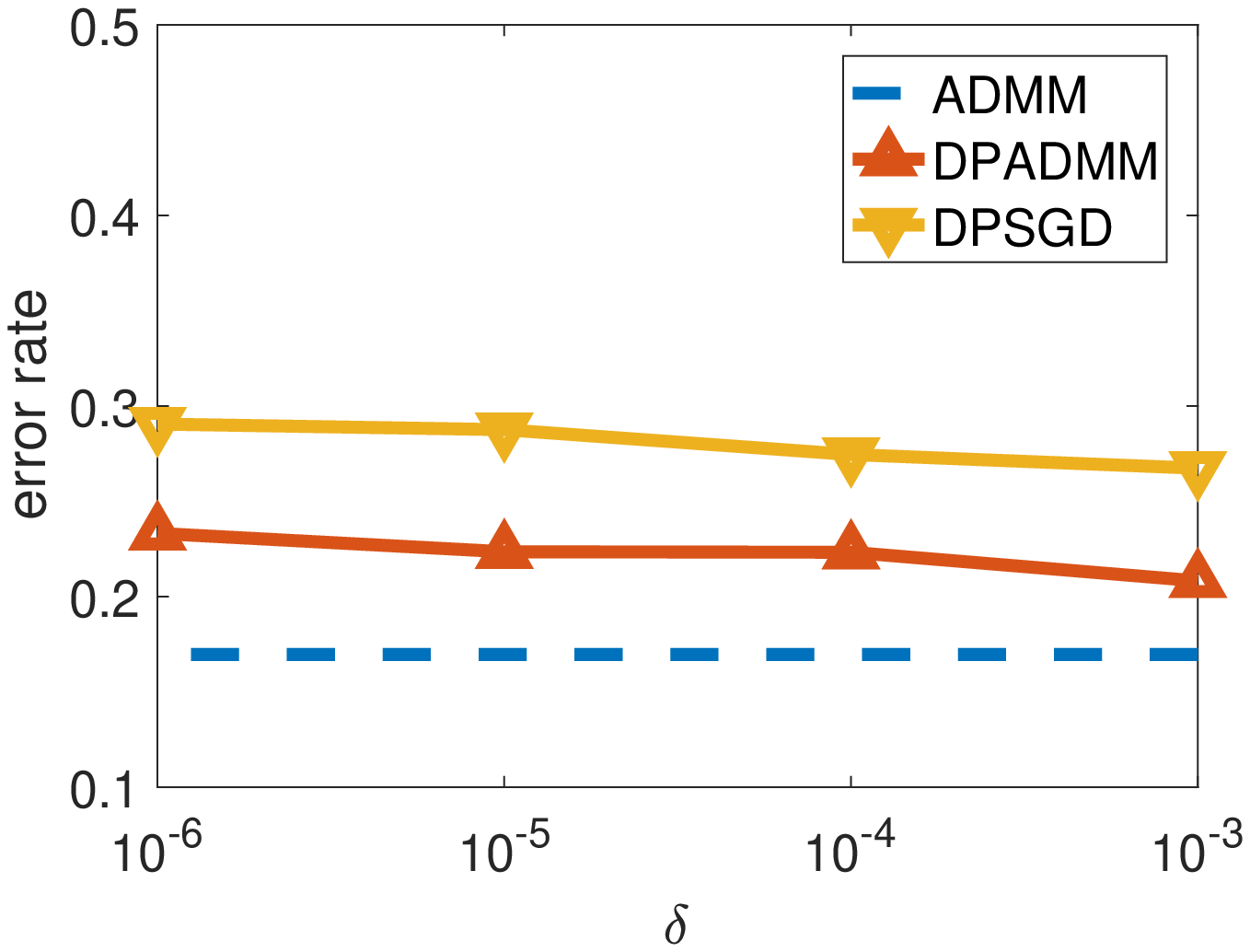}%
		}
		\hfil
		\caption{Accuracy comparison in empirical loss and classification error rate ($l_1$-regularized logistic regression).}
		\label{fig:3}
		\vspace*{-.15in}
	\end{figure*}
	
	In this section, we evaluate the performance of DP-ADMM with both non-smooth objectives and smooth objectives by considering logistic regression problems with $l_1$-norm and $l_2$-norm regularizers, respectively. %Our experiments include a training phase and a testing phase. In the training phase, we train classifiers by our approach and the baseline algorithms based on training data, and meanwhile monitor the training process. In the testing phase, we test the accuracy of the trained classifiers based on the testing data, by comparing the predicted labels by the classifiers and their original labels. 

	\textbf{Dataset.} We evaluate our approach on a real-world dataset: Adult dataset \cite{data:Ad} from UCI Machine Learning Repository. Adult dataset includes $48,842$ instances. Each instance has $14$ attributes such as age, sex, education, occupation, marital status, and native country, and is associated with a label representing whether the income is above $\$ 50,000$ or not. Before the simulation, we firstly preprocess the data by removing all the instances with missing values, converting the categorical attributes into binary vectors, normalizing columns to guarantee the maximum value of each column is 1, normalizing rows to enforce their $l_2$ norm to be less than 1, and converting the labels $\{>50k, <50k\}$ into $\{+1,-1\}$. After this, we obtain $45,222$ entries each with a $104$-dimensional feature vector ($d =104$) and a $1$-dimensional label belonging to $\{+1,-1\}$ ($p=1$).  In each simulation, we sample $40,000$ instances for training, and the remaining $5,222$ instances for testing. In the training process, we divide the training data into $n$ groups randomly, and thus each group contains $40000/n$ data points ($m_i = 40000/n$).
	
	\textbf{Baseline algorithms.} We compare our DP-ADMM (Algorithm \ref{ag:1}) with five baseline algorithms: (1) non-private centralized approach, (2) ADMM algorithm (Algorithm \ref{ag:2}), (3) ADMM algorithm with PVP (Algorithm \ref{ag:3}), %also called primal variable perturbation (PVP); 
	(4) ADMM with dual variable perturbation (DVP) in \cite{ZhZh17}, and (5) differentially private stochastic gradient descent (DPSGD) in \cite{AbCh16} for distributed settings. We evaluate the accuracy and effectiveness of our approach by comparing it with the five baseline algorithms.
	
	\textbf{Setup.} We set up the simulation by MATLAB in an Intel(R) Core(TM) 3.40 GHz computer with 16 GB RAM. In the simulation, we set the total iteration number $t = 100$ and the penalty parameter $\rho = 0.1$, and choose the optimal regularizer parameter $\lambda/n$ to be $10^{-6}$ by $10$-cross-validation in non-private setting. In DPSGD, we set the optimal learning rate to be $0.1$ and the sampling ratio to be $1$. We focus on the settings with strong privacy guarantee and thus we set privacy budget per iteration $\epsilon = \{0.01 ,0.05, 0.1, 0.2\}$ and $\delta = \{10^{-3}, 10^{-4}, 10^{-5}, 10^{-6}\}$, and use moments accountant method to obtain the corresponding total privacy loss $\bar{\epsilon}$. In each simulation, we run it for $10$ times to get averaged result.

	\begin{table}
		\vspace*{.15in}
		\small
		\caption{Computation Time ($100$ iterations).}
		\begin{center} 		\label{Table 3}
			\renewcommand\arraystretch{1.3}
			\begin{tabular}{|*{5}{c|}}
				\hline
				& ADMM & PVP & DVP & DPADMM \\ \hline
				$\epsilon = 0.01$ & $67.242 $s & $102.282$s & $59.743$s & $6.937$s \\ \hline
				$\epsilon = 0.05$ & $67.242 $s & $78.798 $s & $65.935$s & $5.322 $s\\ \hline
				$\epsilon = 0.1$ & $67.242 $s & $79.013 $s & $69.855$s & $5.218 $s \\ \hline
			\end{tabular}
		\end{center}
		\vspace*{-.3in}
	\end{table}
	
	\textbf{Evaluations.} We consider logistic regression problem in a distributed setting and evaluate our approach for logistic regression problems with $l_1$-norm and $l_2$-norm regularizers respectively, in terms of convergence, accuracy, and computation cost. The loss function of binary logistic regression is defined by \eqref{eq:logi}. The convergence properties are evaluated with respect to the augmented objective value, which measures the loss as well as the constraint penalty and is defined as  $ \sum_{i=1}^{n}\big( f_i(\boldsymbol{\bar{w}}_i^{k})+ \rho \Vert \boldsymbol{\bar{w}}_i^{k} - \boldsymbol{\bar{w}}^{k} \Vert \big)$. We evaluate the accuracy by empirical loss $\frac{1}{n}\sum_{i=1}^n \sum_{j=1}^{m_i }\frac{1}{m_i} \ell(\boldsymbol{a}_{i,j}, \boldsymbol{b}_{i,j},\boldsymbol{\tilde{w}}_i^k)$, and classification error rate. We measure the computation cost using the running time of training.

	\subsection{$L_1$-Regularized Logistic Regression}
	
	We obtain the DP-ADMM steps for $l_1$ regularized logistic regression by:
	\begin{subequations}
		\begin{align}
		\boldsymbol{w}_i^{k} = &  \bigg(\frac{1}{m_i}\sum_{j=1}^{m_i} \frac{\boldsymbol{b}_{i,j}\boldsymbol{a}_{i,j}}{1+\exp(\boldsymbol{b}_{i,j}\boldsymbol{\tilde{w}}_i^{{k-1}^{\intercal}}\boldsymbol{a}_{i,j})}   -   \frac{\lambda}{n}  \text{sgn}(\boldsymbol{\tilde{w}}_i^{k-1}) \nonumber \\  + &   \boldsymbol{\gamma}_i^{k-1}  +    \rho\boldsymbol{w}^{k-1} + \boldsymbol{\tilde{w}}_i^{k-1}/\eta_i^{k}\bigg)\bigg(\rho+1/\eta_i^{k}\bigg)^{-1}, \\ 
		\boldsymbol{\tilde{w}}_i^{k} = & \boldsymbol{w}_i^{k} +\mathcal{MN}_{d,p}(0, \sigma_{i,k}^{2}\boldsymbol{\mathrm{I}}_d, \sigma_{i,k}^{2}\boldsymbol{\mathrm{I}}_p ),  \\
		\boldsymbol{w}^{k} = &\frac{1}{n}\sum_{i=1}^n\boldsymbol{\tilde{w}}^{k}_i - \frac{1}{n}\sum_{i=1}^n\boldsymbol{\gamma}^{k-1}_i/\rho, \\
		\boldsymbol{\gamma}_i^{k} = &\boldsymbol{\gamma}_{i}^{k-1} - \rho \left(\boldsymbol{\tilde{w}}_i^{k} - \boldsymbol{w}^{k}\right), 
		\end{align}
	\end{subequations} 
	where $\text{sgn}(\cdot) $ is the sign function.
	
	Since the $l_1$ regularized objective function is convex but non-smooth, we apply Theorem \ref{the:2} to set $\eta_i^{k}$. Since we enforce $\Vert \ell^{'}(\cdot) \Vert \leq 1$ by data preprocessing, and we have $ \Vert R^{'}(\cdot) \Vert \leq \sqrt{d p}$ ($d = 104$ and $p=1$), we set $c_1 = 1$ and $c_2 = \sqrt{104}$. We obtain $\boldsymbol{w}^*$ by pre-training and set $c_w$ to be $23$. According to Theorem \ref{the:2}, we set $\eta_i^{k}$ to be $ 23{\big( 2k(1+10^{-6} \sqrt{104}/n)^2+1664k\ln{(1.25/\delta)}/\big(m_i^2 \epsilon^2\big)  \big)}^{-\frac{1}{2}}$.
	
	Since PVP and DVP cannot be applied when the objective function is non-smooth, we only compare our approach with ADMM and DPSGD in this section. We first investigate the performance of our approach with different numbers of distributed data sources and compare it with the centralized approach. Figure \ref{fig:1} shows that the accuracy of our training model would decrease if we consider larger number of data sources. Since the size of local dataset is smaller for larger number of agents, more noise should be introduced to guarantee the same level of differential privacy, thus degrading the performance of the trained model. This is consistent with Theorem \ref{theo:1} that the noise magnitude is scaled by $1/m_i$. In following simulations, we consider the case when the number of agents $n$ equals $100$. Figure \ref{fig:2} demonstrates the convergence properties of our approach by showing how the augmented objective value converges for different $\epsilon$ and $\delta$. It shows that our approach with larger $\epsilon$ and larger $\delta$ has better convergence, which is consistent with Theorem \ref{the:2}. Finally, we evaluate the accuracy of our approach by empirical loss and classification error rate by comparing with ADMM and DPSGD. Figure \ref{fig:3} shows our approach outperforms DPSGD due to the faster convergence property, demonstrating the advantage of ADMM framework. In addition, Figure \ref{fig:3} shows the privacy-utility trade-off of our approach. When privacy leakage increases (larger $\epsilon$ and larger $\delta$), our approach achieves better utility.

	\subsection{$L_2$-Regularized Logistic Regression}
	\begin{figure*}[t]
		\centering
		\subfloat[$\epsilon = 0.05, \bar{\epsilon}=0.5009,\delta = 10^{-3}$]{\includegraphics[width=1.7in]{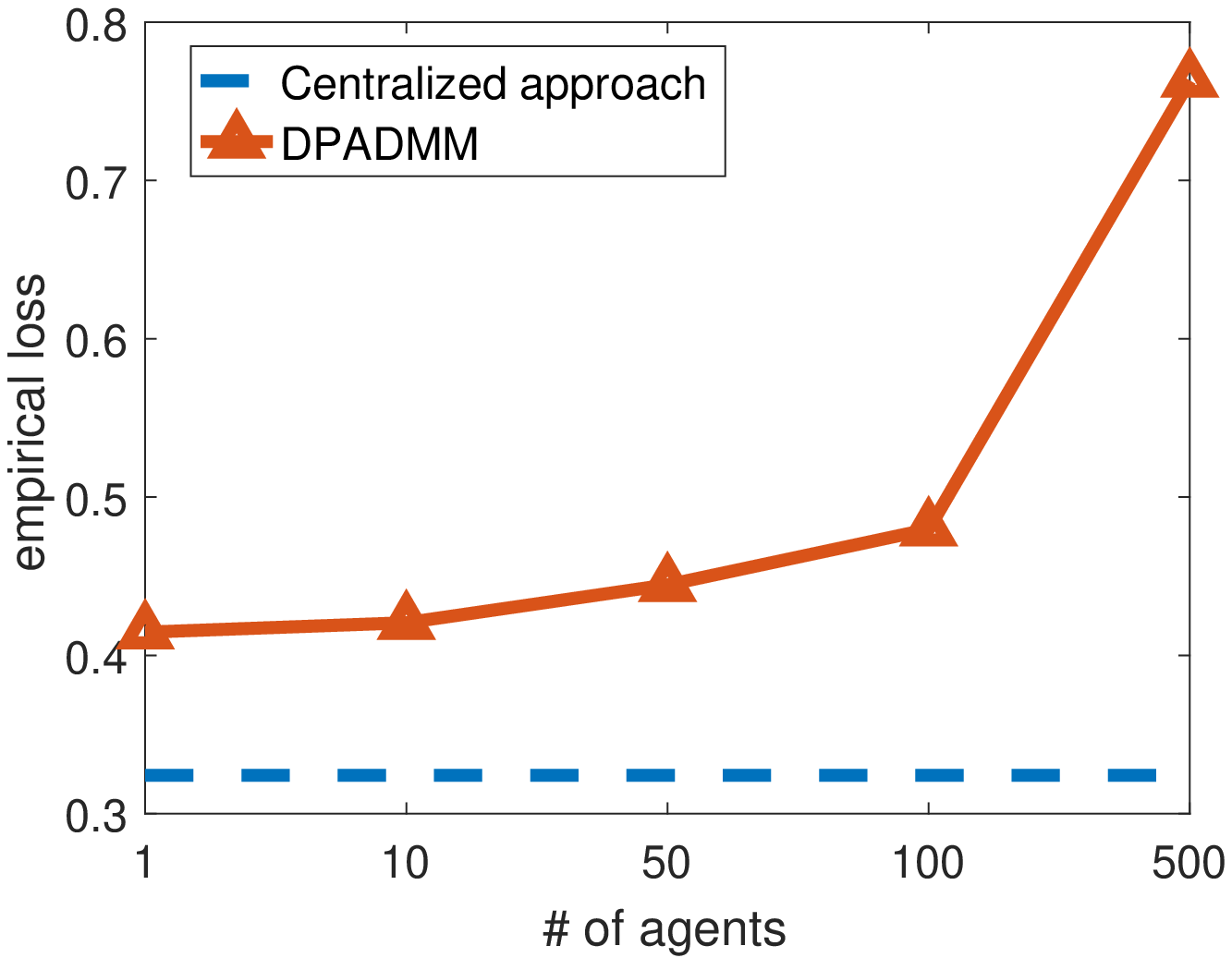}%
			\label{fig:41}}
		\hfil
		\subfloat[$\epsilon = 0.1, \bar{\epsilon}=1.0193,\delta = 10^{-3}$]{\includegraphics[width=1.7in]{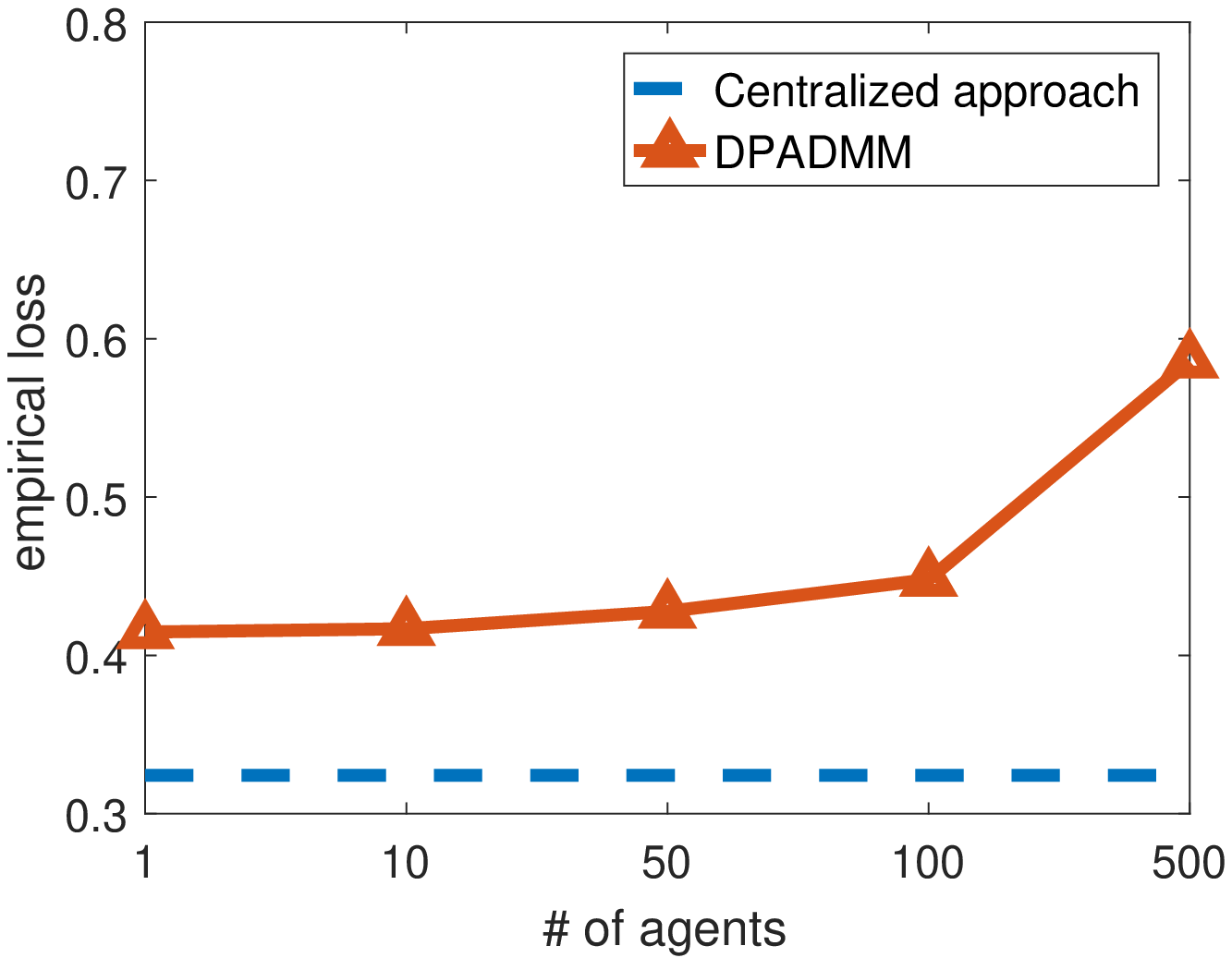}%
			\label{fig:42}}
		\hfil
		\subfloat[$\epsilon = 0.05, \bar{\epsilon}=0.5009,\delta = 10^{-3}$]{\includegraphics[width=1.7in]{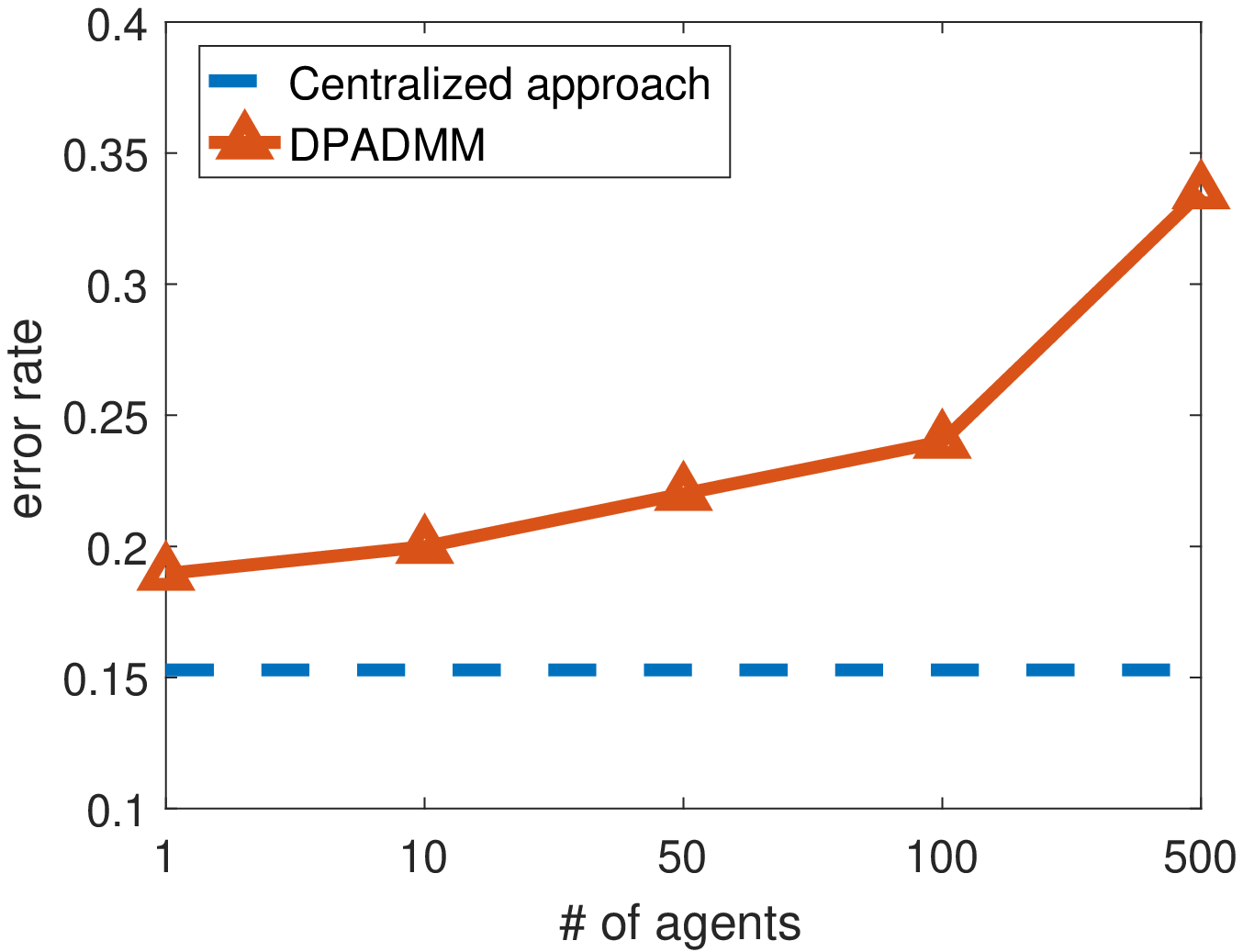}%
			\label{fig:43}	}
		\hfil
		\subfloat[$\epsilon = 0.1, \bar{\epsilon}=1.0193,\delta = 10^{-3}$]{\includegraphics[width=1.7in]{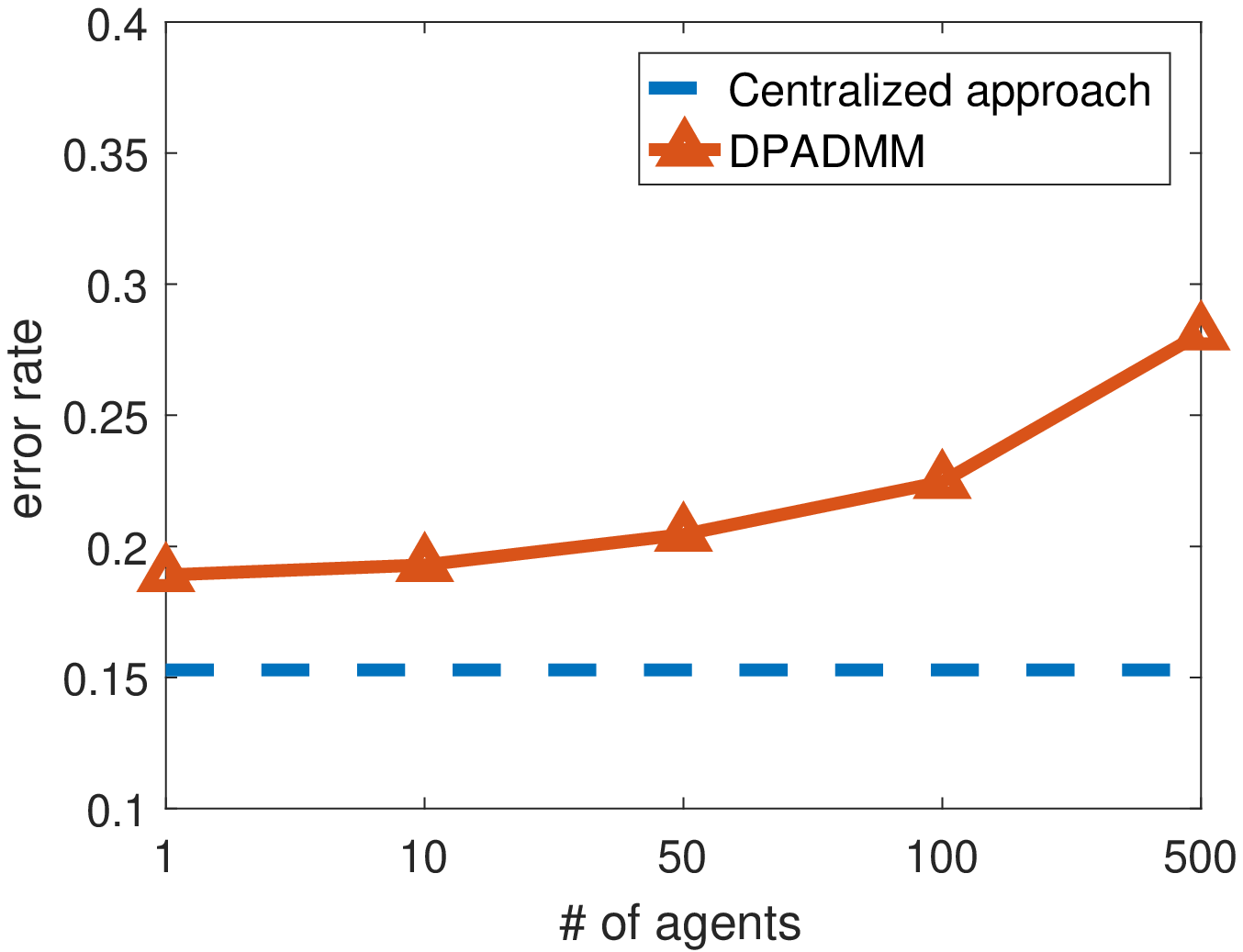}%
			\label{fig:44}	}
		\hfil
		\caption{Impact of distributed data source number on DP-ADMM ($l_2$-regularized logistic regression).}
		\label{fig:4}
		\vspace*{-.15in}
	\end{figure*}
	
	\begin{figure*}[t]
		\centering
		\subfloat[ $\delta = 10^{-3}$]{\includegraphics[width=1.7in]{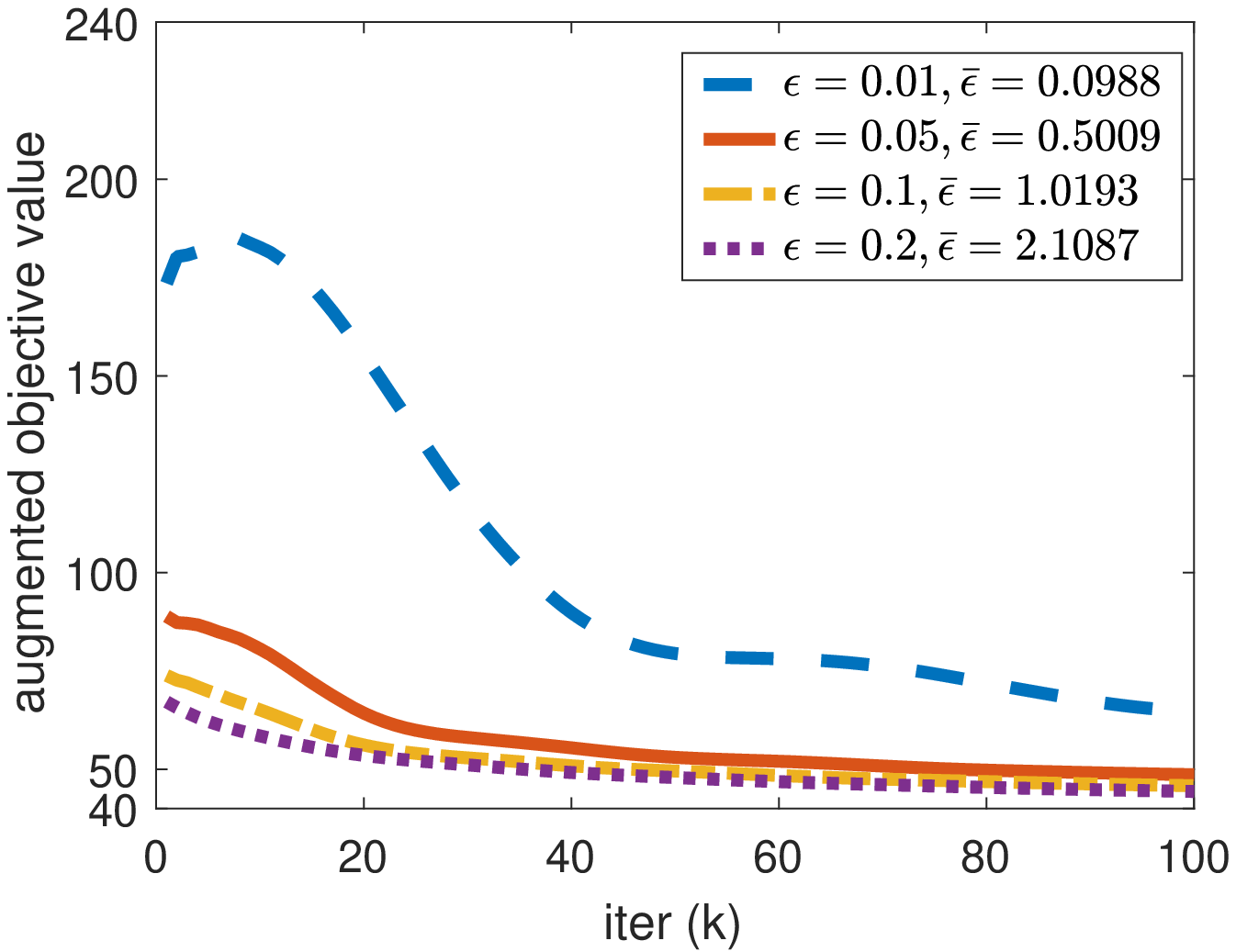}%
		}
		\hfil
		\subfloat[ $\delta = 10^{-4}$]{\includegraphics[width=1.7in]{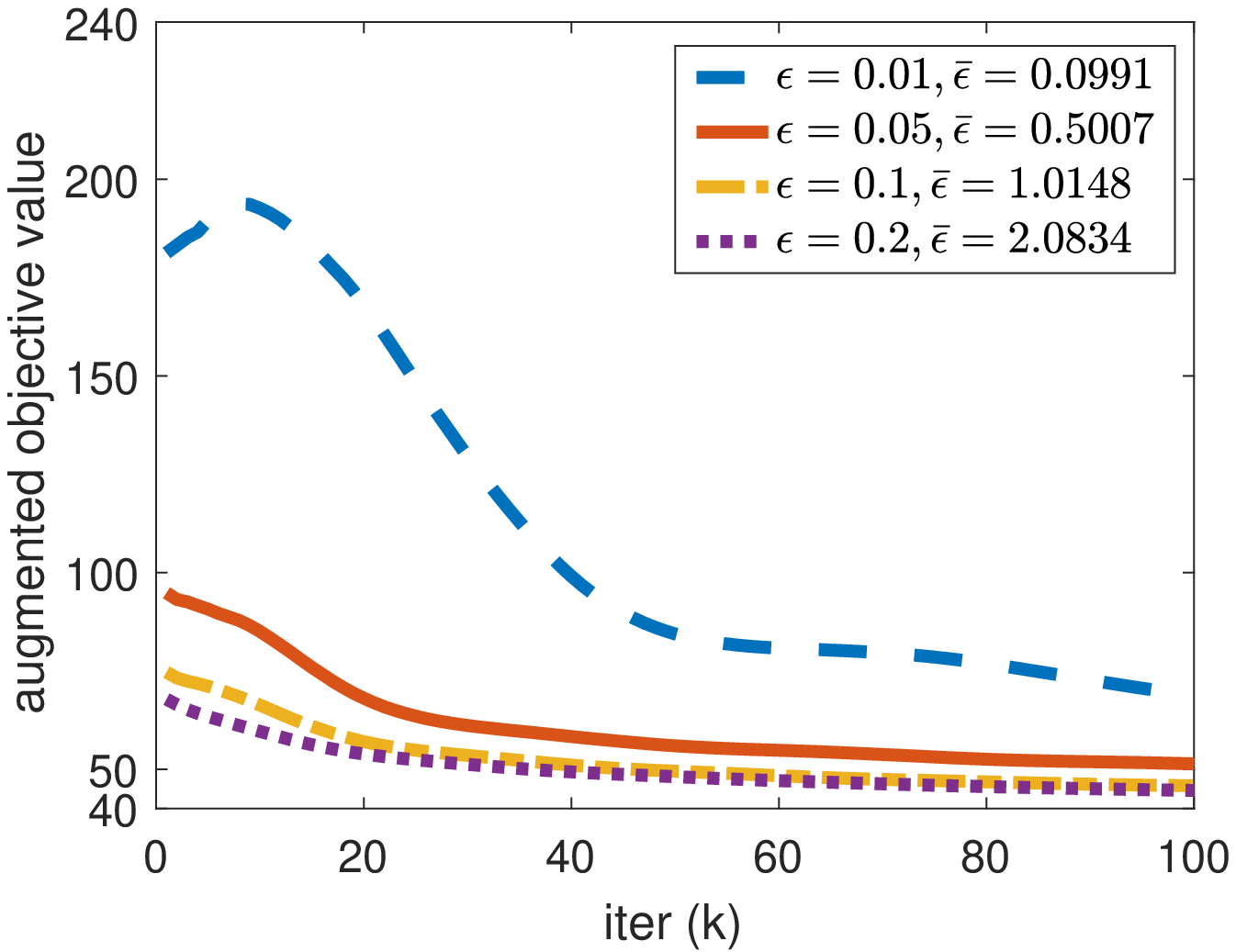}%
		}
		\hfil
		\subfloat[ $\delta = 10^{-5}$]{\includegraphics[width=1.7in]{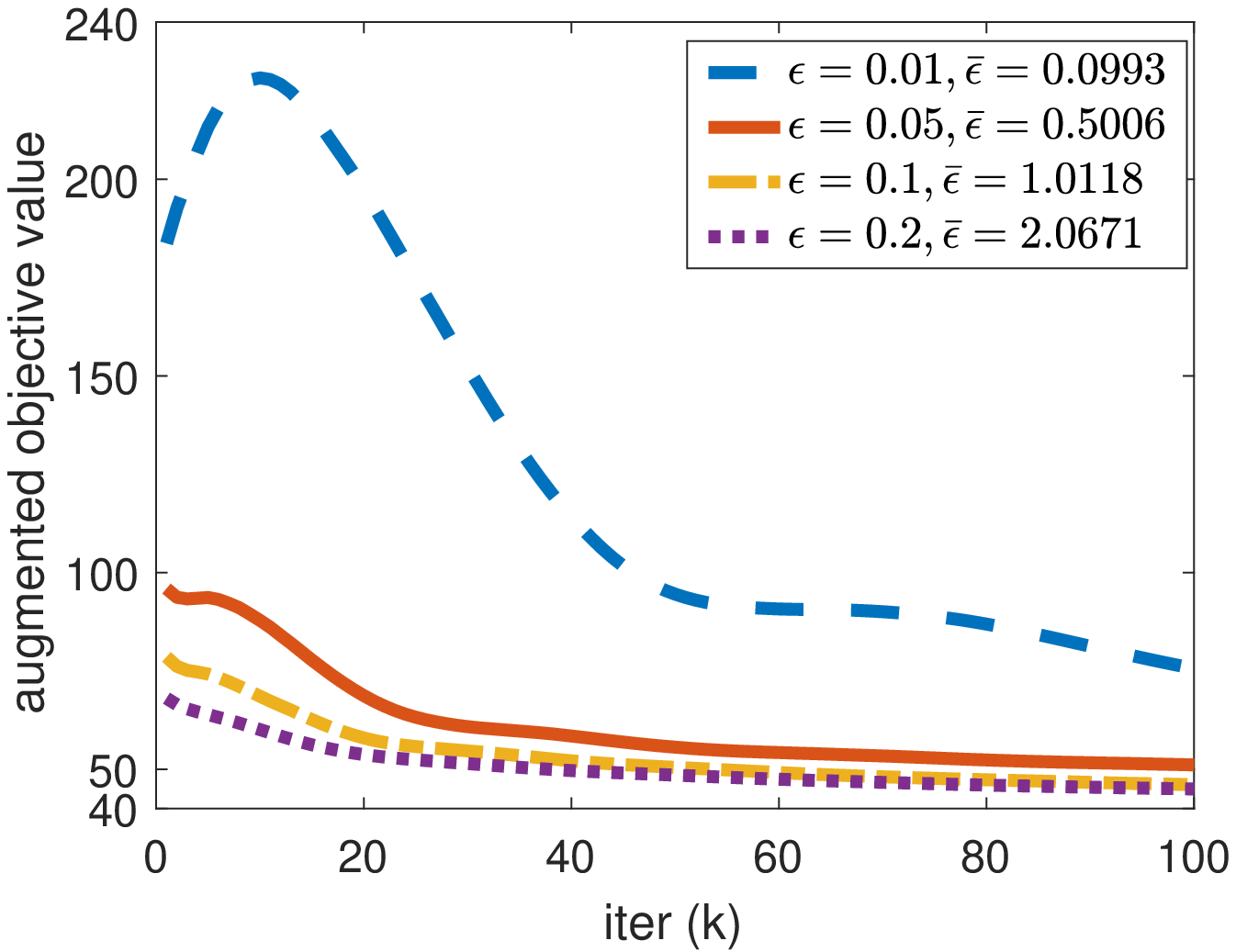}%
		}
		\hfil
		\subfloat[$\delta = 10^{-6}$]{\includegraphics[width=1.7in]{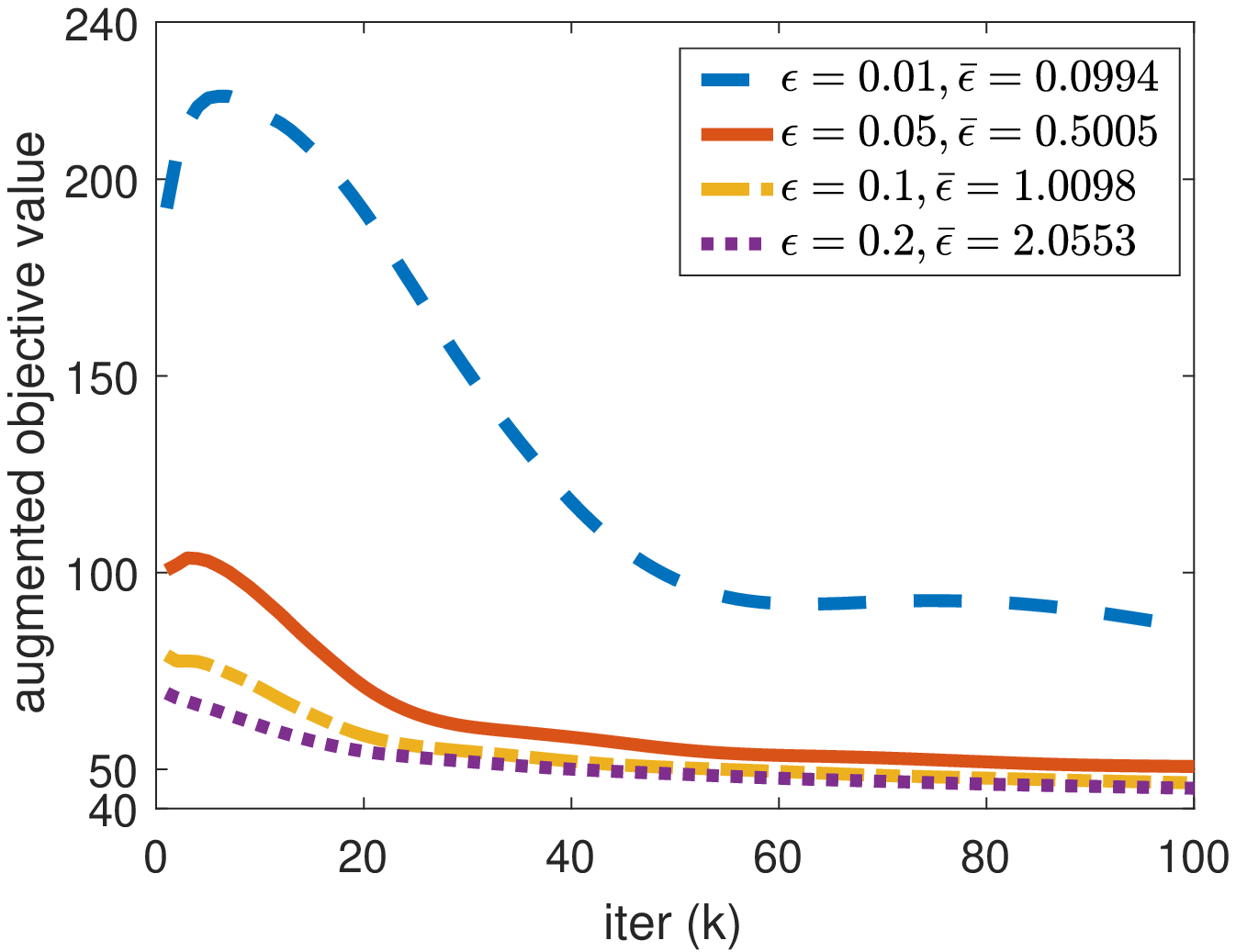}%
		}
		\hfil
		\caption{Convergence properties of DP-ADMM ($l_2$-regularized logistic regression).}
		\label{fig:6}
		\vspace*{-.15in}
	\end{figure*}
	%		\caption{$L_2$-regularized logistic regression: DP-ADMM vs the existing algorithm on convergence}
	
	\begin{figure*}[t]
		\centering
		\subfloat[$\epsilon = 0.05$, $\bar{\epsilon} = 0.5009$, $\delta = 10^{-3}$]{\includegraphics[width=1.7in]{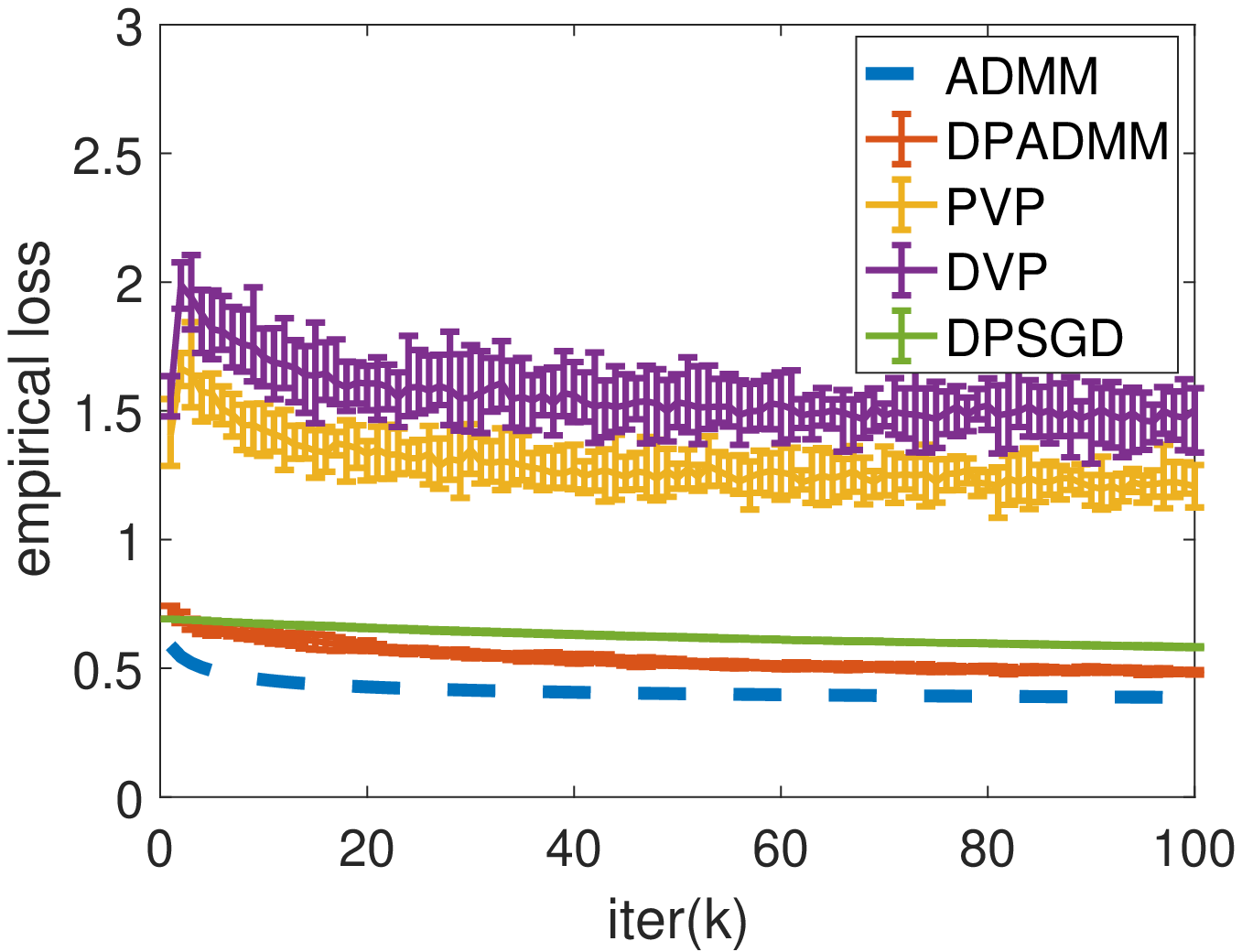}%
		}
		\hfil
		\subfloat[$\epsilon = 0.1$, $\bar{\epsilon} = 1.0193$, $\delta = 10^{-3}$]{\includegraphics[width=1.7in]{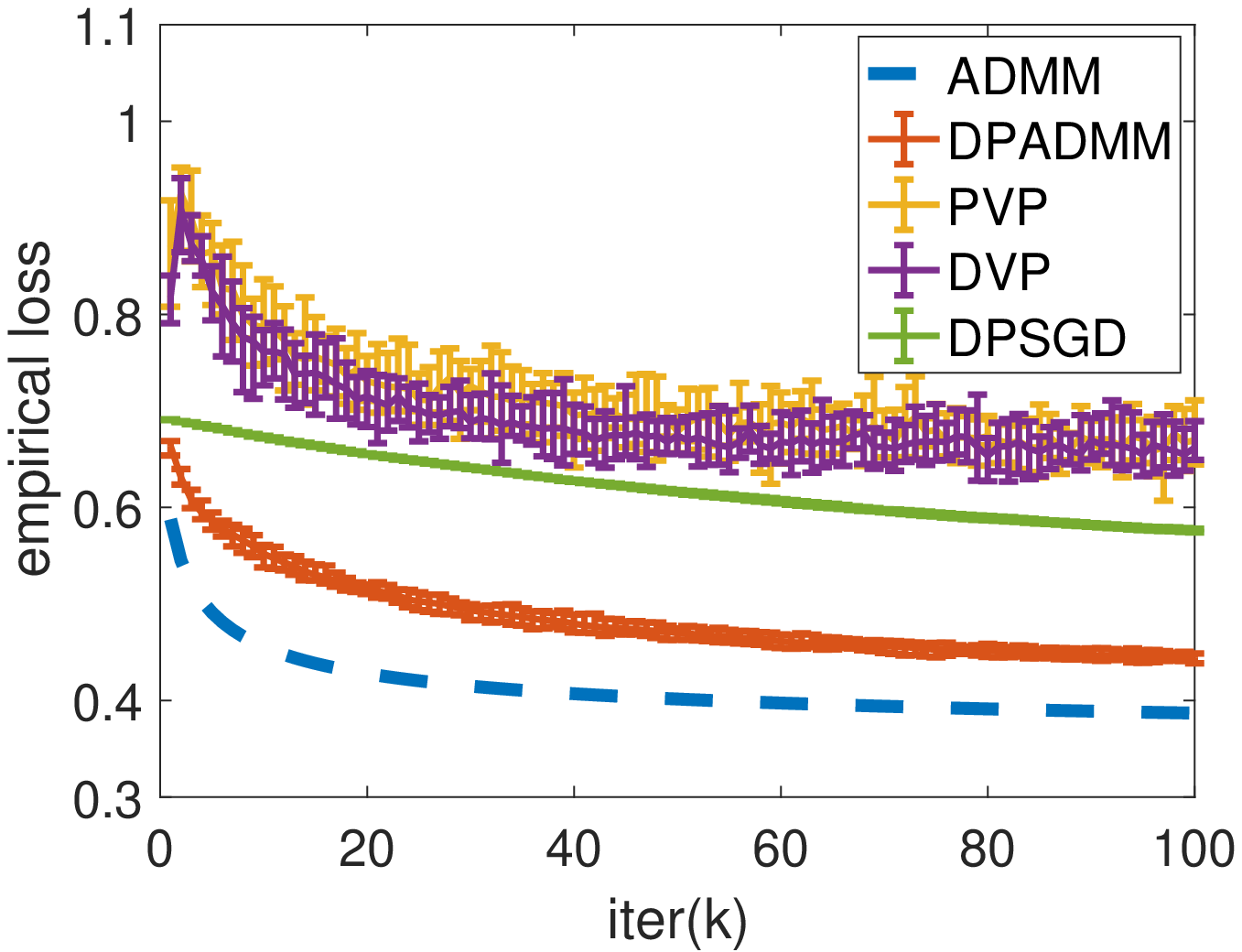}%
		}
		\hfil
		\subfloat[$\delta = 10^{-3}$]{\includegraphics[width=1.7in]{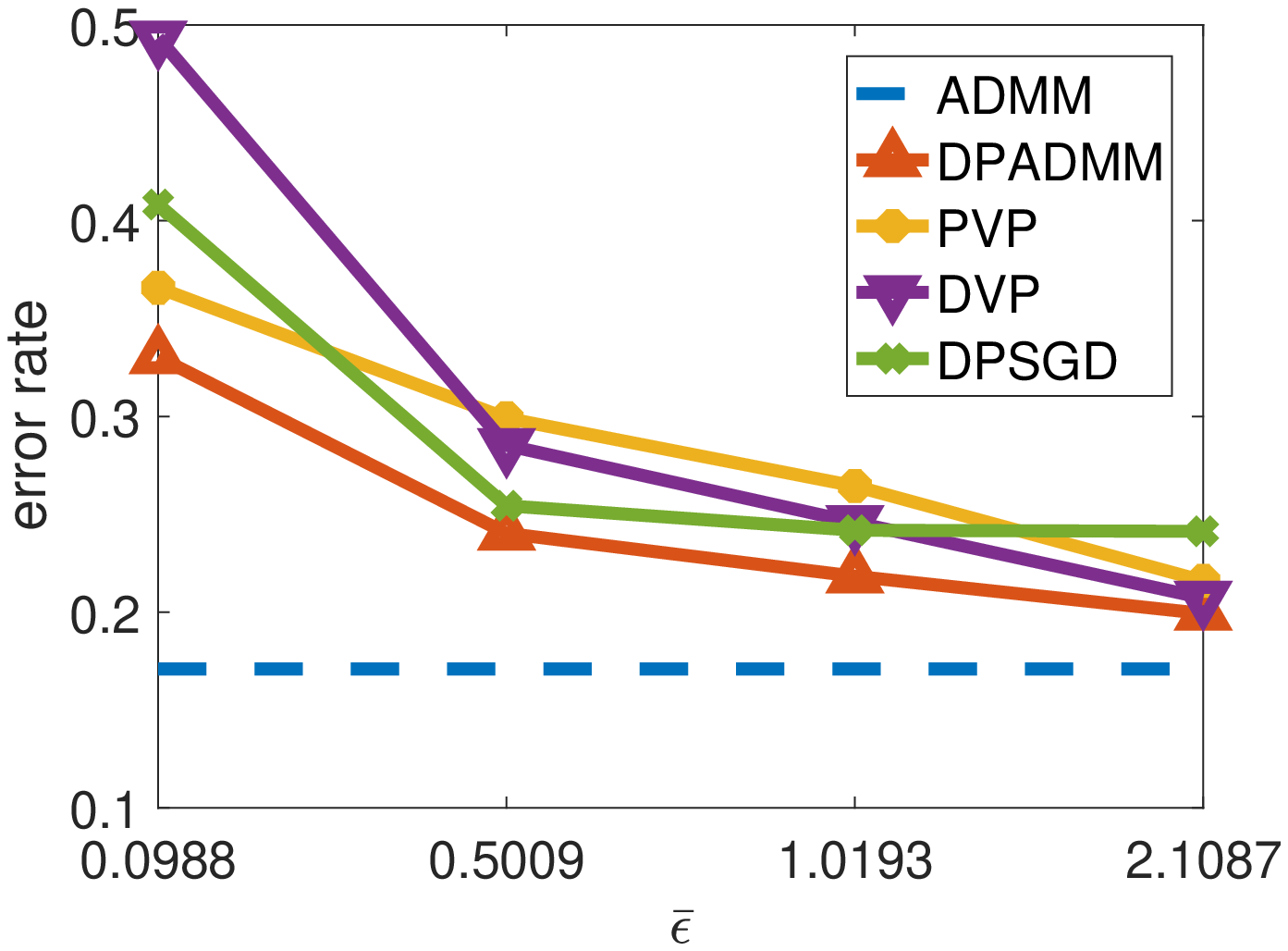}%
		}
		\hfil
		\subfloat[$\epsilon = 0.1$]{\includegraphics[width=1.7in]{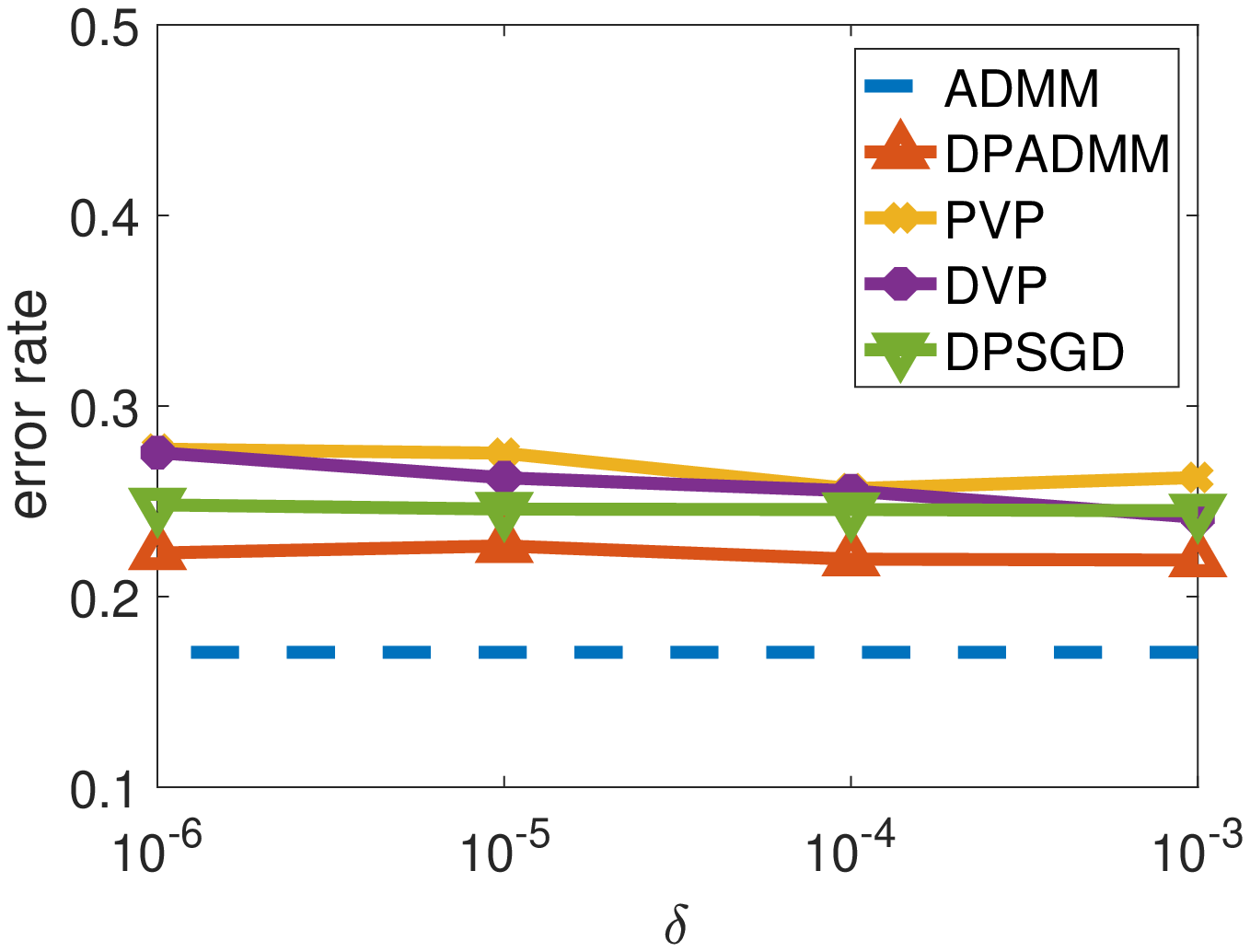}%
		}
		\hfil
		\caption{Accuracy comparison in empirical loss and classification error rate ($l_2$-regularized logistic regression).}
		\label{fig:7}
		\vspace*{-.15in}
	\end{figure*}
	
	The DP-ADMM steps for $l_2$ regularized logistic regression are described as follows:
	\begin{subequations}
		\begin{align}
		\boldsymbol{w}_i^{k} = &  \bigg(\frac{1}{m_i}\sum_{j=1}^{m_i} \frac{\boldsymbol{b}_{i,j}\boldsymbol{a}_{i,j}}{1+\exp(\boldsymbol{b}_{i,j}\boldsymbol{\tilde{w}}_i^{{k-1}^{\intercal}}\boldsymbol{a}_{i,j})}  -     \frac{\lambda}{n}\boldsymbol{\tilde{w}}_i^{k-1} + \boldsymbol{\gamma}_i^{k-1} \nonumber \\ & \quad \quad\quad +  \rho\boldsymbol{w}^{k-1}   + \boldsymbol{\tilde{w}}_i^{k-1}/\eta_i^{k}\bigg)\bigg(\rho+1/\eta_i^{k}\bigg)^{-1}, \\ 
		\boldsymbol{\tilde{w}}_i^{k} = & \boldsymbol{w}_i^{k} +\mathcal{MN}_{d,p}(0, \sigma_{i,k}^{2}\boldsymbol{\mathrm{I}}_d, \sigma_{i,k}^{2}\boldsymbol{\mathrm{I}}_p ),  \\
		\boldsymbol{w}^{k} = &\frac{1}{n}\sum_{i=1}^n\boldsymbol{\tilde{w}}^{k}_i - \frac{1}{n}\sum_{i=1}^n\boldsymbol{\gamma}^{k-1}_i/\rho, \\
		\boldsymbol{\gamma}_i^{k } = &\boldsymbol{\gamma}_{i}^{k-1} - \rho \left(\boldsymbol{\tilde{w}}_i^{k} - \boldsymbol{w}^{k}\right).
		\end{align}
	\end{subequations}
	
	Here the $l_2$ regularized objective function is convex and smooth, thus we apply Theorem \ref{the:3} to set $\eta_i^{k}$. Since we have $ \Vert {\nabla}^2 R(\cdot) \Vert \leq 1$, and we enforce $\Vert \nabla \ell(\cdot) \Vert \leq 1$ and $\Vert {\nabla}^2 \ell(\cdot) \Vert \leq 0.25$ by data preprocessing, thus we set $c_1 = 1$, $c_3 = 0.25$, and $c_4 =1$. We obtain the optimal solution $\boldsymbol{w}^{*}$ by pre-training, and set $c_w$ to be $89$. According to Theorem \ref{the:3}, we set $\eta_i^{k}$ to be $ {\big(0.25+10^{-6}+2\sqrt{416k\ln(1.25/\delta)}/\big(89 m_i \epsilon\big)\big)}^{-1}$.
	
	We fist investigate the performance of our approach under the settings with different numbers of distributed data sources and Figure \ref{fig:4} depicts the corresponding accuracy changes (accuracy decreases with increasing number of agents). Since the total data size is fixed, when we consider a larger number of agents, the size of local dataset is smaller, so the training model has lower accuracy due to more added noise for the same level of privacy guarantee. In the following simulations, we focus on the case where the number of agents is $100$. Next, we show the convergence properties of our approach. Figure \ref{fig:6} demonstrates that under weaker privacy guarantee (larger $\epsilon$ and larger $\delta$), our approach has better convergence, which is consistent with Theorem \ref{the:3}. We evaluate the accuracy of our approach by comparing it with ADMM, PVP, DVP, and DPSGD on empirical loss and classification error rate. Figure \ref{fig:7} shows that our approach outperforms PVP, DVP, and DPSGD. Specifically, ADMM has fast convergence but is sensitive to noise. Thus the methods directly perturbing intermediate results in ADMM (PVP and DVP) have poor performance. Gradient-based method (DPSGD) has good noise-resilience property but converges slowly. Our approach is based on ADMM framework, and combines the approximate augmented Lagrangian function with time-varying Gaussian noise addition to achieve higher utility. Furthermore, the results in Figure \ref{fig:7} also show the utility-privacy trade-off of our approach: larger $\epsilon$ and larger $\delta$ indicating weaker privacy guarantee would result in better utility. Finally, we show the advantage of our approach in computation cost by running time. Table \ref{Table 3} gives the comparison and shows that DP-ADMM has much less computation cost than all three ADMM baseline algorithms, which is resulted from the first-order approximation used in our approach enabling updates with closed-form solutions.

	\section{Related work}   \label{sec:rela}
	
	The existing literature related to our work could be categorized by: privacy-preserving empirical risk minimization, privacy-preserving distributed learning, and variants of ADMM.
	
	\textbf{Privacy-preserving empirical risk minimization.} There have been tremendous research efforts on privacy-preserving empirical risk minimization \cite{ChMo11,BaSm14,WangYe17,ThSm13}. Most of them focus on a centralized setting where sensitive data is collected and stored centrally, thus the privacy leakage comes from the final released trained model. Chaudhuri et al. \cite{ChMo11} propose two perturbation methods: output perturbation and objective perturbation to guarantee $\epsilon$-differential privacy. Bassily et al. \cite{BaSm14} provide a systematic investigation of differentially private algorithms for convex empirical risk minimization and propose efficient algorithms with tighter error bound. Wang et al. \cite{WangYe17} focus on a more general problem: non-convex problem, and propose a faster algorithm based on a proximal stochastic gradient method. Smith and Thakurta \cite{ThSm13} explore the stability of model selection problems, and propose two differentially private algorithms based on perturbation stability and subsampling stability respectively.
	
	\textbf{Privacy-preserving distributed learning.} Preserving privacy in distributed learning is challenging due to frequent information exchange in the iterative process. Recently, much works have been done to develop privacy-preserving distributed learning algorithms. Some of them employ cryptography-based methods in the protocol to hide the private information \cite{BoIv17,WaHu17,ZhAh19,gong2015privacy}. A recent work \cite{ZhAh19} uses partially homomorphic cryptography in ADMM-based distributed learning to preserve data privacy but the proposed approach cannot protect the information leakage of the private user data from the final learned models. In contrast, our approach provides differential privacy in the final trained machine learning models. Among the works on distributed learning with differential privacy, most of them focus on subgradient-based algorithms \cite{BeGu17,HaTo17,HaEg17,HuMi15} and only a few works consider ADMM-based methods \cite{ZhZh17,ZhKh18,ZhangKh18,DiEr19,guo2018practical}. Zhang and Zhu \cite{ZhZh17} propose two perturbation methods: primal perturbation and dual perturbation to guarantee dynamic differential privacy in ADMM-based distributed learning. Zhang et al. \cite{ZhKh18} propose to perturb the penalty parameter of ADMM to guarantee differential privacy. Zhang et al. \cite{ZhangKh18} propose recycled ADMM with differential privacy guarantee where the results from odd iterations could be re-utilized by the even iterations, and thus half of updates incur no privacy leakage. Guo and Gong \cite{guo2018practical} preserve differential privacy in the asynchronous ADMM algorithm. We design an ADMM-based distributed learning scheme with differential privacy which uses approximate augmented Lagrangian function for all iterations and adaptively changes the variance of added Gaussian noise in each iteration. We also use moments accountant method to analyze the total privacy loss to better estimate the trade-off between the data privacy and utility. We are the first to analyze rigorously the convergence rate and utility performance of ADMM with differential privacy.
	
	\textbf{Variants of ADMM.} Some variants of ADMM have been proposed recently for applicability to more generous problems. Linearized ADMM \cite{YaYu13,LinLiu11} replaces the quadratic function in the augmented Lagrangian function with a linearized approximation and thus provides a better way to solve subproblems without closed-form solutions.
	Stochastic ADMM \cite{OuHe13,AzSr14} considers stochastic and composite objective functions caused by natural uncertainties in observations. Our DP-ADMM algorithm inherits the features of linearized ADMM and stochastic ADMM, and guarantees strong differential privacy with good utility and low computation cost.

	\section{Conclusion}  \label{sec:conc}
	In this paper, we have proposed an improved ADMM-based differentially private distributed learning algorithm, DP-ADMM, for a class of learning problems that can be formulated as convex regularized empirical risk minimization. By designing an approximate augmented Lagrangian function and Gaussian mechanism with time-varying variance, our novel approach is noise-resilient, convergent and computation-efficient, especially under high privacy guarantee. We have also applied the moments accountant method to analyze the end-to-end privacy loss of the proposed iterative algorithm. The theoretical convergence guarantee and utility bound of our approach are derived. The evaluations on real-world datasets have demonstrated the effectiveness of our approach in the setting under high privacy guarantee.

	%	\section{Conclusion}  \label{sec:conc}
	%	Distributed learning is popular in many applications but privacy leakage still exist in the shared information. To address this problem, we have proposed an improved ADMM-based differentially private distributed learning algorithm: DP-ADMM for a class of regularized empirical risk minimization problems. Our proposed algorithm is noise-resilient and convergent, and has low computation cost. We have applied the \boldsymbol{a}_{i,j}, \boldsymbol{b}_{i,j} accountant method to bound the total privacy leakage of the proposed algorithm. Our privacy theorem has proved that DP-ADMM can guarantee differential privacy, and can be appled to a general class of convex learning problems. Our convergence theorems have shown that our approach can achieve an $O(1/\sqrt{t})$ rate of convergence, where $t$ is the number of iterations. The evaluations on real-world datasets have demonstrated the accuracy and effectiveness of our approach in the setting where the total privacy leakage is low.  
	
	% or
	
	
	%\columnbreak
	
	\bibliographystyle{IEEEtran}
	\bibliography{myrefer}
	
	\newpage
	
	\onecolumn
	
	\appendices
	
	\section{Lemma \ref{lem:5} ($l_2$ sensitivity of primal variable update in Algorithm \ref{ag:3})} \label{ap:sen}
	
	\begin{lemma} \label{lem:5}
		Assume the objective function is smooth, $R(\cdot)$ is $1$-strongly convex, and $\Vert  \nabla \ell(\cdot) \Vert \leq c_1$. The $l_2$ sensitivity of primal variable update in Algorithm \ref{ag:3} is defined by:
		\begin{equation}
		\max_{\mathcal{D}_i, \mathcal{D}_i^{'}} \Vert \boldsymbol{w}_{i,\mathcal{D}_i}^{k}  - \boldsymbol{w}_{i,\mathcal{D}^{'}_i}^{k}  \Vert =\frac{2 c_1}{(\lambda/n+\rho) m_i}. 
		\end{equation}
	\end{lemma}
	\begin{proof}
		We define:
		\begin{subequations}
			\begin{align}
			G(\boldsymbol{w}_i) &= \mathcal{L}_{\rho,i}(\boldsymbol{w}_i, \boldsymbol{w}^{k-1}, \boldsymbol{\gamma}_i^{k-1}), \nonumber \\
			g(\boldsymbol{w}_i) &= \frac{1}{m_i} \ell(\boldsymbol{a}^{'}_{i,m_i}, \boldsymbol{b}^{'}_{i,m_i},\boldsymbol{w}_i)- \frac{1}{m_i} \ell(\boldsymbol{a}_{i,m_i}, \boldsymbol{b}_{i,m_i},\boldsymbol{w}_i ).  \nonumber  
			\end{align}
		\end{subequations}
		According to the first step of ADMM, we have:
		\begin{subequations}
			\begin{align}
			\boldsymbol{w}_{i,\mathcal{D}_i}^{k} = &\argmin_{\boldsymbol{w}_i}G(\boldsymbol{w}_i),      \\
			\boldsymbol{w}_{i,\mathcal{D}_i^{'}}^{k} = &\argmin_{\boldsymbol{w}_i} G(\boldsymbol{w}_i)+g(\boldsymbol{w}_i). 
			\end{align}
		\end{subequations}
		Also by assuming the smoothness of the objective function, the functions $G(\cdot)$ and $G(\cdot)+g(\cdot)$ are smooth, thus we have:
		\begin{equation}
		\nabla G(\boldsymbol{w}_{i,\mathcal{D}_i}^{k}) =  \nabla G(\boldsymbol{w}_{i,\mathcal{D}_i^{'}}^{k})+\nabla g(\boldsymbol{w}_{i,\mathcal{D}_i^{'}}^{k}) = 0.
		\end{equation}
		
		Since we assume that the regularizer $R(\cdot)$ is $1$-strongly convex, then function $G(\cdot)$ is $(\lambda/n+\rho)$-strongly convex.
		From the Lemma 14 of \cite{shalev2007}, we have:
		\begin{equation}
		{\big(\nabla G(\boldsymbol{w}_{i,\mathcal{D}_i}^{k})-\nabla G(\boldsymbol{w}_{i,\mathcal{D}_i^{'}}^{k})\big)}^{\intercal} (\boldsymbol{w}_{i,\mathcal{D}_i}^{k} - 	\boldsymbol{w}_{i,\mathcal{D}_i^{'}}^{k})  \geq   (\lambda/n+\rho) {\Vert \boldsymbol{w}_{i,\mathcal{D}_i}^{k}- 	\boldsymbol{w}_{i,\mathcal{D}_i^{'}}^{k}\Vert}^2. 
		\end{equation}
		Combining this with the Cauchy-Schwartz inequality, we can get:
		\begin{equation} \label{eq:appb}
		\begin{split}
		\Vert \boldsymbol{w}_{i,\mathcal{D}_i}^{k}  -	\boldsymbol{w}_{i,\mathcal{D}_i^{'}}^{k} \Vert\cdot \Vert \nabla g(\boldsymbol{w}_{i,\mathcal{D}^{'}_i}^{k})\Vert  \geq& {\big(\nabla G(\boldsymbol{w}_{i,\mathcal{D}_i}^{k})-\nabla G(\boldsymbol{w}_{i,\mathcal{D}_i^{'}}^{k})\big)}^{\intercal} (\boldsymbol{w}_{i,\mathcal{D}_i}^{k} - 	\boldsymbol{w}_{i,\mathcal{D}_i^{'}}^{k}) \\ \geq & (\lambda/n+\rho) {\Vert \boldsymbol{w}_{i,\mathcal{D}_i}^{k}- 	\boldsymbol{w}_{i,\mathcal{D}_i^{'}}^{k}\Vert}^2. 
		\end{split}
		\end{equation}
		By dividing both sides of the above inequality by $(\lambda/n+\rho) \Vert \boldsymbol{w}_{i,\mathcal{D}_i}^{k}- 	\boldsymbol{w}_{i,\mathcal{D}_i^{'}}^{k}\Vert$, we can get:
		\begin{equation}
		\Vert \boldsymbol{w}_{i,\mathcal{D}_i}^{k}  - \boldsymbol{w}_{i,\mathcal{D}^{'}_i}^{k}  \Vert  \leq \frac{ \Vert  \nabla \ell(\boldsymbol{a}_{i,m_i}, \boldsymbol{b}_{i,m_i},\boldsymbol{w}_{i,\mathcal{D}^{'}_i}^{k})- \nabla \ell(\boldsymbol{a}^{'}_{i,m_i}, \boldsymbol{b}^{'}_{i,m_i},\boldsymbol{w}_{i,\mathcal{D}^{'}_i}^{k})\Vert}{m_i (\lambda/n+\rho)}. 
		\end{equation}
		As we assume that $\Vert  \nabla \ell(\cdot) \Vert \leq c_1$, then we obtain the result:
		\begin{equation}
		\max \Vert \boldsymbol{w}_{i,\mathcal{D}_i}^{k}  - \boldsymbol{w}_{i,\mathcal{D}^{'}_i}^{k}  \Vert = \frac{2 c_1}{(\lambda/n+\rho) m_i}. 
		\end{equation}
	\end{proof}
	
	\section{Proof of Theorem \ref{theo:4}} \label{ap:a}
	\begin{proof}
		We use the log moments of the privacy loss and their linear composability to get a tight bound of the total privacy loss. The $\tau^{th}$ log moment of the privacy loss of agent $i$ for $k$-th iteration could be defined by the log moment generating function at $\tau$:
		\begin{equation}
		\begin{split}
		\alpha_i^k (\tau)  = \ln\bigg( \mathbb{E}_{\boldsymbol{\tilde{w}}_i^{k}}\bigg[ \bigg(\frac{\Pr\big[\boldsymbol{\tilde{w}}_i^{k} \vert \mathcal{D}_i\big]}{\Pr\big[\boldsymbol{\tilde{w}}_i^{k} \vert \mathcal{D}^{'}_i\big]}\bigg)^{\tau}\bigg]\bigg).
		\end{split}
		\end{equation}
		In the $k$-th iteration of Algorithm \ref{ag:1}, we employ Gaussian mechanism with variance $\sigma_{i,k}^{2}$ to achieve ($\epsilon,\delta$)-differential privacy guarantee. We use $\mu_0$ to denote the probability density function (pdf) of $\mathcal{N}(0,\sigma_{i,k}^{2})$, and $\mu_1$ to denote the pdf of $\mathcal{N}(2c_1/\big(m_i(\rho+1/\eta_i^{k})\big),\sigma_{i,k}^{2})$. We obtain that $\alpha_i^k (\tau)$ by $\alpha_i^k (\tau) = \ln \big(\max (E_1, E_2)\big)$, where
		\begin{equation}
		E_1 = \mathbb{E}_{z\sim \mu_{0}}\bigg[\bigg(\frac{\mu_{0}(z)}{\mu_{1}(z)}\bigg)^{\tau}\bigg] \quad \text{and} \quad  E_2 =  \mathbb{E}_{z\sim \mu_{1}}\bigg[\bigg(\frac{\mu_{1}(z)}{\mu_{0}(z)}\bigg)^{\tau}\bigg]. \nonumber
		\end{equation}
		Since,
		\begin{subequations}
			\begin{align}
			\mathbb{E}_{z\sim \mu_{0}}\bigg[(\frac{\mu_{0}(z)}{\mu_{1}(z)})^{\tau}\bigg] = & \exp{\bigg(\frac{\tau(\tau+1)\epsilon^2}{4\ln(1.25/\delta)}\bigg)},  \\ 
			\mathbb{E}_{z\sim \mu_{1}}\bigg[(\frac{\mu_{1}(z)}{\mu_{0}(z)})^{\tau}\bigg] =& \exp{\bigg(\frac{\tau(\tau+1)\epsilon^2}{4\ln(1.25/\delta)}\bigg)}, 
			\end{align}
		\end{subequations}
		we have:
		\begin{equation}
		\alpha_i^k (\tau) =\frac{\tau(\tau+1)\epsilon^2}{4\ln(1.25/\delta)}. 
		\end{equation}
		According to Theorem 2 (linear composability) in \cite{AbCh16}, we have the $\tau^{th}$ log moment of the overall privacy loss from $i$:
		\begin{equation}
		\alpha_i (\tau)  = \sum_{k=1}^t \alpha_i^k (\tau) = \frac{t\tau(\tau+1)\epsilon^2}{4\ln(1.25/\delta)}. 
		\end{equation}
		We aim to prove that our proposed algorithm DP-ADMM (Algorithm \ref{ag:1}) achieves $(\bar{\epsilon}, \delta)$-differential privacy. According to Theorem 2 (tail bound) in \cite{AbCh16}, we  have:
		\begin{equation*}
		\delta = \min_{\tau \in \mathbb{Z}^{+}}   \exp(\alpha_i(\tau) - \tau \bar{\epsilon}) = \min_{\tau\in \mathbb{Z}^{+}}   \exp\bigg(\frac{t\tau(\tau+1)\epsilon^2}{4\ln(1.25/\delta)}- \tau \bar{\epsilon}\bigg). 
		\end{equation*}
		Since $\delta \in (0,1)$, there exists a positive integer $\tau$ to make $t\tau(\tau+1)\epsilon^2/\big(4\ln(1.25/\delta)\big)- \tau \bar{\epsilon} < 0$. Furthermore, $t\tau(\tau+1)\epsilon^2/\big(4\ln(1.25/\delta)\big)- \tau \bar{\epsilon}$ is a quadratic function w.r.t. $\tau$. Thus, if there is a solution to the above minimization problem, we must have: when $\tau=1$,
		\begin{equation}
		\frac{t\tau(\tau+1)\epsilon^2}{4\ln(1.25/\delta)}- \tau \bar{\epsilon} = \frac{ t\epsilon^2}{2\ln(1.25/\delta)} -  \bar{\epsilon}<0.  
		\end{equation}
		Therefore, we obtain:
		\begin{equation}
		\frac{t \epsilon^2}{2\ln(1.25/\delta)} < \bar{\epsilon}.	\label{eq:proof1_1}
		\end{equation}
		The minimum of $tx(x+1)\epsilon^2/\big(4\ln(1.25/\delta)\big)- x \bar{\epsilon}$ is $-t\epsilon^2/\big(16\ln(1.25/\delta)\big)+\bar{\epsilon}/2-\bar{\epsilon}^2\ln(1.25/\delta)/\big(t\epsilon^2\big)$ when $x \in \mathbb{R}$. Thus:
		\begin{equation}
		\ln(\delta) = \min_{\tau \in \mathbb{Z}^{+}}  \bigg(\frac{t\tau(\tau+1)\epsilon^2}{4\ln(1.25/\delta)}- \tau \bar{\epsilon}\bigg) \geq -\frac{t\epsilon^2}{16\ln(1.25/\delta)}+\frac{\bar{\epsilon}}{2}-\frac{\bar{\epsilon}^2\ln(1.25/\delta)}{t\epsilon^2}  \label{eq:proof1_2}
		\end{equation}
		From \eqref{eq:proof1_1} and \eqref{eq:proof1_2}, we obtain:
		\begin{equation}
		\begin{split}
		\ln(1/\delta)  \leq -\frac{3\bar{\epsilon}}{8} +\frac{\bar{\epsilon}^2\ln(1.25/\delta)}{t\epsilon^2}  \leq \frac{\bar{\epsilon}^2\ln(1.25/\delta)}{t\epsilon^2},
		\end{split}
		\end{equation}
		which leads to the following inequality: 
		\begin{equation}
		\bar{\epsilon} \geq \sqrt{\frac{t\ln(1/\delta)}{\ln(1.25/\delta)}}\epsilon .
		\end{equation}
		Therefore, there exists a constant $c_0$, the overall privacy loss $\bar{\epsilon}$ satisfies:
		\begin{equation}
		\bar{\epsilon} = c_0 \sqrt{t}\epsilon.
		\end{equation}
	\end{proof}
	
	\section{Lemma \ref{lem:1} used in the proof of Lemma \ref{lem:2}} \label{ap:lem2}
	%	The following lemma is used in the proof of Lemma \ref{lem:2} and Lemma \ref{lem:3} to obtain the upper bound of the first order approximation based on each iteration point.
	\begin{lemma} \label{lem:1}
		Assume $L(\cdot)$ is a convex differentiable function. $s \geq 0$ is a scalar. For any vector $\boldsymbol{x} \in \mathbb{R}^d$ and $\boldsymbol{y} \in \mathbb{R}^d$, we denote their Bregman divergence as $D(\boldsymbol{x},\boldsymbol{y}) = h(\boldsymbol{x}) - h(\boldsymbol{y}) - \big\langle \nabla h(\boldsymbol{y}) ,\boldsymbol{x}-\boldsymbol{y}  \big\rangle$, where $h(\cdot)$ is a continuously-differentiable real-valued and strictly convex function. If we define:
		\begin{equation}
		\boldsymbol{x}^{*} = \argmin_{\boldsymbol{x}} L(\boldsymbol{x}) + s D(\boldsymbol{x},\boldsymbol{y}),
		\end{equation}
		then
		\begin{equation}
		\big\langle \nabla L(\boldsymbol{x}^{*}) , \boldsymbol{x}^{*}-\boldsymbol{x} \big\rangle \leq  s\big(D(\boldsymbol{x},\boldsymbol{y})-D(\boldsymbol{x},\boldsymbol{x}^{*})-D(\boldsymbol{x}^{*},\boldsymbol{y})\big).
		\end{equation}
	\end{lemma}
	\begin{proof}
		According to the optimality condition,
		\begin{equation}
		\big\langle \nabla L(\boldsymbol{x}^{*})+  s \nabla D(\boldsymbol{x}^{*},\boldsymbol{y}), \boldsymbol{x}-\boldsymbol{x}^{*}  \big\rangle \geq 0.
		\end{equation}
		Then,
		\begin{equation}
		\begin{split}
		\big\langle \nabla L(\boldsymbol{x}^{*}) , \boldsymbol{x}^{*}-\boldsymbol{x} \big\rangle  \leq & s\big\langle  \nabla D(\boldsymbol{x}^{*},\boldsymbol{y}), \boldsymbol{x}-\boldsymbol{x}^{*}  \big\rangle\\
		= & s\big\langle \nabla h(\boldsymbol{x}^{*})-\nabla h(\boldsymbol{y}), \boldsymbol{x}-\boldsymbol{x}^{*} \big\rangle \\
		= & s\big(D(\boldsymbol{x},\boldsymbol{y})-D(\boldsymbol{x},\boldsymbol{x}^{*})-D(\boldsymbol{x}^{*},\boldsymbol{y})\big).
		\end{split}
		\end{equation}
	\end{proof}

	\section{Proof of Lemma \ref{lem:2}}\label{ap:b}
	\begin{proof}
		%	We first define:
		%		\begin{equation}
		%	\begin{split}
		%		 \phi(\boldsymbol{\tilde{w}}_i^{k},\eta_i^{k}) = &   f^{'}_i(\boldsymbol{\tilde{w}}_i^{k})-  \boldsymbol{\gamma}_i^{k}+ \rho(\boldsymbol{\tilde{w}}_i^{k} - \boldsymbol{w}^{k-1})\\&\quad\quad\quad\quad-(\rho+1/\eta_i^{k}) \boldsymbol{\xi}_i^{k}. 
		%		\end{split}
		%		\end{equation}
		
		Since we assume that $\ell(\cdot)$ and $R(\cdot)$ are convex, the function $f_i(\cdot)$ is convex. Due to the convexity of $f_i(\cdot)$, we have:
		\begin{equation}
		f_i(\boldsymbol{\tilde{w}}_i^{k-1})-f_i(\boldsymbol{w}_i) \leq  \big\langle f^{'}_i(\boldsymbol{\tilde{w}}_i^{k-1}), \boldsymbol{\tilde{w}}_i^{k-1} - \boldsymbol{w}_i \big\rangle,
		\end{equation}
		which can lead to:
		\begin{equation}
		\begin{split}
		f_i(\boldsymbol{\tilde{w}}_i^{k-1})-f_i(\boldsymbol{w}_i)+ \big\langle \boldsymbol{\tilde{w}}_i^{k}- \boldsymbol{w}_i, - \boldsymbol{\gamma}_{i}^{k} \big\rangle 
		\leq & \big\langle f^{'}_i(\boldsymbol{\tilde{w}}_i^{k-1}), \boldsymbol{\tilde{w}}_i^{k-1} - \boldsymbol{w}_i \big\rangle+ \big\langle \boldsymbol{\tilde{w}}_i^{k}- \boldsymbol{w}_i, - \boldsymbol{\gamma}_{i}^{k} \big\rangle\\
		%		= &  \big\langle f^{'}_i(\boldsymbol{\tilde{w}}_i^{k-1})-(\rho+1/\eta_i^{k}) \boldsymbol{\xi}_i^{k}, \boldsymbol{\tilde{w}}_i^{k} - \boldsymbol{w}_i \big\rangle \\& -\big\langle (\rho+1/\eta_i^{k}) \boldsymbol{\xi}_i^{k}, \boldsymbol{w}_i - \boldsymbol{\tilde{w}}_i^{k} \big\rangle \\ & + \big\langle f^{'}_i(\boldsymbol{\tilde{w}}_i^{k-1})-(\rho+1/\eta_i^{k}) \boldsymbol{\xi}_i^{k}, \boldsymbol{\tilde{w}}_i^{k-1} - \boldsymbol{\tilde{w}}_i^{k} \big\rangle \\ & + \big\langle \boldsymbol{\tilde{w}}_i^{k}- \boldsymbol{w}_i, - \boldsymbol{\gamma}_{i}^{k} \big\rangle\\
		= &  \big\langle f^{'}_i(\boldsymbol{\tilde{w}}_i^{k-1})-(\rho+1/\eta_i^{k}) \boldsymbol{\xi}_i^{k}, \boldsymbol{\tilde{w}}_i^{k-1} - \boldsymbol{\tilde{w}}_i^{k} \big\rangle - \big(\rho+1/\eta_i^{k}\big) \big\langle\boldsymbol{\xi}_i^{k}, \boldsymbol{w}_i - \boldsymbol{\tilde{w}}_i^{k-1} \big\rangle \\&+ \big\langle f^{'}_i(\boldsymbol{\tilde{w}}_i^{k-1})-  \boldsymbol{\gamma}_{i}^{k}-(\rho+1/\eta_i^{k}) \boldsymbol{\xi}_i^{k}, \boldsymbol{\tilde{w}}_i^{k} - \boldsymbol{w}_i \big\rangle. \label{eq:hh}  
		\end{split}
		\end{equation}
		According to the Line $10$ of Algorithm \ref{ag:1}, we have:
		\begin{equation}
		\begin{split}
		\big\langle f^{'}_i(\boldsymbol{\tilde{w}}_i^{k-1})-  \boldsymbol{\gamma}_{i}^{k}-(\rho+1/\eta_i^{k}) \boldsymbol{\xi}_i^{k}, \boldsymbol{\tilde{w}}_i^{k} - \boldsymbol{w}_i \big\rangle  = &  \big\langle f^{'}_i(\boldsymbol{\tilde{w}}_i^{k-1})-  \boldsymbol{\gamma}_i^{k-1}+ \rho(\boldsymbol{\tilde{w}}_i^{k} - \boldsymbol{w}^{k-1})-(\rho+1/\eta_i^{k}) \boldsymbol{\xi}_i^{k},   \boldsymbol{\tilde{w}}_i^{k} - \boldsymbol{w}_i \big\rangle \\& + \big\langle \boldsymbol{\tilde{w}}_i^{k} - \boldsymbol{w}_i, \rho(\boldsymbol{w}^{k-1} - \boldsymbol{w}^{k}) \big\rangle. \label{eq:hh1}
		\end{split}
		\end{equation}
		By combining \eqref{eq:hh} and \eqref{eq:hh1}, we obtain:
		\begin{equation}
		\begin{split}
		f_i(\boldsymbol{\tilde{w}}_i^{k-1})-f_i(\boldsymbol{w}_i)+ \big\langle \boldsymbol{\tilde{w}}_i^{k}- \boldsymbol{w}_i, - \boldsymbol{\gamma}_{i}^{k} \big\rangle  
		\leq &  \big\langle f^{'}_i(\boldsymbol{\tilde{w}}_i^{k-1})-(\rho+1/\eta_i^{k}) \boldsymbol{\xi}_i^{k}, \boldsymbol{\tilde{w}}_i^{k-1} - \boldsymbol{\tilde{w}}_i^{k} \big\rangle\\&+ \big\langle f^{'}_i(\boldsymbol{\tilde{w}}_i^{k-1})-  \boldsymbol{\gamma}_i^{k-1}+ \rho(\boldsymbol{\tilde{w}}_i^{k} - \boldsymbol{w}^{k-1}) -(\rho+1/\eta_i^{k}) \boldsymbol{\xi}_i^{k},  \boldsymbol{\tilde{w}}_i^{k} - \boldsymbol{w}_i \big\rangle  \\& + \big\langle \boldsymbol{\tilde{w}}_i^{k} - \boldsymbol{w}_i, \rho(\boldsymbol{w}^{k-1} - \boldsymbol{w}^{k}) \big\rangle -\big(\rho+1/\eta_i^{k}\big) \big\langle\boldsymbol{\xi}_i^{k}, \boldsymbol{w}_i - \boldsymbol{\tilde{w}}_i^{k-1} \big\rangle .  \label{eq:3}
		\end{split}
		\end{equation}
		We handle the last three terms separately. Firstly, we have:
		\begin{equation}
		\begin{split}
		\big\langle \boldsymbol{\tilde{w}}_i^{k} - \boldsymbol{w}_i, \rho(\boldsymbol{w}^{k-1} - \boldsymbol{w}^{k}) \big\rangle 
		= & \frac{\rho}{2}\big({\Vert \boldsymbol{w}_i - \boldsymbol{w}^{k-1} \Vert}^2 - {\Vert \boldsymbol{w}_i - \boldsymbol{w}^{k} \Vert}^2 \big)+ \frac{\rho}{2}\big({\Vert \boldsymbol{\tilde{w}}_i^{k} - \boldsymbol{w}^{k} \Vert}^2 -{\Vert \boldsymbol{\tilde{w}}_i^{k} - \boldsymbol{w}^{k-1} \Vert}^2\big)\\
		\leq & \frac{\rho}{2}\big({\Vert \boldsymbol{w}_i - \boldsymbol{w}^{k-1} \Vert}^2 - {\Vert \boldsymbol{w}_i - \boldsymbol{w}^{k} \Vert}^2\big)+ \frac{\rho}{2}{\Vert \boldsymbol{\tilde{w}}_i^{k} - \boldsymbol{w}^{k} \Vert}^2 \\
		= & \frac{\rho}{2}\big({\Vert \boldsymbol{w}_i - \boldsymbol{w}^{k-1} \Vert}^2 - {\Vert \boldsymbol{w}_i - \boldsymbol{w}^{k} \Vert}^2\big) + \frac{1}{2\rho}{\Vert \boldsymbol{\gamma}_{i}^{k}-\boldsymbol{\gamma}_i^{k-1} \Vert}^2. \label{eq:1}
		\end{split}
		\end{equation}
		According to the Line 4 and 6 of Algorithm \ref{ag:1}, $\boldsymbol{\tilde{w}}_i^{k}$ is equal to the solution to $\min_{\boldsymbol{w}_i} \big\langle  f^{'}_i(\boldsymbol{\tilde{w}}_i^{k-1}),  \boldsymbol{w}_i - \boldsymbol{\tilde{w}}_i^{k-1} \big\rangle -\big\langle \boldsymbol{\gamma}_i^{k-1} ,\boldsymbol{w}_i- \boldsymbol{w}^{k-1} \big\rangle + \rho {\Vert\boldsymbol{w}_i- \boldsymbol{w}^{k-1}\Vert}^2/2 + {\Vert \boldsymbol{w}_i - \boldsymbol{\tilde{w}}_i^{k-1}\Vert}^2/(2\eta_i^{k})-(\rho+1/\eta_i^{k}) \boldsymbol{\xi}_i^{k}\boldsymbol{w}_i$. By applying Lemma \ref{lem:1} where $D(\boldsymbol{x},\boldsymbol{y}) = \frac{1}{2}{\Vert \boldsymbol{x} - \boldsymbol{y} \Vert}^2 $, $s= 1/\eta_i^{k}$, and $L(\boldsymbol{x}) = \big\langle  f^{'}_i(\boldsymbol{\tilde{w}}_i^{k-1}),  \boldsymbol{x} - \boldsymbol{\tilde{w}}_i^{k-1} \big\rangle -\big\langle \boldsymbol{\gamma}_i^{k-1} ,\boldsymbol{x}- \boldsymbol{w}^{k-1} \big\rangle + \rho {\Vert \boldsymbol{x}- \boldsymbol{w}^{k-1}\Vert}^2/2 -(\rho+1/\eta_i^{k}) \boldsymbol{\xi}_i^{k}\boldsymbol{w}_i$, we have:
		\begin{equation}
		\begin{split}
		\big\langle f^{'}_i(\boldsymbol{\tilde{w}}_i^{k-1})-  \boldsymbol{\gamma}_i^{k-1}+ \rho(\boldsymbol{\tilde{w}}_i^{k} - \boldsymbol{w}^{k-1})   -(\rho+1/\eta_i^{k})\boldsymbol{\xi}_i^{k},  \boldsymbol{\tilde{w}}_i^{k} - \boldsymbol{w}_i \big\rangle  
		\leq   \frac{1}{2\eta_i^{k}}\big({\Vert \boldsymbol{w}_i - \boldsymbol{\tilde{w}}_i^{k-1} \Vert}^2 - {\Vert \boldsymbol{w}_i - \boldsymbol{\tilde{w}}_i^{k} \Vert}^2   -{\Vert \boldsymbol{\tilde{w}}_i^{k} - \boldsymbol{\tilde{w}}_i^{k-1} \Vert}^2 \big). \label{eq:2}
		\end{split}
		\end{equation}
		Lastly, based on Young's inequality, we have:
		\begin{equation}
		\begin{split}
		\big\langle f^{'}_i(\boldsymbol{\tilde{w}}_i^{k-1})-(\rho+1/\eta_i^{k}) \boldsymbol{\xi}_i^{k}, \boldsymbol{\tilde{w}}_i^{k-1} - \boldsymbol{\tilde{w}}_i^{k} \big\rangle 
		\leq  \frac{\eta_i^{k}}{2}{\big\Vert f^{'}_i(\boldsymbol{\tilde{w}}_i^{k-1})-(\rho+1/\eta_i^{k}) \boldsymbol{\xi}_i^{k}  \big\Vert}^2 + \frac{1}{2 \eta_i^{k}}{\Vert \boldsymbol{\tilde{w}}_i^{k} - \boldsymbol{\tilde{w}}_i^{k-1} \Vert}^2. \label{eq:4}
		\end{split}
		\end{equation}
		Combining \eqref{eq:3},\eqref{eq:1},\eqref{eq:2}, and \eqref{eq:4}, we have:
		\begin{equation}
		\begin{split}
		f_i(\boldsymbol{\tilde{w}}_i^{k-1})-f_i(\boldsymbol{w}_i)+ \big\langle \boldsymbol{\tilde{w}}_i^{k}- \boldsymbol{w}_i, - \boldsymbol{\gamma}_{i}^{k} \big\rangle 
		\leq & \frac{\eta_i^{k}}{2}{\Vert f^{'}_i(\boldsymbol{\tilde{w}}_i^{k-1})-(\rho+1/\eta_i^{k}) \boldsymbol{\xi}_i^{k}  \Vert}^2- \big(\rho+1/\eta_i^{k}\big)\big\langle  \boldsymbol{\xi}_i^{k}, \boldsymbol{w}_i - \boldsymbol{\tilde{w}}_i^{k-1} \big\rangle \\&+\frac{1}{2\eta_i^{k}}\big(  {\Vert \boldsymbol{w}_i - \boldsymbol{\tilde{w}}_i^{k-1} \Vert}^2-{\Vert \boldsymbol{w}_i - \boldsymbol{\tilde{w}}_i^{k} \Vert}^2\big) \\& + \frac{\rho}{2}({\Vert \boldsymbol{w}_i - \boldsymbol{w}^{k-1} \Vert}^2 - {\Vert \boldsymbol{w}_i - \boldsymbol{w}^{k} \Vert}^2)+ \frac{1}{2\rho}{\Vert \boldsymbol{\gamma}_{i}^{k}-\boldsymbol{\gamma}_i^{k-1} \Vert}^2. \label{eq:5}
		\end{split}
		\end{equation}
		Next, according to our algorithm where $\boldsymbol{\gamma}_i^{k} = \boldsymbol{\gamma}_{i}^{k-1} - \rho (\boldsymbol{\tilde{w}}_i^{k} - \boldsymbol{w}^{k})$ and $\boldsymbol{w}^{k} =  \frac{1}{n}\sum_{i=1}^n\boldsymbol{\tilde{w}}^{k}_i - \frac{1}{n}\sum_{i=1}^n\boldsymbol{\gamma}^{k-1}_i/\rho$, we have:
		\begin{equation}
		\begin{split}
		& \sum_{i=1}^{n} \big\langle \boldsymbol{w}^{k}- \boldsymbol{w} , \boldsymbol{\gamma}_{i}^{k} \big\rangle = 0. \label{eq:6}
		\end{split}
		\end{equation}
		And also, we could obtain:
		\begin{equation}
		\begin{split}
		\big\langle \boldsymbol{\gamma}_{i}^{k}-\boldsymbol{\gamma}_i  ,\boldsymbol{\tilde{w}}_i^{k}- \boldsymbol{w}^{k} \big\rangle 
		=  \frac{1}{\rho} \big\langle \boldsymbol{\gamma}_{i}^{k}-\boldsymbol{\gamma}_i ,\boldsymbol{\gamma}_i^{k-1}-\boldsymbol{\gamma}_{i}^{k} \big\rangle 
		=  \frac{1}{2\rho}\big({\Vert \boldsymbol{\gamma}_i-\boldsymbol{\gamma}_i^{k-1} \Vert}^2 - {\Vert \boldsymbol{\gamma}_i-\boldsymbol{\gamma}_{i}^{k} \Vert}^2 - {\Vert \boldsymbol{\gamma}_{i}^{k}-\boldsymbol{\gamma}_i^{k-1} \Vert}^2\big). \label{eq:7}
		\end{split}
		\end{equation}
		Thus, combining \eqref{eq:5}, \eqref{eq:6} and \eqref{eq:7}, we obtain the result in the Lemma 2:
		\begin{equation}
		\begin{split}
		&\sum_{i=1}^{n}\bigg( f_i(\boldsymbol{\tilde{w}}_i^{k-1})-f_i(\boldsymbol{w}_i) + {(\boldsymbol{u}_i^{k}-\boldsymbol{u}_i)}^{\intercal} F(\boldsymbol{u}_i^{k}) \bigg)\\
		=& \sum_{i=1}^n \bigg( f_i(\boldsymbol{\tilde{w}}_i^{k-1})-f_i(\boldsymbol{w}_i)+  \big\langle -\boldsymbol{\gamma}_{i}^{k},\boldsymbol{\tilde{w}}_i^{k}-\boldsymbol{w}_i \big\rangle  +\big\langle \boldsymbol{\gamma}_{i}^{k}, \boldsymbol{w}^{k}-\boldsymbol{w} \big\rangle +\big\langle \boldsymbol{\gamma}_{i}^{k}-\boldsymbol{\gamma}_i, \boldsymbol{\tilde{w}}_i^{k}-\boldsymbol{w}^{k} \big\rangle \bigg)\\
		\leq &  \sum_{i=1}^{n}\bigg(\frac{\eta_i^{k}}{2}{\Vert f_i^{'}(\boldsymbol{\tilde{w}}_i^{k-1})-(\rho+1/\eta_i^{k}) \boldsymbol{\xi}_i^{k}  \Vert}^2   -\big (\rho+1/\eta_i^{k}\big)\big\langle  \boldsymbol{\xi}_i^{k}, \boldsymbol{w}_i - \boldsymbol{\tilde{w}}_i^{k-1} \big\rangle  + \frac{\rho}{2}\big({\Vert \boldsymbol{w}_i - \boldsymbol{w}^{k-1} \Vert}^2- {\Vert \boldsymbol{w}_i - \boldsymbol{w}^{k} \Vert}^2  \big)   \\ & \quad \quad + \frac{1}{2\eta_i^{k}}\big({\Vert \boldsymbol{w}_i - \boldsymbol{\tilde{w}}_i^{k-1} \Vert}^2  - {\Vert \boldsymbol{w}_i - \boldsymbol{\tilde{w}}_i^{k} \Vert}^2 \big) +  \frac{1}{2\rho}{\Vert \boldsymbol{\gamma}_i-\boldsymbol{\gamma}_i^{k-1} \Vert}^2 - \frac{1}{2\rho}{\Vert \boldsymbol{\gamma}_i-\boldsymbol{\gamma}_{i}^{k} \Vert}^2\bigg). 
		\end{split}
		\end{equation}
	\end{proof}
	
	\section{Proof of Theorem \ref{the:2}}\label{ap:c}
	\begin{proof}
		
		According to the convexity of $f_i(\cdot)$ and the monotonicity of the operator $F(\cdot)$, and applying Lemma \ref{lem:2}, we have:
		\begin{equation}
		\begin{split}
		&  \sum_{i=1}^{n} \bigg(f_i(\boldsymbol{\bar{w}}_i^{t})-f_i(\boldsymbol{w}_i)+  {(\boldsymbol{\bar{u}}_i^{t}-\boldsymbol{u}_i)}^{\intercal} F(\boldsymbol{\bar{u}}_i^{t})\bigg) \\
		= &  \sum_{i=1}^{n}\bigg( f_i(\boldsymbol{\bar{w}}_i^{t})-f_i(\boldsymbol{w}_i)+  \big\langle -\boldsymbol{\bar{\gamma}}_i^{t},\boldsymbol{\bar{w}}_i^{t}-\boldsymbol{w}_i \big\rangle + \big\langle \boldsymbol{\bar{\gamma}}_i^{t}, \boldsymbol{\bar{w}}^{t}-\boldsymbol{w} \big\rangle +\big\langle \boldsymbol{\bar{\gamma}}_i^{t}-\boldsymbol{\gamma}_i, \boldsymbol{\bar{w}}_i^{t}-\boldsymbol{\bar{w}}^{t} \big\rangle \bigg)\\
		\leq & \frac{1}{t} \sum_{k=1}^{t}  \sum_{i=1}^{n} \bigg(f_i(\boldsymbol{\tilde{w}}_i^{k-1})-f_i(\boldsymbol{w}_i)+  {(\boldsymbol{u}_i^{k}-\boldsymbol{u}_i)}^{\intercal} F(\boldsymbol{u}_i^{k})\bigg)\\ 
		= & \frac{1}{t} \sum_{k=1}^{t}  \sum_{i=1}^{n} \bigg(f_i(\boldsymbol{\tilde{w}}_i^{k-1})-f_i(\boldsymbol{w}_i)+  \big\langle -\boldsymbol{\gamma}_{i}^{k},\boldsymbol{\tilde{w}}_i^{k}-\boldsymbol{w}_i \big\rangle   + \big\langle \boldsymbol{\gamma}_{i}^{k}, \boldsymbol{w}^{k}-\boldsymbol{w} \big\rangle +\big\langle \boldsymbol{\gamma}_{i}^{k}-\boldsymbol{\gamma}_i, \boldsymbol{\tilde{w}}_i^{k}-\boldsymbol{w}^{k} \big\rangle\bigg)\\
		\leq &  \sum_{i=1}^{n}\frac{1}{t} \sum_{k=1}^{t} \bigg(\frac{\eta_i^{k}}{2}{\big\Vert f^{'}_i(\boldsymbol{\tilde{w}}_i^{k})-(\rho+1/\eta_i^{k}) \boldsymbol{\xi}_i^{k}  \big\Vert}^2  - \big(\rho+1/\eta_i^{k}\big)\big\langle  \boldsymbol{\xi}_i^{k}, \boldsymbol{w}_i- \boldsymbol{\tilde{w}}_i^{k-1} \big\rangle \bigg) \\& + \frac{1}{t} \sum_{i=1}^{n} \bigg(\frac{1}{2\eta^{t}}{\big\Vert \boldsymbol{w}_i - \boldsymbol{\tilde{w}}_i^{0} \big\Vert}^2 +\frac{\rho}{2}{\Vert \boldsymbol{w}_i- \boldsymbol{w}^{0} \Vert}^2  +\frac{1}{2\rho}{\Vert \boldsymbol{\gamma}_i-\boldsymbol{\gamma}_i^{0} \Vert}^2  \bigg).
		\end{split}
		\end{equation}
		Let $(\boldsymbol{w}_i,\boldsymbol{w})$ be the optimal solution $(\boldsymbol{w}_i^{*},\boldsymbol{w}^{*})$ in the above inequality. We get:
		\begin{equation}
		\begin{split}
		&  \sum_{i=1}^{n}\bigg( f_i(\boldsymbol{\bar{w}}_i^{t})-f_i(\boldsymbol{w}^{*}_i)+  \big\langle -\boldsymbol{\bar{\gamma}}_i^{t},\boldsymbol{\bar{w}}_i^{t}-\boldsymbol{w}_i^{*} \big\rangle + \big\langle \boldsymbol{\bar{\gamma}}_i^{t}, \boldsymbol{\bar{w}}^{t}-\boldsymbol{w}^{*} \big\rangle +\big\langle \boldsymbol{\bar{\gamma}}_i^{t}-\boldsymbol{\gamma}_i, \boldsymbol{\bar{w}}_i^{t}-\boldsymbol{\bar{w}}^{t} \big\rangle \bigg)  \\ 
		%	& \sum_{i=1}^{n} \bigg(f_i(\boldsymbol{\bar{w}}_i^{t})-f_i(\boldsymbol{w}^{*}_i)+  {(\boldsymbol{\bar{u}}_i^{t}-\boldsymbol{u}_i)}^T F(\boldsymbol{\bar{u}}_i^{t})\bigg)\\
		%  \leq & \frac{1}{t} \sum_{k=1}^{t}  \sum_{i=1}^{n} [f_i(\boldsymbol{\tilde{w}}_i^{k})-f_i(\boldsymbol{w}_i)+ \big\langle -\boldsymbol{\gamma}_{i}^{k},\boldsymbol{\tilde{w}}_i^{k}-\boldsymbol{w}_i^* \big\rangle \\& + \big\langle \boldsymbol{\gamma}_{i}^{k}, \boldsymbol{w}^{k}-\boldsymbol{w}^* \big\rangle +\big\langle \boldsymbol{\gamma}_{i}^{k}-\boldsymbol{\gamma}_i^*, \boldsymbol{\tilde{w}}_i^{k}-\boldsymbol{w}^{k} \big\rangle]\\
		%		\leq &  \sum_{i=1}^{n}\frac{1}{t} \sum_{k=1}^{t} \bigg(\frac{\eta_i^{k}}{2}{\big\Vert f^{'}_i(\boldsymbol{\tilde{w}}_i^{k})-(\rho+1/\eta_i^{k}) \boldsymbol{\xi}_i^{k}  \big\Vert}^2 \\& - \big(\rho+1/\eta_i^{k}\big)\big\langle  \boldsymbol{\xi}_i^{k}, \boldsymbol{w}_i^{*}- \boldsymbol{\tilde{w}}_i^{k-1} \big\rangle \bigg) + \frac{1}{t} \sum_{i=1}^{n} \bigg(\frac{1}{2\eta^{t}}{\big\Vert \boldsymbol{w}_i^{*} - \boldsymbol{\tilde{w}}_i^{0} \big\Vert}^2  \\&+\frac{\rho}{2}{\Vert \boldsymbol{w}_i^{*} - \boldsymbol{w}^{0} \Vert}^2 \\& \quad \quad \quad \quad \quad  +\frac{1}{2\rho}{\Vert \boldsymbol{\gamma}_i-\boldsymbol{\gamma}_i^{0} \Vert}^2  \bigg)\\
		\leq &  \sum_{i=1}^{n} \frac{1}{t} \sum_{k=1}^{t} \frac{\eta_i^{k}}{2}{\big\Vert f^{'}_i(\boldsymbol{\tilde{w}}_i^{k-1})-(\rho+1/\eta_i^{k}) \boldsymbol{\xi}_i^{k}  \big\Vert}^2 -   \sum_{i=1}^{n}\frac{1}{t} \sum_{k=1}^{t} \big (\rho+1/\eta_i^{k}\big)  \big\langle\boldsymbol{\xi}_i^{k}, \boldsymbol{w}_i^{*} - \boldsymbol{\tilde{w}}_i^{k-1} \big\rangle \\& +\frac{1}{t}\sum_{i=1}^n \frac{c_w^2}{2\eta_i^t}+ \frac{n}{t}\frac{\rho}{2}c_w^2+ \frac{1}{t}  \sum_{i=1}^{n} \frac{1}{2\rho}{\Vert \boldsymbol{\gamma}_i-\boldsymbol{\gamma}_i^{0} \Vert}^2. 
		\end{split}
		\end{equation}
		The above inequality holds for all $\boldsymbol{\gamma}_i$, thus it also holds for $\boldsymbol{\gamma}_i \in \{ \boldsymbol{\gamma}_i: \Vert \boldsymbol{\gamma}_i \Vert \leq \beta \}$. By letting $ \boldsymbol{\gamma}_i$ be the optimal solution, we have the maximum of the left side of the above inequality:
		\begin{equation}
		\begin{split}
		& \max_{ \{ \boldsymbol{\gamma}_i: \Vert \boldsymbol{\gamma}_i \Vert \leq \beta \} }  \sum_{i=1}^{n}\bigg(f_i(\boldsymbol{\bar{w}}_i^{t})-f_i(\boldsymbol{w}^{*}_i)+  \big\langle -\boldsymbol{\bar{\gamma}}_i^{t},\boldsymbol{\bar{w}}_i^{t}-\boldsymbol{w}_i^{*} \big\rangle  + \big\langle \boldsymbol{\bar{\gamma}}_i^{t}, \boldsymbol{\bar{w}}^{t}-\boldsymbol{w}^{*} \big\rangle +\big\langle \boldsymbol{\bar{\gamma}}_i^{t}-\boldsymbol{\gamma}_i, \boldsymbol{\bar{w}}_i^{t}-\boldsymbol{\bar{w}}^{t} \big\rangle \bigg)  \\ 
		= & \max_{ \{ \boldsymbol{\gamma}_i: \Vert \boldsymbol{\gamma}_i \Vert \leq \beta \} }  \sum_{i=1}^{n}\bigg( f_i(\boldsymbol{\bar{w}}_i^{t})-f_i(\boldsymbol{w}_i) - \boldsymbol{\gamma}_i(\boldsymbol{\bar{w}}_i^{t}- \boldsymbol{\bar{w}}^t)\bigg) \\
		= &\quad  \sum_{i=1}^{n}\bigg( f_i(\boldsymbol{\bar{w}}_i^{t})-f_i(\boldsymbol{w}_i) - \max_{ \{ \boldsymbol{\gamma}_i: \Vert \boldsymbol{\gamma}_i \Vert \leq \beta \} }  \boldsymbol{\gamma}_i(\boldsymbol{\bar{w}}_i^{t}- \boldsymbol{\bar{w}}^t)\bigg) \\
		= & \quad  \sum_{i=1}^{n}\bigg( f_i(\boldsymbol{\bar{w}}_i^{t})-f_i(\boldsymbol{w}_i) + \beta (\Vert \boldsymbol{\bar{w}}_i^{t}- \boldsymbol{\bar{w}}^t\Vert)\bigg). 
		\end{split}
		\end{equation}
		And we also get the maximum of the right side:
		\begin{equation}
		\begin{split}
		&\sum_{i=1}^{n} \frac{1}{t} \sum_{k=1}^{t} \frac{\eta_i^{k}}{2}{\big\Vert f^{'}_i(\boldsymbol{\tilde{w}}_i^{k-1})-(\rho+1/\eta_i^{k}) \boldsymbol{\xi}_i^{k}  \big\Vert}^2 -  \sum_{i=1}^{n}\frac{1}{t} \sum_{k=1}^{t} \big (\rho+1/\eta_i^{k}\big)  \big\langle\boldsymbol{\xi}_i^{k}, \boldsymbol{w}_i^{*} - \boldsymbol{\tilde{w}}_i^{k-1} \big\rangle \\
		& + \frac{1}{t}\sum_{i=1}^n \frac{c_w^2}{2\eta_i^t}+\frac{\rho n}{2 t}c_w^2 + \max_{ \{ \boldsymbol{\gamma}_i: \Vert \boldsymbol{\gamma}_i \Vert \leq \beta \} } \frac{1}{t}  \sum_{i=1}^{n} \frac{1}{2\rho}{\Vert \boldsymbol{\gamma}_i-\boldsymbol{\gamma}_i^{0} \Vert}^2 \\
		= &  \sum_{i=1}^{n} \frac{1}{t} \sum_{k=1}^{t} \frac{\eta_i^{k}}{2}{\big\Vert f^{'}_i(\boldsymbol{\tilde{w}}_i^{k-1})-(\rho+1/\eta_i^{k}) \boldsymbol{\xi}_i^{k}  \big\Vert}^2 -   \sum_{i=1}^{n}\frac{1}{t} \sum_{k=1}^{t} \big (\rho+1/\eta_i^{k}\big)  \big\langle\boldsymbol{\xi}_i^{k}, \boldsymbol{w}_i^{*} - \boldsymbol{\tilde{w}}_i^{k-1} \big\rangle \\ & + \frac{1}{t}\sum_{i=1}^n \frac{c_w^2}{2\eta_i^t}+\frac{\rho n}{2 t}c_w^2+ \frac{n}{t} \frac{\beta^2}{2\rho}. 	        
		\end{split}
		\end{equation}
		Thus, we obtain the inequality:
		\begin{equation}
		\begin{split}
		&  \sum_{i=1}^{n}\bigg( f_i(\boldsymbol{\bar{w}}_i^{t})-f_i(\boldsymbol{w}_i) + \beta \Vert \boldsymbol{\bar{w}}_i^{t}- \boldsymbol{\bar{w}}^t\Vert \bigg) \\
		\leq &   \sum_{i=1}^{n} \frac{1}{t} \sum_{k=1}^{t} \frac{\eta_i^{k}}{2}{\big\Vert f^{'}_i(\boldsymbol{\tilde{w}}_i^{k-1})-(\rho+1/\eta_i^{k}) \boldsymbol{\xi}_i^{k}  \big\Vert}^2 -   \sum_{i=1}^{n}\frac{1}{t} \sum_{k=1}^{t} \big (\rho+1/\eta_i^{k}\big)  \big\langle\boldsymbol{\xi}_i^{k}, \boldsymbol{w}_i^{*} - \boldsymbol{\tilde{w}}_i^{k-1} \big\rangle \\ & + \frac{1}{t}\sum_{i=1}^n \frac{c_w^2}{2\eta_i^t}+\frac{\rho n}{2 t}c_w^2+ \frac{n}{t} \frac{\beta^2}{2\rho}. \label{eq:8}
		\end{split}
		\end{equation}
		Since we assume $\Vert \ell^{'}(\cdot) \Vert \leq c_1$ and $\Vert R^{'}(\cdot) \Vert \leq c_2$, we have $\mathbb{E}\big[\big\Vert f^{'}_i(\boldsymbol{\tilde{w}}_i^{k-1})-(\rho+1/\eta_i^{k}) \boldsymbol{\xi}_i^{k}  \big\Vert^2\big]=(c_1+\lambda c_2/n)^2 + 8dp c_1^2\ln{(1.25/\delta)}/\big(m_i^2\epsilon^2\big)$.
		%	\begin{equation}
		%	\begin{split}
		%	&\mathbb{E}\bigg[\big\Vert f^{'}_i(\boldsymbol{\tilde{w}}_i^{k-1})-(\rho+1/\eta_i^{k}) \boldsymbol{\xi}_i^{k}  \big\Vert^2\bigg] \\=&\mathbb{E}\bigg[\big\Vert f^{'}_i(\boldsymbol{\tilde{w}}_i^{k-1}) \big\Vert^2\bigg] +\mathbb{E}\bigg[ \big\Vert(\rho+1/\eta_i^{k}) \boldsymbol{\xi}_i^{k}  \big\Vert^2 \bigg]\\ = & \quad \big\Vert f^{'}_i(\boldsymbol{\tilde{w}}_i^{k-1}) \big\Vert^2 + d(\rho+1/\eta_i^{k})^2 \sigma^2_{i,k+1}\\ \leq &   \quad (c_1+\lambda c_2/n)^2 + \frac{8dc_1^2\ln{(1.25/\delta)}}{m_i^2\epsilon^2}.
		%	\end{split}
		%	\end{equation}
		With $\mathbb{E}\big[\big\langle  \boldsymbol{\xi}_i^{k}, \boldsymbol{w}_i^{*} - \boldsymbol{\tilde{w}}_i^{k-1} \big\rangle \big] = 0$ and  $\eta_i^{k} = c_w\big( 2k(c_1+\lambda c_2/n)^2+16kdpc_1^2\ln{(1.25/\delta)}/\big(m_i^2 \epsilon^2\big)  \big)^{-\frac{1}{2}}$, by taking expectation of the inequality \eqref{eq:8}, we obtain:
		\begin{equation}
		\begin{split}
		&  \mathbb{E}\bigg[ \sum_{i=1}^{n} \big(f_i(\boldsymbol{\bar{w}}_i^{t})-f_i(\boldsymbol{w}_i^*) + \beta \Vert \boldsymbol{\bar{w}}_i^{t}- \boldsymbol{\bar{w}}^t\Vert\big) \bigg] \\
		\leq & \sum_{i=1}^{n}\frac{1}{t}\sum_{k=1}^{t} \mathbb{E}\bigg[ \frac{\eta_i^{k}}{2}{\Vert f^{'}_i(\boldsymbol{\tilde{w}}_i^{k-1})-(\rho+1/\eta_i^{k}) \boldsymbol{\xi}_i^{k}  \Vert}^2\bigg] - \sum_{i=1}^{n}\frac{1}{t} \sum_{k=1}^{t} \big(\rho+1/\eta_i^{k}\big)\mathbb{E}\bigg[ \big\langle \boldsymbol{\xi}_i^{k}, \boldsymbol{w}_i^{*} - \boldsymbol{\tilde{w}}_i^{k-1} \big\rangle\bigg] \\&  +  \frac{1}{t}\sum_{i=1}^n \frac{c_w^2}{2\eta_i^t}+\frac{\rho n}{2 t}c_w^2+\frac{n}{t} \frac{\beta^2}{2\rho},
		\end{split}
		\end{equation}
		which leads to the result in the theorem:
		\begin{equation}
		\begin{split}
		&  \mathbb{E}\bigg[ \sum_{i=1}^{n} \big(f_i(\boldsymbol{\bar{w}}_i^{t})-f_i(\boldsymbol{w}_i^*) + \beta \Vert \boldsymbol{\bar{w}}_i^{t}- \boldsymbol{\bar{w}}^t\Vert\big) \bigg] \\
		= & \sum_{i=1}^n\frac{1}{t}\sum_{k=1}^{t}\frac{c_w}{2\sqrt{2k}}\sqrt{ (c_1+\lambda c_2/n)^2+\frac{8dpc_1^2\ln{(1.25/\delta)}}{m_i^2 \epsilon^2} }\\&+\sum_{i=1}^n \frac{1}{t}\sum_{k=1}^{t}\frac{c_w\sqrt{2t}}{2}\sqrt{ (c_1+\lambda c_2/n)^2+\frac{8dpc_1^2\ln{(1.25/\delta)}}{m_i^2 \epsilon^2} } +\frac{n \rho }{2t}c_w^2 +\frac{n\beta^2}{2\rho t} \\
		=& \sum_{i=1}^n \frac{ c_w}{ 2\sqrt{2} t}\sqrt{ (c_1+\lambda c_2/n)^2+\frac{8dpc_1^2\ln{(1.25/\delta)}}{m_i^2 \epsilon^2} }\bigg(\sum_{k=1}^{t}\frac{1}{\sqrt{k}}+ 2\sqrt{t}\bigg)+ \frac{ n \rho  }{2t}c_w^2+\frac{n\beta^2}{2\rho t}\\
		\leq & \sum_{i=1}^n \frac{ \sqrt{2} c_w }{ \sqrt{t} }\sqrt{ (c_1+\lambda c_2/n)^2+\frac{8dpc_1^2\ln{(1.25/\delta)}}{m_i^2 \epsilon^2} } +\frac{n(\rho c_w^2 + \beta^2/\rho)}{2t}. 
		\end{split}
		\end{equation}
	\end{proof}

	\section{Proof of Lemma \ref{lem:3}}\label{ap:d}
	\begin{proof}
		As we assume that $\ell(\cdot) $ and $R(\cdot)$ are smooth and convex, $\big\Vert \nabla^2 \ell(\cdot) \big\Vert \leq c_3$, and $\big\Vert \nabla^2 R(\cdot) \big\Vert \leq c_4$, thus we have $\big\Vert \nabla^2 f_i(\cdot)  \big\Vert = \big\Vert \nabla^2 \ell(\cdot)+\lambda/n\nabla^2 R(\cdot) \big\Vert \leq c_3+\lambda c_4/n $ is bounded. This leads to:
		\begin{equation}
		\big\Vert \nabla f_i(\boldsymbol{x}) - \nabla f_i(\boldsymbol{y}) \big\Vert \leq (c_3+\lambda c_4/n) \big\Vert \boldsymbol{x} - \boldsymbol{y} \big\Vert. 
		\end{equation}
		Thus, $f_i(\cdot)$ is $(c_3+\lambda c_4/n)$-Lipschitz smooth. According to the property of Lipschitz smooth, we have:
		\begin{equation}
		\begin{split}
		f_i(\boldsymbol{\tilde{w}}_i^{k}) \leq & f_i(\boldsymbol{\tilde{w}}_i^{k-1}) + \big\langle \nabla f_i(\boldsymbol{\tilde{w}}_i^{k-1}), \boldsymbol{\tilde{w}}_i^{k} - \boldsymbol{\tilde{w}}_i^{k-1} \big\rangle + \frac{c_3+\lambda c_4/n}{2}{\Vert \boldsymbol{\tilde{w}}_i^{k} - \boldsymbol{\tilde{w}}_i^{k-1} \Vert}^2 \\
		=& f_i  (\boldsymbol{\tilde{w}}_i^{k-1})+ \big(\rho+1/\eta_i^{k}\big)\big\langle  \boldsymbol{\xi}_i^{k} , \boldsymbol{\tilde{w}}_i^{k} - \boldsymbol{\tilde{w}}_i^{k-1} \big\rangle  +  \frac{c_3+\lambda c_4/n}{2}{\Vert \boldsymbol{\tilde{w}}_i^{k}  - \boldsymbol{\tilde{w}}_i^{k-1} \Vert}^2 \\& +   \big\langle \nabla f_i(\boldsymbol{\tilde{w}}_i^{k-1})-(\rho+1/\eta_i^{k}) \boldsymbol{\xi}_i^{k}, \boldsymbol{\tilde{w}}_i^{k} - \boldsymbol{\tilde{w}}_i^{k-1} \big\rangle. \label{eq:9}
		\end{split}
		\end{equation}
		Due to the convexity of $f_i(\cdot)$, we have:
		\begin{equation}
		f_i(\boldsymbol{\tilde{w}}_i^{k})-f_i(\boldsymbol{w}_i) \leq  \big\langle \nabla f_i(\boldsymbol{\tilde{w}}_i^{k}), \boldsymbol{\tilde{w}}_i^{k} - \boldsymbol{w}_i \big\rangle. \label{eq:10}
		\end{equation}
		According to \eqref{eq:9} and \eqref{eq:10}, we have:
		\begin{equation}
		\begin{split}
		f_i(\boldsymbol{\tilde{w}}_i^{k})-f_i(\boldsymbol{w}_i)+ \big\langle \boldsymbol{\tilde{w}}_i^{k}- \boldsymbol{w}_i, - \boldsymbol{\gamma}_{i}^{k} \big\rangle
		\leq &f_i(\boldsymbol{\tilde{w}}_i^{k-1})-f_i(\boldsymbol{w}_i)+ \big(\rho+1/\eta_i^{k}\big)\big\langle  \boldsymbol{\xi}_i^{k}, \boldsymbol{\tilde{w}}_i^{k} - \boldsymbol{\tilde{w}}_i^{k-1} \big\rangle \\& + \big\langle \nabla f_i(\boldsymbol{\tilde{w}}_i^{k-1})-(\rho+1/\eta_i^{k}) \boldsymbol{\xi}_i^{k}, \boldsymbol{\tilde{w}}_i^{k} - \boldsymbol{\tilde{w}}_i^{k-1} \big\rangle  \\&+ \frac{c_3+\lambda c_4/n}{2}{\Vert \boldsymbol{\tilde{w}}_i^{k} - \boldsymbol{\tilde{w}}_i^{k-1} \Vert}^2 +  \big\langle \boldsymbol{\tilde{w}}_i^{k}- \boldsymbol{w}_i, - \boldsymbol{\gamma}_{i}^{k} \big\rangle,
		\end{split}
		\end{equation}
		which leads to:
		\begin{equation}
		\begin{split}
		f_i(\boldsymbol{\tilde{w}}_i^{k})-f_i(\boldsymbol{w}_i)+ \big\langle \boldsymbol{\tilde{w}}_i^{k}- \boldsymbol{w}_i, - \boldsymbol{\gamma}_{i}^{k} \big\rangle
		%		\leq&  \big\langle \nabla f_i(\boldsymbol{\tilde{w}}_i^{k-1}), \boldsymbol{\tilde{w}}_i^{k-1} - \boldsymbol{w}_i \big\rangle + \big\langle \boldsymbol{\tilde{w}}_i^{k}- \boldsymbol{w}_i, - \boldsymbol{\gamma}_{i}^{k} \big\rangle \\& +  \big\langle \nabla f_i(\boldsymbol{\tilde{w}}_i^{k-1})-(\rho+1/\eta_i^{k}) \boldsymbol{\xi}_i^{k}, \boldsymbol{w}_i - \boldsymbol{\tilde{w}}_i^{k-1} \big\rangle \\&  +\big( \rho+1/\eta_i^{k}\big)\big\langle \boldsymbol{\xi}_i^{k}, \boldsymbol{\tilde{w}}_i^{k} - \boldsymbol{\tilde{w}}_i^{k-1} \big\rangle  + \frac{c_3+\lambda c_4/n}{2}{\Vert \boldsymbol{\tilde{w}}_i^{k} - \boldsymbol{\tilde{w}}_i^{k-1} \Vert}^2 \\& + \big\langle \nabla f_i(\boldsymbol{\tilde{w}}_i^{k})-(\rho+1/\eta_i^{k}) \boldsymbol{\xi}_i^{k}, \boldsymbol{\tilde{w}}_i^{k} - \boldsymbol{w}_i \big\rangle  \\
		%	= &-\big\langle (\rho+1/\eta_i^{k}) \boldsymbol{\xi}_i^{k}, \boldsymbol{w}_i - \boldsymbol{\tilde{w}}_i^{k} \big\rangle \\&+ \frac{c_3+\lambda c_4/n}{2}{\Vert \boldsymbol{\tilde{w}}_i^{k} - \boldsymbol{\tilde{w}}_i^{k-1} \Vert}^2 \\&+ \big\langle \nabla f_i(\boldsymbol{\tilde{w}}_i^{k})-(\rho+1/\eta_i^{k}) \boldsymbol{\xi}_i^{k}, \boldsymbol{\tilde{w}}_i^{k} - \boldsymbol{w}_i \big\rangle \\ & + \big\langle \boldsymbol{\tilde{w}}_i^{k}- \boldsymbol{w}_i, - \boldsymbol{\gamma}_{i}^{k} \big\rangle \\
		\leq&  \big\langle \nabla f_i(\boldsymbol{\tilde{w}}_i^{k-1})-\boldsymbol{\gamma}_i^{k}-(\rho+1/\eta_i^{k}) \boldsymbol{\xi}_i^{k}  +  \rho( \boldsymbol{\tilde{w}}_i^{k} - \boldsymbol{w}^{k-1}) , \boldsymbol{\tilde{w}}_i^{k} - \boldsymbol{w}_i \big\rangle \\ &+ \big(\rho+1/\eta_i^{k}\big)\big\langle  \boldsymbol{\xi}_i^{k}, \boldsymbol{\tilde{w}}_i^{k} - \boldsymbol{\tilde{w}}_i^{k-1} \big\rangle+ \frac{c_3+\lambda c_4/n}{2}{\Vert \boldsymbol{\tilde{w}}_i^{k} - \boldsymbol{\tilde{w}}_i^{k-1} \Vert}^2 \\ & +\big\langle \boldsymbol{\tilde{w}}_i^{k} - \boldsymbol{w}_i, \rho(\boldsymbol{w}^{k-1} - \boldsymbol{w}^{k}) \big\rangle - \big(\rho+1/\eta_i^{k}\big) \big\langle  \boldsymbol{\xi}_i^{k},  \boldsymbol{w}_i - \boldsymbol{\tilde{w}}_i^{k-1} \big\rangle   . \label{eq:11}
		\end{split}
		\end{equation}
		Based on Young's inequality,
		\begin{equation}
		\begin{split}
		\big\langle(\rho+1/\eta_i^{k}) \boldsymbol{\xi}_i^{k}, \boldsymbol{\tilde{w}}_i^{k} - \boldsymbol{\tilde{w}}_i^{k-1} \big\rangle 
		\leq  \frac{1}{2(1/\eta_i^{k}-(c_3+\lambda c_4/n))}{\big\Vert (\rho+1/\eta_i^{k}) \boldsymbol{\xi}_i^{k}  \big\Vert}^2 + \frac{1/\eta_i^{k}-(c_3+\lambda c_4/n)}{2}{\big\Vert \boldsymbol{\tilde{w}}_i^{k} - \boldsymbol{\tilde{w}}_i^{k-1} \big\Vert}^2. \label{eq:12}
		\end{split}
		\end{equation}
		Combining \eqref{eq:1}, \eqref{eq:2}, \eqref{eq:11} and \eqref{eq:12}, we have:
		\begin{equation}
		\begin{split}
		f_i(\boldsymbol{\tilde{w}}_i^{k})-f_i(\boldsymbol{w}_i)+ \big\langle \boldsymbol{\tilde{w}}_i^{k}- \boldsymbol{w}_i, - \boldsymbol{\gamma}_{i}^{k} \big\rangle 
		\leq &\frac{(\rho+1/\eta_i^{k}) ^2}{2(1/\eta_i^{k}-(c_3+\lambda c_4/n))}{\big\Vert \boldsymbol{\xi}_i^{k}  \big\Vert}^2  -\big(\rho+1/\eta_i^{k}\big)\big\langle  \boldsymbol{\xi}_i^{k}, \boldsymbol{w}_i - \boldsymbol{\tilde{w}}_i^{k-1} \big\rangle \\
		&+ \frac{1}{2\eta_i^{k}}({\Vert \boldsymbol{w}_i - \boldsymbol{\tilde{w}}_i^{k-1} \Vert}^2 - {\Vert \boldsymbol{w}_i - \boldsymbol{\tilde{w}}_i^{k} \Vert}^2)\\& + \frac{\rho}{2}({\Vert \boldsymbol{w}_i - \boldsymbol{w}^{k-1} \Vert}^2 - {\Vert \boldsymbol{w}_i - \boldsymbol{w}^{k} \Vert}^2)+ \frac{1}{2\rho}{\Vert \boldsymbol{\gamma}_{i}^{k}-\boldsymbol{\gamma}_i^{k-1} \Vert}^2. \label{eq:13}
		\end{split}
		\end{equation}
		Combining \eqref{eq:13}, \eqref{eq:6} and \eqref{eq:7}, we get the result as desired:
		\begin{equation}
		\begin{split}
		&\sum_{i=1}^{n}  \bigg(f_i(\boldsymbol{\tilde{w}}_i^{k})-f_i(\boldsymbol{w}_i)+ {(\boldsymbol{u}_i^{k}-\boldsymbol{u}_i)}^{\intercal} F(\boldsymbol{u}_i^{k})\bigg)\\
		=&\frac{1}{n} \sum_{i=1}^n \bigg( f_i(\boldsymbol{\tilde{w}}_i^{k})-f_i(\boldsymbol{w}_i)+  \big\langle -\boldsymbol{\gamma}_{i}^{k},\boldsymbol{\tilde{w}}_i^{k}-\boldsymbol{w}_i \big\rangle + \big\langle \boldsymbol{\gamma}_{i}^{k}, \boldsymbol{w}^{k}-\boldsymbol{w} \big\rangle +\big\langle \boldsymbol{\gamma}_{i}^{k}-\boldsymbol{\gamma}_i, \boldsymbol{\tilde{w}}_i^{k}-\boldsymbol{w}^{k} \big\rangle \bigg)\\
		\leq &\sum_{i=1}^{n}\bigg( \frac{\big(\rho+1/\eta_i^{k}\big)^2 }{2(1/\eta_i^{k}-(c_3+\lambda c_4/n))}{\big\Vert \boldsymbol{\xi}_i^{k} \big \Vert}^2 - \big(\rho+1/\eta_i^{k}\big)\big\langle  \boldsymbol{\xi}_i^{k}, \boldsymbol{w}_i - \boldsymbol{\tilde{w}}_i^{k-1} \big\rangle +  \frac{1}{2\eta_i^{k}}({\Vert \boldsymbol{w}_i - \boldsymbol{\tilde{w}}_i^{k-1} \Vert}^2 - {\Vert \boldsymbol{w}_i - \boldsymbol{\tilde{w}}_i^{k} \Vert}^2) \\ & \quad \quad +\frac{\rho}{2}({\Vert \boldsymbol{w}_i - \boldsymbol{w}^{k-1} \Vert}^2 - {\Vert \boldsymbol{w}_i - \boldsymbol{w}^{k} \Vert}^2)  + \frac{1}{2\rho}({\Vert \boldsymbol{\gamma}_i-\boldsymbol{\gamma}_i^{k-1} \Vert}^2 - {\Vert \boldsymbol{\gamma}_i-\boldsymbol{\gamma}_{i}^{k} \Vert}^2)\bigg). 
		\end{split}
		\end{equation}
	\end{proof}

	\section{Proof of Theorem \ref{the:3}}\label{ap:e}
	\begin{proof}
		
		According to the convexity of $f_i(\cdot)$ and the monotonicity of $F(\cdot)$, and applying Lemma \ref{lem:3}, we have:
		\begin{equation}
		\begin{split}
		&  \sum_{i=1}^{n} \bigg(f_i(\boldsymbol{\bar{w}}_i^{t})-f_i(\boldsymbol{w}_i)+  {(\boldsymbol{\bar{u}}_i^{t}-\boldsymbol{u}_i)}^{\intercal} F(\boldsymbol{\bar{u}}_i^{t})\bigg) \\
		=  & \sum_{i=1}^{n}\bigg(f_i(\boldsymbol{\bar{w}}_i^{t})-f_i(\boldsymbol{w}_i)+  \big\langle -\boldsymbol{\bar{\gamma}}_i^{t},\boldsymbol{\bar{w}}_i^{t}-\boldsymbol{w}_i \big\rangle+ \big\langle \boldsymbol{\bar{\gamma}}_i^{t}, \boldsymbol{\bar{w}}^{t}-\boldsymbol{w} \big\rangle +\big\langle \boldsymbol{\bar{\gamma}}_i^{t}-\boldsymbol{\gamma}_i, \boldsymbol{\bar{w}}_i^{t}-\boldsymbol{\bar{w}}^{t} \big\rangle\bigg)\\
		\leq & \frac{1}{t} \sum_{k=1}^{t}  \sum_{i=1}^{n} \bigg( f_i(\boldsymbol{\tilde{w}}_i^{k})-f_i(\boldsymbol{w}_i)+   {(\boldsymbol{u}_i^{k}-\boldsymbol{u}_i)}^{\intercal} F(\boldsymbol{u}_i^{k})\bigg) \\
		= & \frac{1}{t} \sum_{k=1}^{t}  \sum_{i=1}^{n} \bigg(f_i(\boldsymbol{\tilde{w}}_i^{k})-f_i(\boldsymbol{w}_i)+  \big\langle -\boldsymbol{\gamma}_{i}^{k},\boldsymbol{\tilde{w}}_i^{k}-\boldsymbol{w}_i \big\rangle   + \big\langle \boldsymbol{\gamma}_{i}^{k}, \boldsymbol{w}^{k}-\boldsymbol{w} \big\rangle +\big\langle \boldsymbol{\gamma}_{i}^{k}-\boldsymbol{\gamma}_i, \boldsymbol{\tilde{w}}_i^{k}-\boldsymbol{w}^{k} \big\rangle\bigg) \\ 
		\leq & \sum_{i=1}^{n}\frac{1}{t} \sum_{k=1}^{t} \bigg(\frac{\big(\rho+1/\eta_i^{k}\big)^2}{2(1/\eta_i^{k}-(c_3+\lambda c_4/n))} {\big\Vert  \boldsymbol{\xi}_i^{k}  \big\Vert}^2-\big (\rho+1/\eta_i^{k}\big)  \big\langle\boldsymbol{\xi}_i^{k}, \boldsymbol{w}_i - \boldsymbol{\tilde{w}}_i^{k-1} \big\rangle \bigg)\\& + \frac{1}{t} \sum_{i=1}^{n} \bigg(\frac{1}{2\eta_i^{t}}{\Vert \boldsymbol{w}_i - \boldsymbol{\tilde{w}}_i^{0} \Vert}^2+\frac{\rho}{2}{\Vert \boldsymbol{w}_i - \boldsymbol{w}^{0} \Vert}^2 +\frac{1}{2\rho}{\Vert \boldsymbol{\gamma}_i-\boldsymbol{\gamma}_i^{0} \Vert}^2 \bigg ).
		\end{split}
		\end{equation}
		By letting $(\boldsymbol{w}_i,\boldsymbol{w})$ be the optimal solution $(\boldsymbol{w}_i^{*},\boldsymbol{w}^{*})$, we have:
		\begin{equation}
		\begin{split}
		&  \sum_{i=1}^{n}\bigg(f_i(\boldsymbol{\bar{w}}_i^{t})-f_i(\boldsymbol{w}^{*}_i)+  \big\langle -\boldsymbol{\bar{\gamma}}_i^{t},\boldsymbol{\bar{w}}_i^{t}-\boldsymbol{w}_i^{*} \big\rangle + \big\langle \boldsymbol{\bar{\gamma}}_i^{t}, \boldsymbol{\bar{w}}^{t}-\boldsymbol{w}^{*} \big\rangle +\big\langle \boldsymbol{\bar{\gamma}}_i^{t}-\boldsymbol{\gamma}_i, \boldsymbol{\bar{w}}_i^{t}-\boldsymbol{\bar{w}}^{t} \big\rangle\bigg)  \\ 
		%   \leq & \sum_{k=1}^{t} \sum_{i=1}^{n} [ \frac{{\Vert -(\rho+1/\eta_i^{k}) \boldsymbol{\xi}_i^{k}  \Vert}^2}{2(1/\eta_i^{k}-(c_3+\lambda c_4/n))} \\&+  \frac{1}{2\eta_i^{k}}({\Vert \boldsymbol{w}_i^{*} - \boldsymbol{\tilde{w}}_i^{k} \Vert}^2 - {\Vert \boldsymbol{w}_i^{*} - \boldsymbol{\tilde{w}}_i^{k} \Vert}^2) \\ &+\frac{\rho}{2}({\Vert \boldsymbol{w}_i^{*} - \boldsymbol{w}^{k-1} \Vert}^2 - {\Vert \boldsymbol{w}_i^{*} - \boldsymbol{w}^{k} \Vert}^2) \\& + \big\langle -(\rho+1/\eta_i^{k}) \boldsymbol{\xi}_i^{k}, \boldsymbol{w}_i^{*} - \boldsymbol{\tilde{w}}_i^{k} \big\rangle \\&+ \frac{1}{2\rho}({\Vert \boldsymbol{\gamma}_i-\boldsymbol{\gamma}_i^{k} \Vert}^2 - {\Vert \boldsymbol{\gamma}_i-\boldsymbol{\gamma}_{i}^{k} \Vert}^2)] \\
		%	\leq & \sum_{i=1}^{n}\frac{1}{t} \sum_{k=1}^{t} \bigg(\frac{\big(\rho+1/\eta_i^{k}\big)^2}{2(1/\eta_i^{k}-(c_3+\lambda c_4/n))} {\big\Vert  \boldsymbol{\xi}_i^{k}  \big\Vert}^2\\& \quad \quad \quad \quad \quad -\big (\rho+1/\eta_i^{k}\big)  \big\langle\boldsymbol{\xi}_i^{k}, \boldsymbol{w}_i^{*} - \boldsymbol{\tilde{w}}_i^{k-1} \big\rangle \bigg)\\& + \frac{1}{t} \sum_{i=1}^{n} (\frac{1}{2\eta_i^{t}}{\Vert \boldsymbol{w}_i^{*} - \boldsymbol{\tilde{w}}_i^{0} \Vert}^2+\frac{\rho}{2}{\Vert \boldsymbol{w}_i^{*} - \boldsymbol{w}^{0} \Vert}^2+\frac{1}{2\rho}{\Vert \boldsymbol{\gamma}_i-\boldsymbol{\gamma}_i^{0} \Vert}^2  )\\
		= &  \sum_{i=1}^{n}\frac{1}{t} \sum_{k=1}^{t} \bigg(\frac{\big(\rho+1/\eta_i^{k}\big)^2}{2(1/\eta_i^{k}-(c_3+\lambda c_4/n))} {\big\Vert  \boldsymbol{\xi}_i^{k}  \big\Vert}^2  -\big (\rho+1/\eta_i^{k}\big)  \big\langle\boldsymbol{\xi}_i^{k}, \boldsymbol{w}_i^{*} - \boldsymbol{\tilde{w}}_i^{k-1} \big\rangle \bigg) \\&+\frac{1}{t}\sum_{i=1}^n \frac{c_w^2}{2\eta_i^t}+\frac{\rho n}{2 t}c_w^2+ \frac{1}{t} \sum_{i=1}^{n} \frac{1}{2\rho}{\Vert \boldsymbol{\gamma}_i-\boldsymbol{\gamma}_i^{0} \Vert}^2 . 
		\end{split}
		\end{equation}
		The above inequality holds for all $\boldsymbol{\gamma}_i$, thus it also holds for $\boldsymbol{\gamma}_i \in \{ \boldsymbol{\gamma}_i: \Vert \boldsymbol{\gamma}_i \Vert \leq \beta \}$. By letting $\boldsymbol{\gamma}_i$ be the optimum, we have
		\begin{equation}
		\begin{split}
		& \max_{ \{ \boldsymbol{\gamma}_i: \Vert \boldsymbol{\gamma}_i \Vert \leq \beta \} }  \sum_{i=1}^{n}\bigg(f_i(\boldsymbol{\bar{w}}_i^{t})-f_i(\boldsymbol{w}^{*}_i)+  \big\langle -\boldsymbol{\bar{\gamma}}_i^{t},\boldsymbol{\bar{w}}_i^{t}-\boldsymbol{w}_i^{*} \big\rangle + \big\langle \boldsymbol{\bar{\gamma}}_i^{t}, \boldsymbol{\bar{w}}^{t}-\boldsymbol{w}^{*} \big\rangle +\big\langle \boldsymbol{\bar{\gamma}}_i^{t}-\boldsymbol{\gamma}_i, \boldsymbol{\bar{w}}_i^{t}-\boldsymbol{\bar{w}}^{t} \big\rangle\bigg)  \\ 
		= & \max_{ \{ \boldsymbol{\gamma}_i: \Vert \boldsymbol{\gamma}_i \Vert \leq \beta \} }  \sum_{i=1}^{n}\bigg( f_i(\boldsymbol{\bar{w}}_i^{t})-f_i(\boldsymbol{w}_i) - \boldsymbol{\gamma}_i(\boldsymbol{\bar{w}}_i^{t}- \boldsymbol{\bar{w}}^t)\bigg) \\
		= & \quad \quad \sum_{i=1}^{n}\bigg( f_i(\boldsymbol{\bar{w}}_i^{t})-f_i(\boldsymbol{w}_i) + \beta\Vert \boldsymbol{\bar{w}}_i^{t}- \boldsymbol{\bar{w}}^t\Vert \bigg). \label{eq:14}
		%		\leq &  \sum_{i=1}^{n}\frac{1}{t} \sum_{k=1}^{t} [\frac{{\Vert -(\rho+1/\eta_i^{k}) \boldsymbol{\xi}_i^{k}  \Vert}^2 }{2(1/\eta_i^{k}-(c_3+\lambda c_4/n))} \\& + \big\langle -(\rho+1/\eta_i^{k}) \boldsymbol{\xi}_i^{k}, \boldsymbol{w}_i^{*} - \boldsymbol{\tilde{w}}_i^{k} \big\rangle ] +\frac{N}{t}(\frac{c_w^2}{2\eta_t}+\frac{\rho}{2}c_w^2)\\& +  \max_{ \{ \boldsymbol{\gamma}_i: \Vert \boldsymbol{\gamma}_i \Vert \leq \beta \} }  \frac{1}{t} \sum_{i=1}^{n} \frac{1}{2\rho}{\Vert \boldsymbol{\gamma}_i-\boldsymbol{\gamma}_i^{0} \Vert}^2 \\
		%		= &  \sum_{i=1}^{n}[ f_i(\boldsymbol{\bar{w}}_i^{t})-f_i(\boldsymbol{w}_i) + \beta (\boldsymbol{\bar{w}}_i^{t}- \boldsymbol{\bar{w}}^t)]\\
		%	\leq &  \sum_{i=1}^{n}\frac{1}{t} \sum_{k=1}^{t} [\frac{{\Vert -(\rho+1/\eta_i^{k}) \boldsymbol{\xi}_i^{k}  \Vert}^2}{2(1/\eta_i^{k}-(c_3+\lambda c_4/n))} \\& + \big\langle -(\rho+1/\eta_i^{k}) \boldsymbol{\xi}_i^{k}, \boldsymbol{w}_i^{*} - \boldsymbol{\tilde{w}}_i^{k} \big\rangle ] \\& +\frac{N}{t}(\frac{c_w^2}{2\eta_t}+\frac{\rho}{2}c_w^2) + \frac{N}{t} \frac{\beta^2}{2\rho}. 
		\end{split}
		\end{equation}
		Since we have $\mathbb{E}\big[\big\langle\boldsymbol{\xi}_i^{k}, \boldsymbol{w}_i^{*} - \boldsymbol{\tilde{w}}_i^{k-1} \big\rangle\big] = 0$ and $\mathbb{E}\big[{\big\Vert  \boldsymbol{\xi}_i^{k}  \big\Vert}^2\big] = dp\sigma_{i,k}^{2}  = 8dp\ln(1.25/\delta) c_1^2/\big(m_i^2 \epsilon^2{(\rho+ 1/\eta_i^{k})}^2\big)$ due to the variance definition, we take the expectation of the \eqref{eq:14} and let $\eta_i^{k} = {\big(c_3+\lambda c_4/n+2c_1\sqrt{4dpk\ln(1.25/\delta)}/\big(\epsilon m_i c_w\big)\big)}^{-1}$, which leads to the result:
		\begin{equation}
		\begin{split}
		&  \mathbb{E}\bigg[ \sum_{i=1}^{n}\big( f_i(\boldsymbol{\bar{w}}_i^{t})-f_i(\boldsymbol{w}_i^*) + \beta \big\Vert \boldsymbol{\bar{w}}_i^{t}- \boldsymbol{\bar{w}}^t\big\Vert\big)\bigg] \\
		\leq & \mathbb{E}\bigg[ \sum_{i=1}^{n}\frac{1}{t} \sum_{k=1}^{t} \frac{\big(\rho+1/\eta_i^{k}\big)^2}{2(1/\eta_i^{k}-(c_3+\lambda c_4/n))} {\big\Vert  \boldsymbol{\xi}_i^{k}  \big\Vert}^2\bigg]  - \sum_{i=1}^{n}\frac{1}{t} \sum_{k=1}^{t}\big(\rho+1/\eta_i^{k}\big)\mathbb{E}\bigg[\big\langle  \boldsymbol{\xi}_i^{k}, \boldsymbol{w}_i^{*} - \boldsymbol{\tilde{w}}_i^{k-1} \big\rangle\bigg] \\& +  \frac{1}{t}\sum_{i=1}^n \frac{c_w^2}{2\eta_i^t}+\frac{\rho n}{2 t}c_w^2+ \max_{ \{ \boldsymbol{\gamma}_i: \Vert \boldsymbol{\gamma}_i \Vert \leq \beta \} } \frac{1}{t}  \sum_{i=1}^{n} \frac{1}{2\rho}{\Vert \boldsymbol{\gamma}_i-\boldsymbol{\gamma}_i^{0} \Vert}^2 \\
		%	= & \frac{1}{t} \sum_{i=1}^{n}\sum_{k=1}^{t}\frac{2d\ln(1.25/\delta)4 c_1^2/\epsilon^2}{\sqrt{4dk\ln(1.25/\delta)}2c_1/(m_i \epsilon c_w)} \\&+\sum_{i=1}^n \frac{1}{t}\frac{c_w^2(c_3+\lambda c_4/n+\sqrt{4k\ln(1.25/\delta)}2c_1/(\epsilon c_w)}{2} \\& +\frac{n \rho }{2t}c_w^2 +\frac{n}{t} \frac{\beta^2}{2\rho}\\
		=& \sum_{i=1}^{n} \frac{  c_wc_1 \sqrt{dp\ln(1.25/\delta)} }{ m_i\epsilon t}\bigg(\sum_{k=1}^{t}\frac{1}{\sqrt{k}}+2\sqrt{t}\bigg)  +\frac{n c_w^2 (c_3+\lambda c_4/n)}{2t} + \frac{\rho n }{2t}c_w^2 +\frac{n}{t} \frac{\beta^2}{2\rho}\\
		\leq & \sum_{i=1}^{n} \frac{ 4  c_w  c_1 \sqrt{dp\ln(1.25/\delta)} }{m_i \epsilon \sqrt{t} }+\frac{n c_w^2(c_3+ \lambda c_4/n)}{2}+\frac{n c_w^2\rho+n \beta^2/\rho}{2}. 
		\end{split}
		\end{equation}
	\end{proof}
	

\end{document}